\documentclass[oneside,10pt]{article}
\usepackage[left=2.5cm,top=1.5cm,right=2.5cm,nohead,nofoot]{geometry}

\usepackage{setspace} 
\usepackage[utf8x]{inputenc}
\usepackage{user_commands}
\usepackage{amssymb,amsmath}
\usepackage{epsfig, color}
\usepackage{multirow}
\usepackage{algorithm}
\usepackage{algorithmic}
\usepackage{ifthen}
\usepackage{import}
\usepackage{amsthm}
\usepackage[bookmarks,colorlinks]{hyperref}
\hypersetup{citecolor=blue}

\author{Onay Urfalioglu and Orhan Arikan}
\newtheorem{theorem}{Theorem}[section]
\newtheorem{lemma}[theorem]{Lemma}

\begin{document}

\title{Symmetry Breaking in Neuroevolution: A Technical Report}

% \author{\name Onay Urfalioglu and Orhan Arikan \email \{onay,oarikan\}@ee.bilkent.edu.tr \\
%        \addr Department of Electrical and Electronics Engineering\\
%        Bilkent University\\
%        06800 Ankara, Turkey}
%        \AND
%        \name Orhan Arikan \email oarikan@ee.bilkent.edu.tr \\
%        \addr Department of Electrical and Electronics Engineering\\
%        Bilkent University\\
%        06800 Ankara, Turkey}

% \editor{}

\maketitle

\begin{abstract}
Artificial Neural Networks (ANN) comprise important symmetry properties, which can influence the performance of Monte Carlo methods in Neuroevolution. 
The problem of the symmetries is also known as the competing conventions problem or simply as the permutation problem. In the literature, symmetries are mainly addressed in Genetic Algoritm based approaches. However, investigations in this direction based on other Evolutionary Algorithms (EA) are rare or missing. Furthermore, there are different and contradictionary reports on the efficacy of symmetry breaking. By using a novel viewpoint, we offer a possible explanation for this issue. As a result, we show that a strategy which is invariant to the global optimum can only be successfull on certain problems, whereas it must fail to improve the global convergence on others. We introduce the \emph{Minimum Global Optimum Proximity} principle as a generalized and adaptive strategy to symmetry breaking, which depends on the location of the global optimum. We apply the proposed principle to Differential Evolution (DE) and Covariance Matrix Adaptation Evolution Strategies (CMA-ES), which are two popular and conceptually different global optimization methods. Using a wide range of feedforward ANN problems, we experimentally illustrate significant improvements in the global search efficiency by the proposed symmetry breaking technique.
\end{abstract}
% \begin{keywords}
%   feedforward neural networks, global optimization, symmetry breaking, permutation problem, competing conventions, evolutionary algorithms
% \end{keywords}

\section{Introduction}
Artificial Neural Networks (ANN) are general function approximators~\cite{haykin-1998} and can be used to find a functional representation of a data set. Another point of view is that ANN's represent a way of data compression~\cite{Arbib2002}. The compression ratio depends on the number of neurons used in the ANN which encodes the data: the less neurons at the same representation quality, the better the compression. 

Given a problem, there are generally two kinds of optimization tasks for the learning process of ANN's. The first one is to find a network topology, i.e., the optimal number of layers and the optimal number of neurons per layer. The second task is to find the parameters of the network, given a topology. In this paper, we focus on the second task and assume a predefined topology.

The estimation of the ANN-parameters is generally a computationally demanding task~\cite{sirma2001}. The corresponding Maximum-Likelihood derived error function comprises many local optima. Therefore, local search techniques to find an optimal solution generally fail and typically converge to a suboptimal solution~\cite{haykin-1998}. In addition, local search techniques are mainly sequential methods and parallel implementations are limited. On the other hand, global optimization techniques based on Monte Carlo methods such as the Genetic Algorithm (GA)~\cite{Goldberg1989,Michalewicz1994}, Covariance Matrix Adaptation Evolution Strategies (CMA-ES)~\cite{Hansen:1996,Hansen:2003} or Differential Evolution (DE)~\cite{Storn95a,Pri96,Vesterstrom2004} are generally very well parallelizable. Differential Evolution is one of the most popular and robust Monte Carlo global search methods, which outperforms many other evolutionary algorithms on a wide range of problems~\cite{BersiniDLSG96,tusar2007,Xu2007}. DE is successfully used in various engineering problems such as multiprocessor synthesis~\cite{Rae1998}, optimization of radio network designs~\cite{Mendes2006}, training Radial Basis Function networks~\cite{Liu2005}, training multi layer neural networks~\cite{Ilonen2003} and many others~\cite{Chakraborty:2008:ADE}. On the other hand, CMA-ES is a state-of-the-art evolutionary algorithm, which is also used for ANN-learning~\cite{Siebel2007,EANT2-CMAES-APP-SiebBoet09,Gomez2008} and other engineering tasks~\cite{Shir2006,jiang2008,Siebel2008}.

Due to inherent symmetries in the parametric representation of ANN's, there are multiple \emph{global} optima in the parameter space. The multiple global optima result from point symmetries and permutation symmetries~\cite{Sussmann1992,Thierens96}. In the literature, this problem is also known as the \emph{competing conventions problem}, or simply the \emph{permutation problem}. In~\cite{Thierens1993,Thierens96}, significant improvements are reported by different approaches to symmetry breaking for GA's. However, in both publications, the improvement is shown using only one single test-case, respectively. On the other hand, in~\cite{Hancock1992,Haflidason2009} contradictionary results are presented, where the effect of removing these symmetries on GA's is reported to be minimal and negligable, and even leading to reduced performance.

Furthermore, crossover operators used in GA's are reported to be a source of the problems caused by symmetries~\cite{Garcia-ANN-Crossover2006}. Therefore, some researchers disable crossover or apply EA's which do not have crossover at all~\cite{Yao199883}.

To our best knowledge, there are no reports on the impact of the ANN-symmetries regarding the performance of the DE and CMA-ES methods.
In this paper, we show that the performance of DE and CMA-ES are highly sensitive to the presence of multiple global optima, and that symmetries are also an issue on the performance of EA's without crossover operators. We show that there are infinitely many ways of symmetry breaking, which differ in the way they partititon the parameter space. Furthermore, we argue that an effective way of partitioning should depend on the location of the global optimum and its symmetric replicas. Therefore, we derive a symmetry breaking operator based on considerations about the partitioning of the ANN-parameter space, which is optimal according to a \emph{Minimum Global Optimum Proximity} condition. By theoretical considerations and numerous experimental studies on offline supervised learning problems, we show that typical approaches to symmetry breaking, which are invariant to the global optimum, may lead to superior or inferior results, depending on the ANN-problem. \\ 
On the other hand, we show that the proposed global optimum variant approach for symmetry breaking leads to consistent and significant improvements in the estimation of ANN-parameters. \\

The paper is organized as follows. In the following Section, we briefly review Artificial Feedforward Neural Networks (ANN). Section~\ref{sec:ANN-symmetries} defines the term 'symmetry' and introduces the types of symmetries found in the optimization of ANN-parameters. In Section~\ref{seq:go-inv-sb}, we discuss existing approaches to symmetry breaking. In this Section, we also reformulate the rules applied by existing approaches to prepare a more general view to the topic. In Section~\ref{sec:mgopp}, we introduce the 'Minimum global optimum proximity' principle and propose symmetry breaking methods based on this principle. In Section~\ref{sec:experiments}, we present the conducted experiments and obtained results, followed by the Section of Conclusions, where the main contributions are emphasized.

\section{Brief review of Artificial Feedforward Neural Networks}\label{sec:ANN-brief}
Artificial (Feedforward) Neural Networks (ANN) are used for approximation of functions $f: \mathbb{R}^d \rightarrow \mathbb{R}^q$. ANN's typically have multiple layers of artificial neurons. Assuming that an ANN has $L$ layers, the first and the last layer are called as the input and the output layer, respectively. Remaining $L-2$ layers are called as hidden layers.

For the $n$-th neuron $(l,n)$ in layer $l$, we denote a parameter vector by 
\begin{equation}
\bm{\eta}^l_{n}=(\bm{w}^l_{n},\tau^l_{n}),\ n=1,...,N_l,
\end{equation}  
where $\bm{w}^l_{n}$ is the weight vector of dimension equal to the number of inputs available to the neuron and $\tau^l_{n}$ is the shift scalar. The output of a tanh-type sigmoid neuron $(l,n)$ is given by 
\begin{equation}\label{eq:tanh}
 x^l_{n}=\tanh({\bm{w}^l_{n}}^{\top}\bm{x}^{l-1}+\tau^l_{n}),%\equiv \(\bm{\theta}_{m,n};\bm{y}_{m-1}),
\end{equation} 
where $\bm{x}^l=(x^l_{1}, \ldots, x^l_{N_l})$ is the output vector of layer $l$. After all hidden layers $l=2,3,...,L-1$ are evaluated, the output layer component $\hat{y}_n$ of the output vector $\hat{\bm{y}}$ is typically obtained by the following two alternative ways:
% \begin{equation}
\begin{eqnarray}
%\label{eq:lin-ouput-layer}
%  y=\sum_{i=1}^{N_M}w_{M,i}\cdot\tanh(\bm{w}_{M-1,i}^{\top}\bm{x}_{M-2}+\tau_{M-1,i}).
\hat{y}_n={\bm{w}^L_{n}}^{\top}\bm{x}^{L-1},\ && n=1,...,q \text{ (regression)},\label{eq:regr-lin-ouput-layer}\\
\hat{y}_n=\tanh({\bm{w}^L_{n}}^{\top}\bm{x}^{L-1}),\ && n=1,...,q \text{ (classification).}\label{eq:class-lin-ouput-layer}
\end{eqnarray}
We denote the parameter vector of all neurons in a layer $l$ by $\bm{\lambda}^l$, where
\begin{equation}
 \bm{\lambda}^l=(\bm{\eta}^l_{1},\ldots,\bm{\eta}^l_{N_l}).
\end{equation} 
The vector of all the parameters in the network is given by 
\begin{equation}
\bm{\theta}_a=(\bm{\lambda}^2,\ldots,\bm{\lambda}^{L-1},\bm{w}^L_{1},\ldots,\bm{w}^L_{q}),
\end{equation} 
where $\bm{w}_n^L=(w^L_{n,1},\ldots,w^L_{n,N_{L-1}}),\ n=1,...,q$, is the vector of the output layer weights for output $\hat{y}_n$. The function defined by the network is denoted by
\begin{equation}
 \hat{\bm{y}}=\Omega(\bm{\theta}_a;\bm{x}),
\end{equation} 
where $\bm{x}$ is the input vector, which is notationwise equal to the output of the input layer, so that $\bm{x}^1\equiv\bm{x}$.
\\
Assuming additive normal i.i.d. noise on the available data $(\bm{x}_k,\bm{y}_k), k=1,...,K$, the ML-estimate $\hat{\bm{\theta}}_a$ of the parameters $\bm{\theta}_a$ can be obtained by the minimizer to the following least squares optimization problem:
\begin{equation}\label{eq:ANN-argmin}
 \hat{\bm{\theta}}_a=\arg\min_{\bm{\theta}_a}\sum_{k=1}^K(\bm{y}_k-\Omega(\bm{\theta}_a;\bm{x}_k))^{\top}(\bm{y}_k-\Omega(\bm{\theta}_a;\bm{x}_k)).
\end{equation} 
For regression problems, the output layer is linear as shown in Eqn.~\eqref{eq:regr-lin-ouput-layer}. Thus, the corresponding weights $\bm{w}^{L}$ can be determined by a least squares method, as described in~\cite{Masters93}, which we adopt in this paper. This has the advantage that global search is applied only to the non-linear part of the parameter space, which generally speeds up convergence. For classification problems, we assume that 
an output vector $\bm{y}$ of a data-sample designating class $i$ has the following format
\begin{equation}
 y_j=\left\{
\begin{matrix}
1 & \mathrm{for}\ j=i\\
0 & \mathrm{else.}             
\end{matrix}\right.
\end{equation} 

Although the output layer is non-linear as shown in Eqn.~\eqref{eq:class-lin-ouput-layer}, corresponding weights $\bm{w}_n^L$ can still be determined linearly in the training phase. For this, the output vectors of the training data are rescaled by factor 20, such that $\tanh(20)\approx 1$ and $\tanh(0)=0$. The weights of the output layer are determined by a least squares method using the rescaled data. Given the remaining parameters, Eqn.~\eqref{eq:ANN-argmin} is applied by using the non-rescaled data.

Consequently, the parameter vector $\bm{\theta}$ for the global optimization can be reduced to
\begin{equation}\label{go-pars}
 \bm{\theta}=(\bm{\lambda}^2,\ldots,\bm{\lambda}^{L-1}).
\end{equation} 
The important problem of how to choose the net topology is not considered in this paper. For a given net-topology, we focus on the effect of  symmetry breaking on the efficiency of the optimization of the parameters in~(\ref{go-pars}). In the following Section, we investigate the symmetries in the ANN-parameter space.

\section{Symmetries in ANN's}\label{sec:ANN-symmetries}
A \emph{symmetry} is an operator $\Phi$ which does not change the output of an ANN when applied to the parameter vector $\bm{\theta}$:
\begin{equation}\label{symm-def}
 \Omega(\bm{\theta};\bm{x})=\Omega(\Phi(\bm{\theta});\bm{x}), \ \forall \bm{\theta},\bm{x}.
\end{equation} 
Non reducable ANN's comprise two types of symmetries~\cite{Sussmann1992}. The first type is a \emph{point symmetry} on the neuron parameter level, since
\begin{equation}\label{eq:point-symm-eq}
w\tanh(x)=-w\tanh(-x),\ \forall w,x. 
\end{equation}
The following definition of a point symmetry operator $O^l_{n}$
\begin{equation}\label{eq:point-symm}
 O^l_{n}(\bm{\theta}): \left\{\begin{array}{lcl}
\bm{\eta}^l_{n} & \rightarrow & -\bm{\eta}^l_{n}\\
w^{l+1}_{i,n} & \rightarrow & -w^{l+1}_{i,n},\ i=1,\ldots,N_{l+1}                          
\end{array}\right.
\end{equation} 
changes the sign of the parameters of neuron $(l,n)$ and the $n$-th weight component $w^{l+1}_{i,n}$ of all neurons $(l+1,i)$ in the following layer $l+1$. It satisfies the symmetry condition because of Eqn.~\eqref{eq:point-symm-eq}. In Fig.~\ref{fig:point-symm-example}, an example for the application of $O^2_{1}$ 
is shown. For each layer $l$, the point symmetry yields $2^{N_l}$ symmetric replicas of the parameter vector $\bm{\theta}$.
\begin{figure}[h]
   \begin{center}
      \scalebox{0.5}{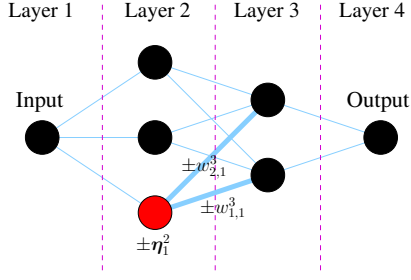}
      \caption{\label{fig:point-symm-example} \it Application of the point symmetry operator $O^2_{1}$, which changes the signs of 
$\bm{\eta}^2_1$-parameters in layer two and $w^3_{i,1}$-parameters in layer three, respectively.}
   \end{center} 
\end{figure}

The second type of symmetry is a \emph{permutation symmetry} by the neuron parameters $\bm{\eta}$ and the corresponding weight 
parameters in the next layer. A permutation operator $P^l_{j,k}$ defined by 
\begin{equation}\label{eq:perm-symm}
 P^l_{j,k}(\bm{\theta}):\left\{\begin{array}{lcl}
\bm{\eta}^l_{j}\leftrightarrow\bm{\eta}^l_{k}\\
w^{l+1}_{i,j}\leftrightarrow w^{l+1}_{i,k},\ i=1,\ldots,N_{l+1}
\end{array}\right.
\end{equation} 
leaves the output invariant. Note that $P^l_{j,k}=P^l_{k,j}$. In Fig.~\ref{fig:perm-symm-example}, the application of $P^2_{1,2}=P^2_{2,1}$ is illustrated. In each layer $l$, 
there are $N_l!$ symmetric replicas of the parameter vector $\bm{\theta}$ due to permutation symmetries. Combining both symmetries, the total count of symmetric replicas per layer $l$ is $2^{N_l}N_l!$. Another important property is that the length of the vector $\bm{\theta}$ is invariant under such symmetry operators,
\begin{equation}\label{eq:length-invar}
||\Phi(\bm{\theta})||=||\bm{\theta}||,\ \forall \bm{\theta},
\end{equation}  
since the point symmetry operator only changes the sign of some components of the parameter vector, whereas the permutation symmetry operator only swaps some components. 
\begin{figure}[tb]
   \begin{center}
      \scalebox{0.5}{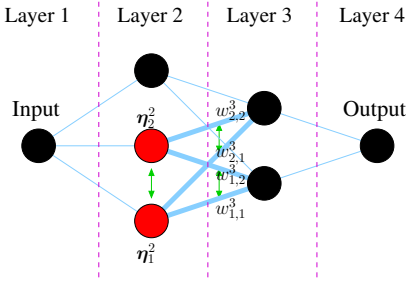}
      \caption{\label{fig:perm-symm-example} \it Application of the permutation symmetry operator $P^2_{1,2}=P^2_{2,1}$, which exchanges 
the parameters $\bm{\eta}^2_1\leftrightarrow\bm{\eta}^2_2$ in layer two and the parameters $w^3_{1,1}\leftrightarrow w^3_{1,2}$, $w^3_{2,1}\leftrightarrow w^3_{2,2}$ in layer three.}
   \end{center} 
\end{figure}
\begin{lemma}
 Symmetry operators are linear and orthogonal operators.
\end{lemma}
\begin{proof}
 The proof for the linearity of these operators is trivial and therefore omitted in this paper. The orthogonality follows from Eqn.~\eqref{eq:length-invar}: 
\begin{align}
 ||\Phi \bm{\theta}||=||\bm{\theta}||, \forall \bm{\theta}  & \Rightarrow ||\Phi \bm{\theta}||^2=||\bm{\theta}||^2, \forall \bm{\theta}  \\
& \Rightarrow (\Phi \bm{\theta})^{\top}(\Phi \bm{\theta}) = \bm{\theta}^{\top}\bm{\theta}, \forall \bm{\theta} \\
& \Rightarrow \bm{\theta}^{\top}\Phi^{\top}\Phi\bm{\theta} = \bm{\theta}^{\top}\bm{\theta}, \forall \bm{\theta} \Rightarrow \Phi^{-1}=\Phi^{\top}\label{eq:orthogonal}.
\end{align} 
\end{proof}
Furthermore, applying the same point symmetry operator two times subsequently does not change the parameter vector, since switching the signs of selected components a second time reverts the first sign-change. The same holds also for the permutation symmetry operator: swapping the selected components a second time reverts the first swapping. Therefore, we can write
\begin{equation}\label{eq:identity-as prod}
O^l_iO^l_i=\mathcal{I},\ P^l_{i,j}P^l_{i,j}=\mathcal{I},
\end{equation} 
where $\mathcal{I}$ is the identity operator. As a result, point symmetry, permutation symmetry as well as joint symmetry operators correspond to rotations and all symmetric replicas of a global optimum lie on a hypersphere.
Since such symmetries multiply the local and global optima count in the parameter space, the ultimate goal of symmetry breaking is to reduce the total number of local optima in the parameter space by avoiding all but one symmetrically equivalent space partitions. 

There are infinitely many ways for symmetry breaking by using the operators $O^l_{n}$ and $P^l_{j,k}$, which depend on the \emph{condition} upon which these operators are applied. As an example, consider a 2-D point symmetry as illustrated in Fig.~\ref{fig:point-symm-r1}. Limiting the search space to the upper half plane ($y>0$) is one possibility to break the symmetry, where only one global optimum remains and the space is separated into two partitions. In this case, the point symmetry operator is to be applied only for $y<0$. Another possibility is to reduce the space to the right half plane ($x>0$). This is realized by applying the point symmetry operator only on the condition $x<0$. By rotating the coordinate system, we obtain infinitely many other ways to separate and reduce the space. As a result, there is a degree of freedom on the choice of a specific condition or separation. We derive similar results also for the permutation symmetry. In Section~\ref{sec:mgopp}, we argue that there is an optimal choice for a specific symmetry breaking condition (separation) based on considerations about the location of the global optimum. We exploit the degree of freedom on the choice of a specific condition by choosing a condition such that the distance of the global optimum to the separating region is maximal. In other words, we demand that the proximity of the global optimum to the separating region is minimal. This way, the influence of neighboring global optima is minimized and the symmetry breaking can be realized most effectively.

A detailed discussion about an optimal separation follows in Section~\ref{sec:mgopp}.
\section{Existing approaches to deal with symmetries}\label{seq:go-inv-sb}
A commonly used method is to reduce the parameter space to one single symmetrically equivalent region, also called partition. To achieve this, the following rules can be applied~\cite{Thierens96}:
\begin{description}
\item[rule-1] The shift parameter of \emph{all} neurons is ensured to be positive by flipping the signs of the parameters when required, for each neuron.
\item[rule-2] In each hidden layer, neurons are sorted according to the shift parameter.
\end{description}
This method and all other similar methods can be realized by applying a chain of the operators $O^l_{n}$ and $P^l_{j,k}$. 
In the following, we show that these rules are suboptimal, and in some cases may even cause inferior performance. We show that rules for symmetry breaking should take the position of the global optimum into account in order to be effective. Therefore, we denote rule-1 and rule-2 as \emph{global optimum invariant}, and rules which depend on the global optimum as \emph{global optimum variant}. %In order to simplify the discussion, we assume that the \emph{basin}, or the region of influence of the global optimum is isotropic.
\subsection{Global optimum invariant point symmetry breaking}\label{sec:subopt-point-symm}
Assuming a point symmetric function $f(x,y)=f(-x,-y),\ \forall x,y$, Fig.~\ref{fig:point-symm-r1} shows two cases where rule-1 is applied such that all $y$-coordinates are forced to be positive. As a consequence, all solution candidates are located in the upper half plane and the parameter space is effecively reduced. There is only one remaining global optimum $\bar{\bm{\eta}}$. In the left plot, the global optima $\bar{\bm{\eta}}$ and $-\bar{\bm{\eta}}$ are relatively far away from the $x$-axis, whereas in the right plot, the global optima are close to the $x$-axis, although they have the same distance to the origin in both plots. In case of the right plot, there exists an 'artificial' local optimum due to the proximity of the hidden global optimum $-\bar{\bm{\eta}}$, where some solution candidates may be attracted to. The main problem is that after applying symmetry breaking, some solution candidates may still be closer to the hidden global optimum $-\bar{\bm{\eta}}$ than to $\bar{\bm{\eta}}$. As a result, the goal of reducing the influence of other global optima is not fully achieved. Furthermore, the introduced artificial local optimum may trap some solution candidates without having a chance to ever reach the corresponding 'hidden' global optimum $-\bar{\bm{\eta}}$. We believe that this is the main reason why an inferior performance is reported by some symmetry breaking approaches. Note that this situation depends on the location of the global optimum, which in turn depends on the problem at hand. Therefore, this issue arises on some problems, whereas on others, a symmetry breaking with increased performance can be achieved by these rules.
\begin{figure}[t]
   \begin{center}
      \scalebox{0.5}{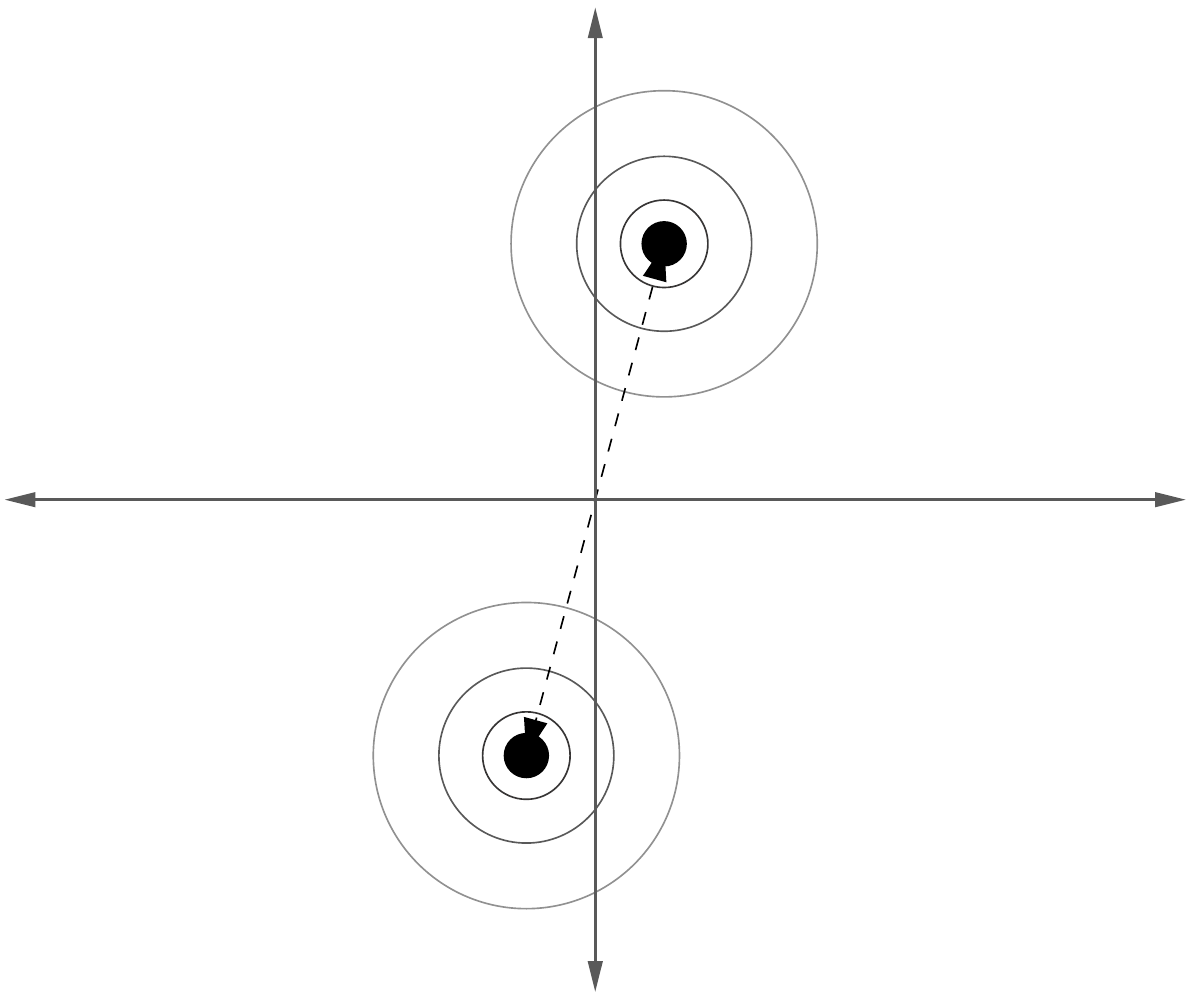}
      \scalebox{0.5}{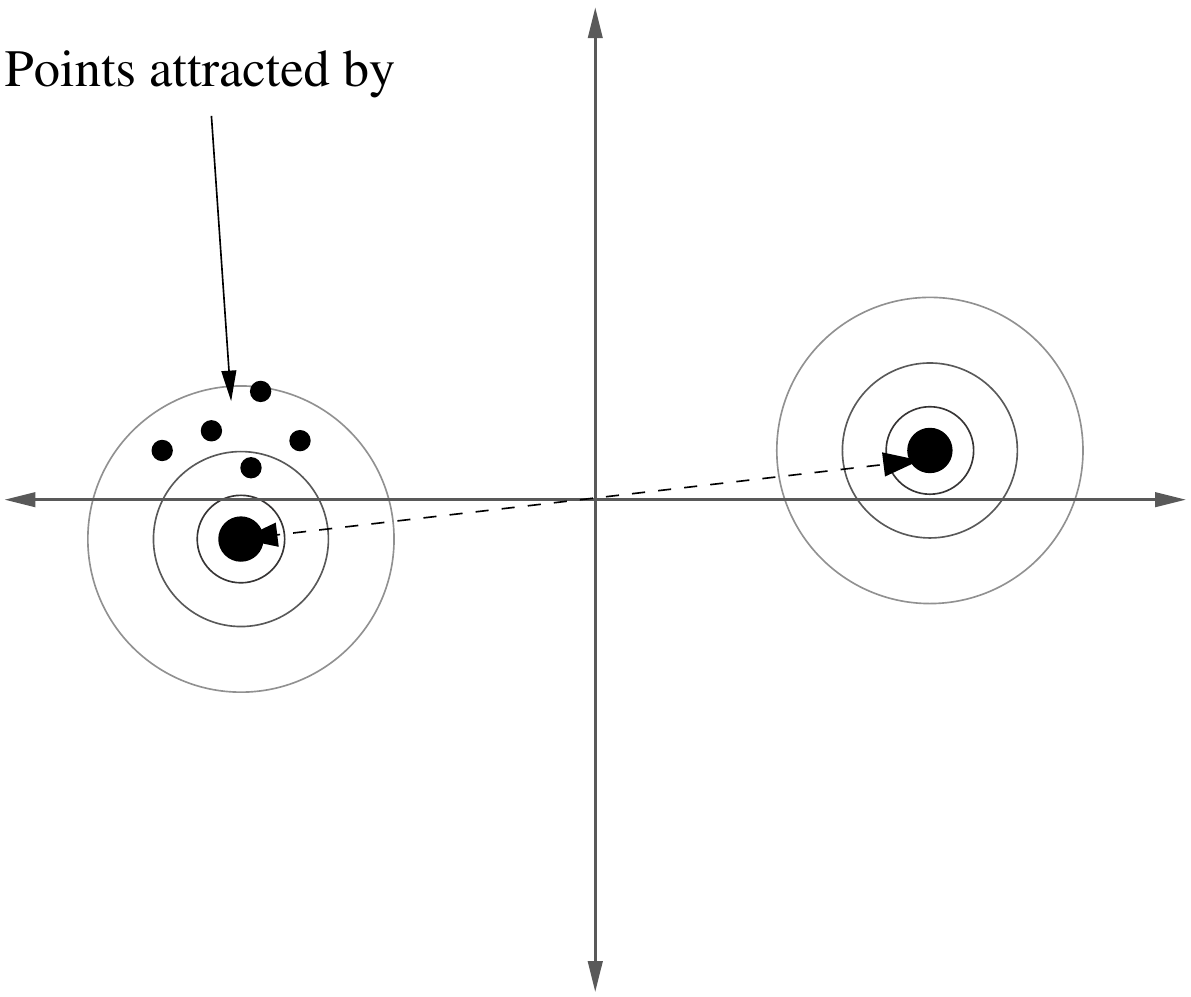}
      \caption{\label{fig:point-symm-r1} \it Example for a point symmetry in 2-D, where $f(x,y)=f(-x,-y)\ \forall x,y$. It is assumed that rule-1 is applied to force all solution candidates to be in the upper half plane ($y \geq 0$). As a result, the parameter space is effecively reduced and there is only one remaining global optimum $\bar{\bm{\eta}}$. In the left plot, the global optima $\bar{\bm{\eta}}$ and $-\bar{\bm{\eta}}$ are relatively far away from the $x$-axis, whereas in the right plot, the global optima are close to the $x$-axis, although they have the same distance to the origin in both cases. In case of the right plot, there exists an 'artificial' local optimum due to the proximity of the hidden global optimum $-\bar{\bm{\eta}}$, where some solution candidates may be attracted to. The main problem is that after applying symmetry breaking, some solution candidates may still be closer to the hidden global optimum $-\bar{\bm{\eta}}$ than to $\bar{\bm{\eta}}$.}
   \end{center} 
\end{figure}
In Fig.~\ref{fig:point-symm-r1}, the $x$-axis is the region $\mathcal{S}$ of separation
\begin{equation}
\mathcal{S}=\{\lambda:(\lambda,0)\}.
\end{equation} 
The separating region depends on the rule and divides the parameter space into partitions. As an example, an alternative rule, which would force all $x$ coordintates to be positive, would have the $y$-axis as the separating region. We repeat that the distance of the global optimum to the separating region is crucial for effective symmetry breaking, and that it should be arranged to have this distance as large as possible. Another equivalent goal is to apply symmetry breaking such that no solution candidate is closer to the hidden global optimum than to the global optimum of the selected partition.
\subsection{Global optimum invariant permutation symmetry breaking}\label{sec:subopt-perm-symm}
Similar problems caused by rule-1 also arise by the application of rule-2. This is shown in the following example. We use a 2x2 parameter structure, i.e., two neurons with two parameters $(a_i,b_i)$ per neuron $i$: $\bm{\theta}=(a_1,b_1,a_2,b_2)$. From the permutation symmetry follows that 
\begin{equation}
 f((a_1,b_1,a_2,b_2))=f((a_2,b_2,a_1,b_1)),\ \forall a_1,b_1,a_2,b_2,
\end{equation} 
where $f$ shall be the error function.
Let the global optimum be at $\bar{\bm{\theta}}=(2,1,-2,3)$. There are two possibilities to apply rule-2: sorting by parameter $a$ or sorting by parameter $b$, respectively. The separating region varies for each choice. Choosing to sort by parameter $a$ yields $\mathcal{S}_a$, whereas sorting by parameter $b$ yields $\mathcal{S}_b$:
\begin{equation}
\mathcal{S}_a=\{\alpha,\beta,\lambda:(\lambda,\alpha,\lambda,\beta)\},\ \mathcal{S}_b=\{\alpha,\beta,\lambda:(\alpha,\lambda,\beta,\lambda)\}.
\end{equation}
We show that each separation region has a different distance to the global optimum $\bar{\bm{\theta}}$. The closest point on $\mathcal{S}_a$ to $\bar{\bm{\theta}}$ is at $\lambda=0,\alpha=1,\beta=3$, which yields the distance $\sqrt{8}$. On the other hand, the closest point on $\mathcal{S}_b$ to $\bar{\bm{\theta}}$ is at $\lambda=2,\alpha=2,\beta=-2$, which yields the distance $\sqrt{2}$. In this example, applying rule-2 by ordering the $a$-coordinates results in a better sparation of the partitions. Would the global optimum be at $\bar{\bm{\theta}}=(1,2,3,-2)$, the opposite case would apply. Consequently, similar to rule-1 in the previous Section~\ref{sec:subopt-point-symm}, rule-2 can only be effective on some problems. 

\section{Minimum global optimum proximity principle}~\label{sec:mgopp}
In this Section we propose new methods for symmetry breaking to avoid the problems described in Section~\ref{seq:go-inv-sb}. Here, we assume that the \emph{basin}, or the region of influence of the global optimum is isotropic. Although this assumption does not apply in general, it is introduced to simplify the discussion. Also, this simplification enables us to easily derive theoretically motivated methods, which prove to be very effective in a wide range of problems. In the presentation, we first consider the point symmetry, then the permutation symmetry and finally the general joint symmetry as a combination of both point and permutation symmetries.
\subsection{Minimum global optimum proximity principle for point symmetry}
The differences between possible rules to apply the point symmetry operator arise from the \emph{condition} on which the operator is to be applied. Fig.~\ref{fig:point-symm-go1} shows different rules with corresponding separation regions for breaking a point symmetry in relation to the global optimum.
\begin{figure}[h]
   \begin{center}
      \scalebox{0.37}{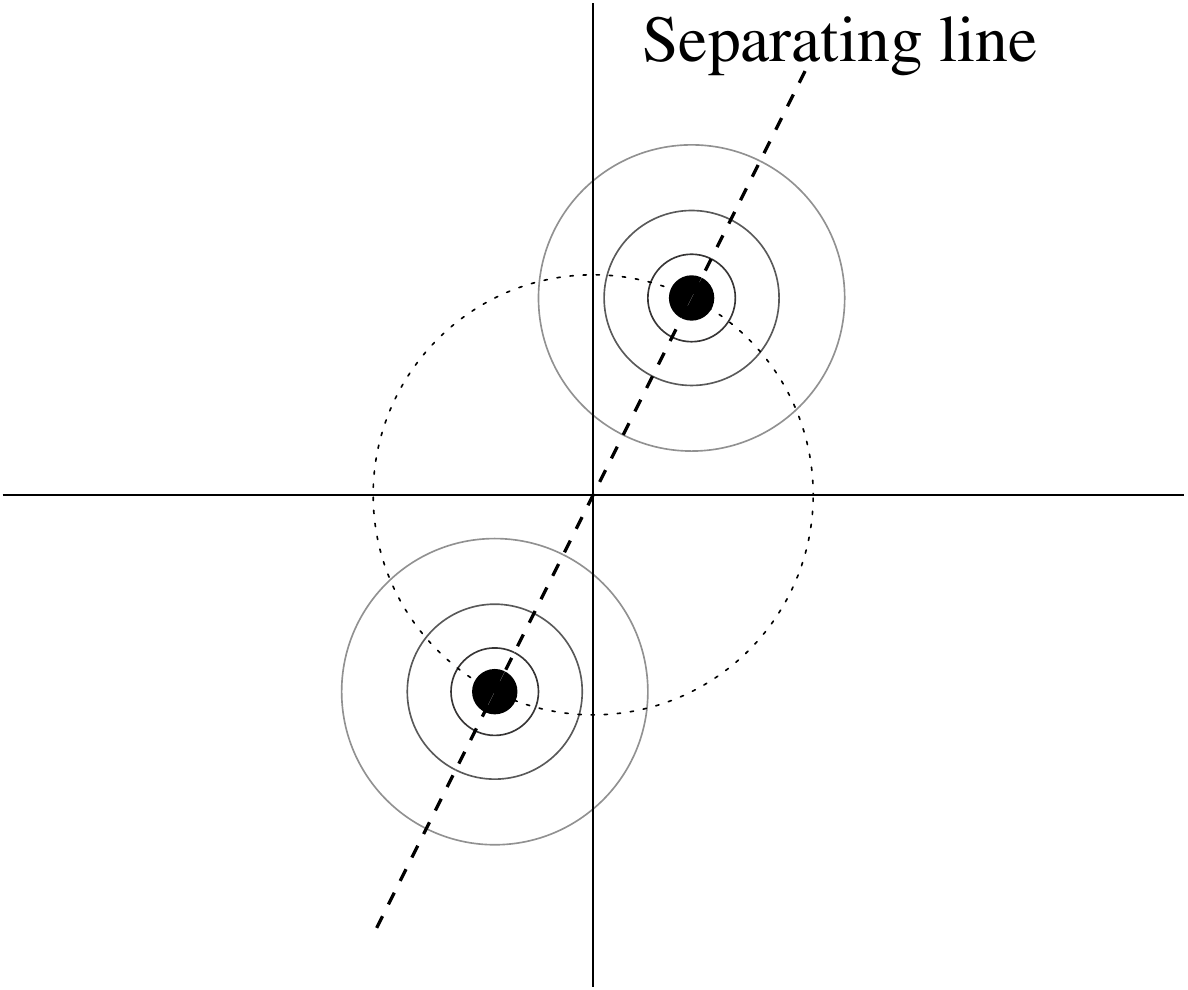}
      \scalebox{0.37}{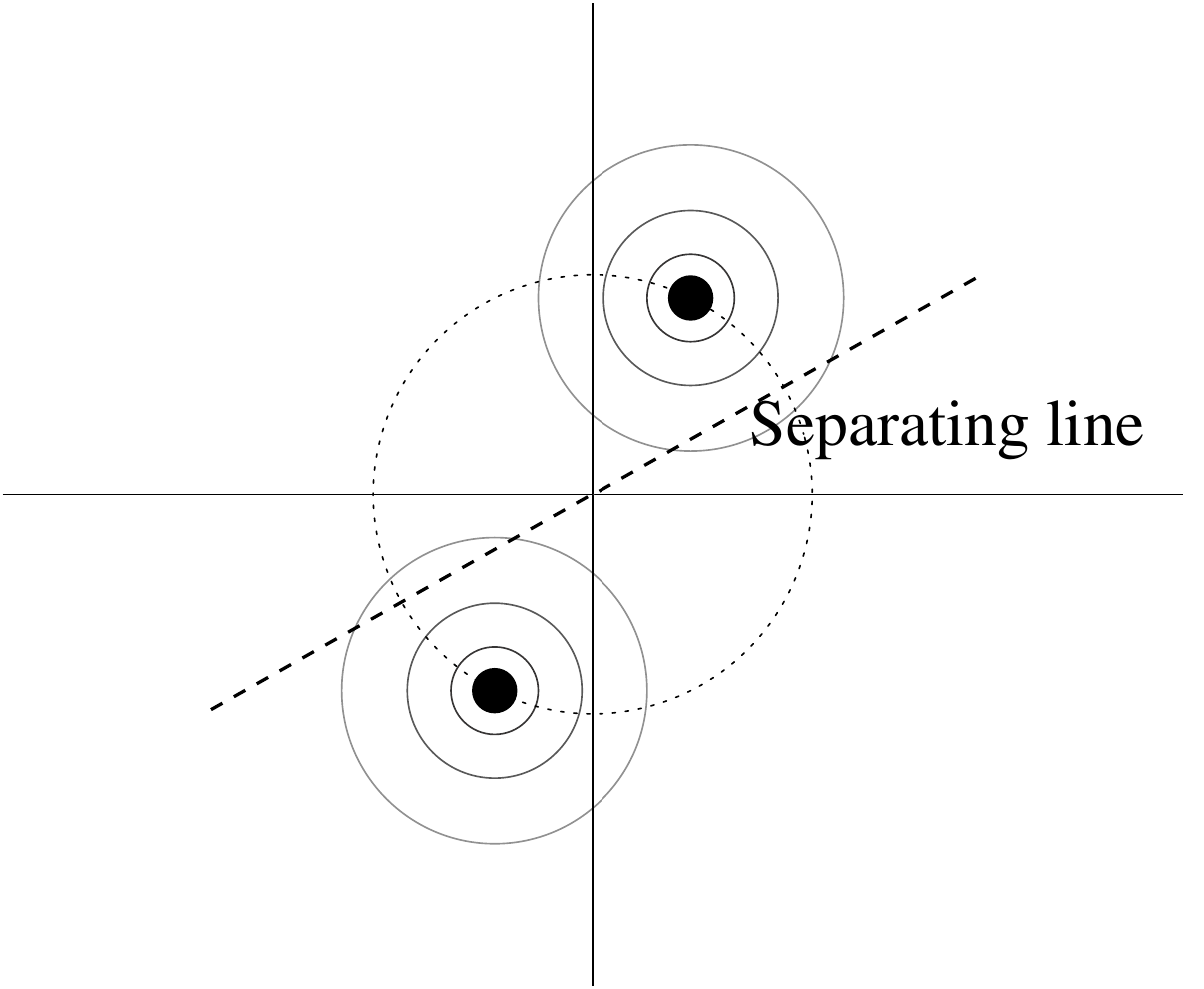}
      \scalebox{0.37}{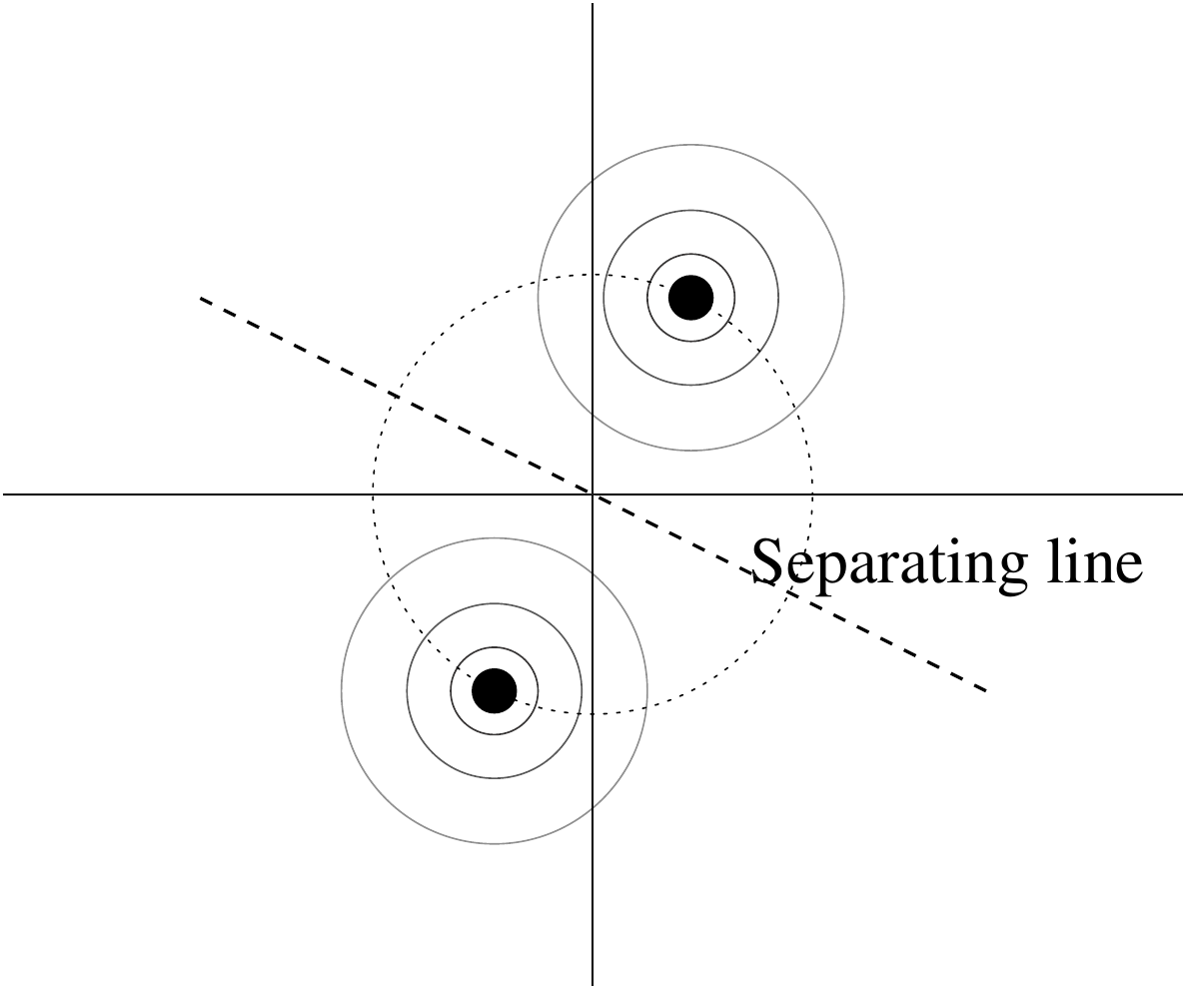}
      \caption{\label{fig:point-symm-go1} \it Example for a point symmetry in 2-D, where $f(\bm{\eta})=f(-\bm{\eta})\ \forall \bm{\eta}$. The plots show worst case (left), suboptimal (middle) and optimal separation lines (right) for point symmetry breaking. The separating line divides the parameter space in two parts, where each partition contains a global optimum ($\bar{\bm{\eta}}$ and $-\bar{\bm{\eta}}$).}
   \end{center} 
\end{figure}
It can be seen that the separating region which has 
maximum distances to the global optima, which means that the according \emph{proximity} is minimal, enables the optimal separation or partitioning. This way, an optimal isolation between all 
symmetric replicas of the global optimum is achieved. As a result, the disturbing influence of other neighboring global optima is decreased to a minimum, which in turn effectively maximizes the attraction of the global optimum of the selected partition.

The following Lemma provides a more general perspective for rule-1 presented in Section~\ref{seq:go-inv-sb}. Note that the shift parameter is the last entry in the parameter vector. 
\begin{lemma}\label{lm:gen-point-symm}
Rule-1 from Section~\ref{seq:go-inv-sb} modifies a parameter vector $\bm{\eta}$ as:
\begin{equation}\label{eq:point-subopt}
\bm{\eta}'=\left\{\begin{array}{lcl}
\bm{\eta} & \mathrm{if} & ||\bm{\eta}-(0,...,0,1)||^2 \leq ||-\bm{\eta}-(0,...,0,1)||^2\\
-\bm{\eta} & \mathrm{if} & ||\bm{\eta}-(0,...,0,1)||^2 > ||-\bm{\eta}-(0,...,0,1)||^2.
\end{array}\right.
\end{equation}
\end{lemma}
\begin{proof} 
From the first line in Eqn.~\eqref{eq:point-subopt} follows with $\bm{\eta}=(x_1,...,x_D)$ and a reference vector $\bm{r}=(0,...,0,1)$
\begin{align}
 ||\bm{\eta}-\bm{r}||^2 &\leq ||-\bm{\eta}-\bm{r}||^2\label{rel:point-symm}\\
 ||\bm{\eta}-(0,...,0,1)||^2 &\leq ||-\bm{\eta}-(0,...,0,1)||^2\\
\Leftrightarrow \left(\sum_{i=1}^{D-1}x_i^2\right)+(x_D-1)^2 &\leq \left(\sum_{i=1}^{D-1}x_i^2\right)+(-x_D-1)^2.
\end{align}
Further simplifying both sides of the equation yields
 \begin{equation}
 -x_D \leq x_D\ \Leftrightarrow x_D \geq 0.
\end{equation}
This means that the conditional Equation~(\ref{eq:point-subopt}) is equivalent to rule-1 which demands that the shift parameters shall be positive.
\end{proof}

The rule-structure introduced by Lemma~\ref{lm:gen-point-symm} can be used to formulate the following strategy to maximize the distance of the global optimum $\bar{\bm{\eta}}$ to the separating region.
\begin{equation}\label{eq:point-opt}
\bm{\eta}'=\left\{\begin{array}{lcl}
\bm{\eta} & \mathrm{for} & ||\bm{\eta}-\bar{\bm{\eta}}||^2 \leq ||-\bm{\eta}-\bar{\bm{\eta}}||^2\\
-\bm{\eta} & \multicolumn{2}{c}{\mathrm{otherwise}}% \mathrm{for} & & ||\bm{\eta}-\bar{\bm{\eta}}||^2 > ||-\bm{\eta}-\bar{\bm{\eta}}||^2.
\end{array}\right.
\end{equation}
\begin{theorem}\label{theorem-optimal-point-symm}
The solution candidate $\bm{\eta}'$ determined by rule~(\ref{eq:point-opt}) is always closer to $\bar{\bm{\eta}}$ than to $-\bar{\bm{\eta}}$.
\end{theorem}
We will prove Theorem~\ref{theorem-optimal-point-symm} in a more general setting in Section~\ref{sec:ideal-symm-brk}.

\subsection{Minimum global optimum proximity principle for permutation symmetry}
In this Section we introduce an optimal rule for breaking a permutation symmetry for parameter spaces with two blocks of permutation-invariant parameters. We define a parameter vector $\bm{\theta}$ as
\begin{equation}
 \bm{\theta}=(\bm{\eta}_1,\bm{\eta}_2)\equiv (\bm{\eta}_1|\bm{\eta}_2),
\end{equation}
where the notation $(\bm{\eta}_1|\bm{\eta}_2)$ is used to emphasize the block structure.
The permutation symmetry is given by
\begin{equation}
 f(\bm{\theta})=f((\bm{\eta}_1,\bm{\eta}_2))=f(P\bm{\theta})=f((\bm{\eta}_2,\bm{\eta}_1)),\ \forall \bm{\theta},
\end{equation} 
where $f$ is the error function and $P$ is a permutation operator defined by 
\begin{equation}
P(\bm{\eta}_1,\bm{\eta}_2)=(\bm{\eta}_2,\bm{\eta}_1),\ \forall \bm{\eta}_1,\bm{\eta}_2.
\end{equation} 
The following Lemma restates rule-2 as a distance dependent rule.
\begin{lemma}\label{lm:gen-perm-symm}
Assuming the shift parameter is the last parameter in the parameter block $\bm{\eta}$, rule-2, presented in Section~\ref{seq:go-inv-sb}, can alternatively be described in a more general form by the following rule: 
\begin{equation}\label{eq:perm-subopt}
\bm{\theta}'=\left\{\begin{array}{lcl}
(\bm{\eta}_1|\bm{\eta}_2) & \mathrm{for} & ||(\bm{\eta}_1|\bm{\eta}_2)-(0,...,0|0,...,1)||^2 \leq ||(\bm{\eta}_2|\bm{\eta}_1)-(0,...,0|0,...,1)||^2\\
(\bm{\eta}_2|\bm{\eta}_1) & \multicolumn{2}{c}{\mathrm{otherwise}}
\end{array}\right.
\end{equation}
\end{lemma}
\begin{proof}
From Eqn.~\eqref{eq:perm-subopt} follows with $\bm{\eta}_i=(x_{i,1},...,x_{i,D})$
\begin{align}
 &||(x_{1,1},...,x_{1,D}|x_{2,1},...,x_{2,D})-(0,...,0|0,...,1)||^2 \nonumber\\
\leq & ||(x_{2,1},...,x_{2,D}|x_{1,1},...,x_{1,D})-(0,...,0|0,...,1)||^2\\
\Leftrightarrow & x_{1,D}^2+(x_{2,D}-1)^2 \leq (x_{1,D}-1)^2+x_{2,D}^2\\
\Leftrightarrow & x_{2,D} \leq x_{1,D}
\end{align}
\end{proof}
We state the following proposal in order to maximize the distance of the global optimum $\bar{\bm{\theta}}$ to the separating region, according to the rule-structure introduced by Lemma~\ref{lm:gen-perm-symm}
\begin{equation}\label{eq:perm-opt}
\bm{\theta}'=\left\{\begin{array}{lcl}
(\bm{\eta}_1|\bm{\eta}_2) & \mathrm{for} & ||(\bm{\eta}_1|\bm{\eta}_2)-\bar{\bm{\theta}}||^2 \leq ||(\bm{\eta}_2|\bm{\eta}_1)-\bar{\bm{\theta}}||^2\\
(\bm{\eta}_2|\bm{\eta}_1) & \multicolumn{2}{c}{\mathrm{otherwise}}
\end{array}\right.
\end{equation}
\begin{theorem}\label{theorem-perm}
The solution candidate $\bm{\theta}'$ determined by rule~(\ref{eq:perm-opt}) is always closer to $\bar{\bm{\theta}}$ than to $P\bar{\bm{\theta}}$.
\end{theorem}
Theorem~\ref{theorem-perm} will be proved in a more general setting in Section~\ref{sec:ideal-symm-brk}.
\subsection{Ideal symmetry breaking}\label{sec:ideal-symm-brk}
For a given ANN-optimization problem, let $\mathcal{P}$ be the set of all possible symmetry operators. Note that a symmetry operator $\Phi\in\mathcal{P}$ may be a point symmetry, a permutation symmetry or a joint symmetry operator. A joint symmetry operator is generally composed of a chain of point symmetry and permutation symmetry operators. As an example, $\Phi=O^2_{2}\circ P^3_{2,4}$ applies a permutation symmetry followed by a point symmetry operator. 
The following properties of symmetry operators are relevant in the following discussion. According to Eqn.~\eqref{symm-def}, a symmetry operator does not change the output of the ANN when applied to the parameter vector $\bm{\theta}$. According to Eqn.~\eqref{eq:length-invar} a symmetry operator does not change the length of a parameter vector. Furthermore, according to Eqn.~\eqref{eq:orthogonal}, symmetry operators are orthogonal. 

Given a parameter vector $\bm{\theta}$, the set $\mathcal{R}_{\bm{\theta}}$ of all symmetric replicas of $\bm{\theta}$ is defined by
\begin{equation}
 \mathcal{R}_{\bm{\theta}}=\{\Phi \in \mathcal{P}: \Phi\bm{\theta}\}=\mathcal{P}\bm{\theta}.
\end{equation} 
Recall that the ultimate goal of symmetry breaking is to minimize the influence of all symmetric replicas of the selected global optimum and to concentrate the global search to the partition where the selected global optimum is located. To achieve this, we propose the following joint separation condition:
\begin{equation}\label{eq:sep-cond}
\bm{\theta}'=\arg\min_{\tilde{\bm{\theta}}\in \mathcal{R}_{\bm{\theta}}}||\tilde{\bm{\theta}}-\bar{\bm{\theta}}||^2.
\end{equation} 
In other words, this optimization selects the closest symmetric replica of $\bm{\theta}$ to the selected global optimum $\bar{\bm{\theta}}$. Finding the closest symmetric replica of $\bm{\theta}$ means finding the corresponding symmetry operator $\Phi'$, where
\begin{equation}\label{eq:phi-prime}
 \bm{\theta}'=\Phi'\bm{\theta}.
\end{equation} 
In case the parameter vector $\bm{\theta}$ is already close to $\bar{\bm{\theta}}$, i.e., it is in the corresponding partition, the solution for $\Phi'$ is the identity operator $\mathcal{I}$. Note that, according to Eqn.~\eqref{eq:identity-as prod}, the identity operator $\mathcal{I}$ is in $\mathcal{P}$. In Fig.~\ref{fig:sb-proof-illustr}, ideal symmetry breaking according to Eqn.~\eqref{eq:sep-cond} is illustrated on a hypothetical 2-D space.
\begin{figure}[tb]
\centering
\scalebox{0.75}{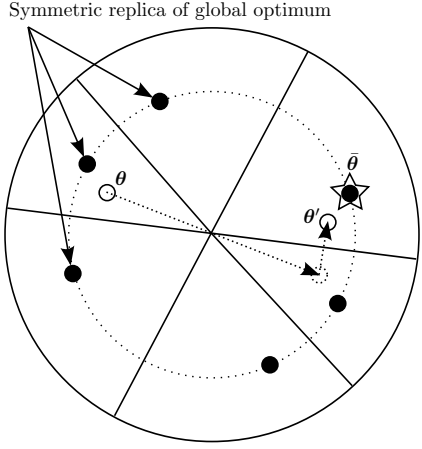}
% \includesvg{ideal-sb-illustration}
\caption{\label{fig:sb-proof-illustr} \it Ideal symmetry breaking according to Eqn.~\eqref{eq:sep-cond} shown on a hypothetical 2-D space. In this example, applying a point symmetry operator followed by a permutation symmetry operator maps $\bm{\theta}$ to $\bm{\theta}'$, which is located at the partition of the selected global optimum $\bar{\bm{\theta}}$, marked by a star. Note that a symmetry operator corresponds to a rotation, which preserves lengths as well as angles.}
\end{figure}

\begin{theorem}\label{theorem-sep-cond}
The solution $\bm{\theta}'$ determined by Equation~\eqref{eq:sep-cond} ensures that no other symmetric replica of the selected global optimum $\bar{\bm{\theta}}$ is closer to $\bm{\theta}'$ than $\bar{\bm{\theta}}$. In other words, it minimizes the influence of the symmetric replicas of the selected global optimum.
\end{theorem}
\begin{proof}
We prove this by contradiction. According to Eqn.~\eqref{eq:sep-cond}, $||\bm{\theta}'-\bar{\bm{\theta}}||^2$ is minimal. Assume that there exists a global optimum replica $\bar{\bm{\theta}}'\neq \bar{\bm{\theta}}$ with
\begin{equation}\label{eq:dist-smaller}
 ||\bm{\theta}'-\bar{\bm{\theta}'}||^2 < ||\bm{\theta}'-\bar{\bm{\theta}}||^2.
\end{equation} 
Due to the underlying symmetry properties, each global optimum replica can be mapped to another replica by a symmetry operator, i.e., there exists a symmetry operator $\Phi\in\mathcal{P}/\{\mathcal{I}\}$ which satisfies
\begin{equation}\label{eq:phi-prime-proof}
 \bar{\bm{\theta}}'=\Phi\bar{\bm{\theta}} \Rightarrow \bar{\bm{\theta}}=\Phi^{-1}\bar{\bm{\theta}}'.
\end{equation}  
Due to length-preserving property of symmetry operators, using Eqn.~\eqref{eq:phi-prime}, the left-hand side of the Relation~\eqref{eq:dist-smaller} can be written as
\begin{equation}\label{eq:transf-diff}
 ||\bm{\theta}'-\bar{\bm{\theta}'}||^2=||\Phi^{-1}(\bm{\theta}'-\bar{\bm{\theta}'})||^2=||\Phi^{-1}\bm{\theta}'-\Phi^{-1}\bar{\bm{\theta}'}||^2=||\Phi^{-1}\bm{\theta}'-\bar{\bm{\theta}}||^2
\end{equation} 
Since $\Phi\neq \mathcal{I}$ and therefore $\Phi^{-1} \neq \mathcal{I}$, it follows that $\Phi^{-1}\bm{\theta}'\neq \bm{\theta}'$. But this means that $\bm{\theta}'$ does not minimize the distance to $\bar{\bm{\theta}}$, which contradicts Eqn.~\eqref{eq:sep-cond}. 
\end{proof}

\subsection{Approximations of the ideal separation}
In order to take advantage of these results, we have to address two issues. First, the global optimum is not known a priori. 
%However, iterative algorithms like DE produce intermediate results at each iteration, which can be regarded as an approximation of the global optimum. This approximation becomes better with increasing iteration number. 
Second, the brute force method for finding an optimal solution to~(\ref{eq:sep-cond}) has exponential complexity, but a low-complexity algorithm is desired. In order to circumvent the first problem, we propose to use an estimate for the global optimum, which can be determined by the population of solution candidates at each iteration of the applied Monte Carlo method. Naturally, this estimate improves with increasing iteration number. The second problem can be addressed by using an approximation for the ideal separation achieved by~(\ref{eq:sep-cond}). 

To describe the proposed method, for each neuron $(l,n)$, we define a symmetry relevant parameter block $\bm{\beta}^l_n$ as
\begin{eqnarray}
 \bm{\beta}^l_n&=&(\bm{\eta}^l_{n},w^{l+1}_{1,n},\ldots,w^{l+1}_{N_{l+1},n}), \ l=2,...,L-2,\\
 \bm{\beta}^{L-1}_n&=&\bm{\eta}^{L-1}_{n},
\end{eqnarray}
which includes also some corresponding parameters from the next layer $l+1$.
%For the last hidden layer $l=L-1$, we define $\bm{\beta}^{L-1}_n=(\bm{\eta}^{L-1}_{n})$.
Given a parameter vector $\bm{\theta}$ and an estimate of the global optimum $\hat{\bm{\theta}}$ with corresponding parameter blocks $\bm{\beta}^l_n$ and $\hat{\bm{\beta}}^l_n$, the pseudocode~\ref{alg:heuristic} describes the proposed approximation for ideal symmetry breaking.
\begin{algorithm}[h]
{\small
\caption{\it Proposed symmetry breaking method. A symmetry operator $\Phi$ is only applied to the parameter vector $\bm{\theta}$ when it decreases the distance to the global optimum $\hat{\bm{\theta}}$, i.e., $||\Phi\bm{\theta}-\hat{\bm{\theta}}||<||\bm{\theta}-\hat{\bm{\theta}}||$. Algorithm input: $\bm{\theta}$ and $\hat{\bm{\theta}}$. Effect: modification of the parameter vector $\bm{\theta}$ when appropriate.}\label{alg:heuristic}
\begin{algorithmic}
\STATE {{\bf [breaking point symmetry]}}
\FOR {all hidden layers $l=2,\ldots,L-1$}
    \FOR {all neurons $(l,n),\ n=1,\ldots,N_l$ per layer $l$}
        \STATE {// would the point symmetry operator $O^l_{n}$ decrease the distance? ($||O^l_{n}\bm{\theta}-\hat{\bm{\theta}}|| \stackrel{?}{<} ||\bm{\theta}-\hat{\bm{\theta}}||$)}
        \STATE calculate distance-square for NOT applying $O^l_{n}$: $D_1=||\bm{\beta}^l_n-\hat{\bm{\beta}}^l_n||^2$
        \STATE calculate distance-square for applying $O^l_{n}$: $D_2=||-\bm{\beta}^l_n-\hat{\bm{\beta}}^l_n||^2$
        \IF {$D_1 > D_2$} 
            \STATE apply point symmetry operator $O^l_{n}$: set $\bm{\beta}^l_n=-\bm{\beta}^l_n$
        \ENDIF
    \ENDFOR
\ENDFOR
\STATE {{\bf [breaking permutation symmetry]}}
\FOR {all hidden layers $l=2,\ldots,L-1$}
    \STATE {randomly choose two neurons $(l,m),\ (l,n) \in\{1,\ldots,N_l\}$ in hidden layer $l$ with $m\neq n$}
%     \STATE {randomly choose a neuron $m\in\{1,\ldots,N_l\}/\{n\}$ in hidden layer $l$}
    \STATE {// would the permutation operator $P^l_{m,n}$ decrease the distance? ($||P^l_{m,n}\bm{\theta}-\hat{\bm{\theta}}|| \stackrel{?}{<} ||\bm{\theta}-\hat{\bm{\theta}}||$)}
    \STATE calculate distance-square for NOT applying $P^l_{m,n}$: $D_1=||\bm{\beta}^l_n-\hat{\bm{\beta}}^l_n||^2+||\bm{\beta}^l_m-\hat{\bm{\beta}}^l_m||^2$
    \STATE calculate distance-square for applying $P^l_{m,n}$: $D_2=||\bm{\beta}^l_n-\hat{\bm{\beta}}^l_m||^2+||\bm{\beta}^l_m-\hat{\bm{\beta}}^l_n||^2$
    \IF {$D_1 > D_2$} 
        \STATE apply permutation symmetry operator $P^l_{m,n}$: swap $\bm{\beta}^l_m \leftrightarrow \bm{\beta}^l_n$
    \ENDIF
\ENDFOR
\end{algorithmic}}
\end{algorithm}
In Fig.~\ref{fig:sb-examples}, the effect of the several symmetry breaking approaches is demonstrated on a hypothetical 2-D parameter space.
%%%%%%%%%%%%%%%%%%%%%%%%%%%%%%%%%%%%%%%%%%%%%%%%%%%%%%%%%%%%%%%%%%%%%%%%%%%%%%%%%%%%%%%%%%%%%%%%%%%%%%%%%%%%%%%%
% plot
\begin{figure*}[h!]
   \begin{center}
      \scalebox{0.5}{\includegraphics{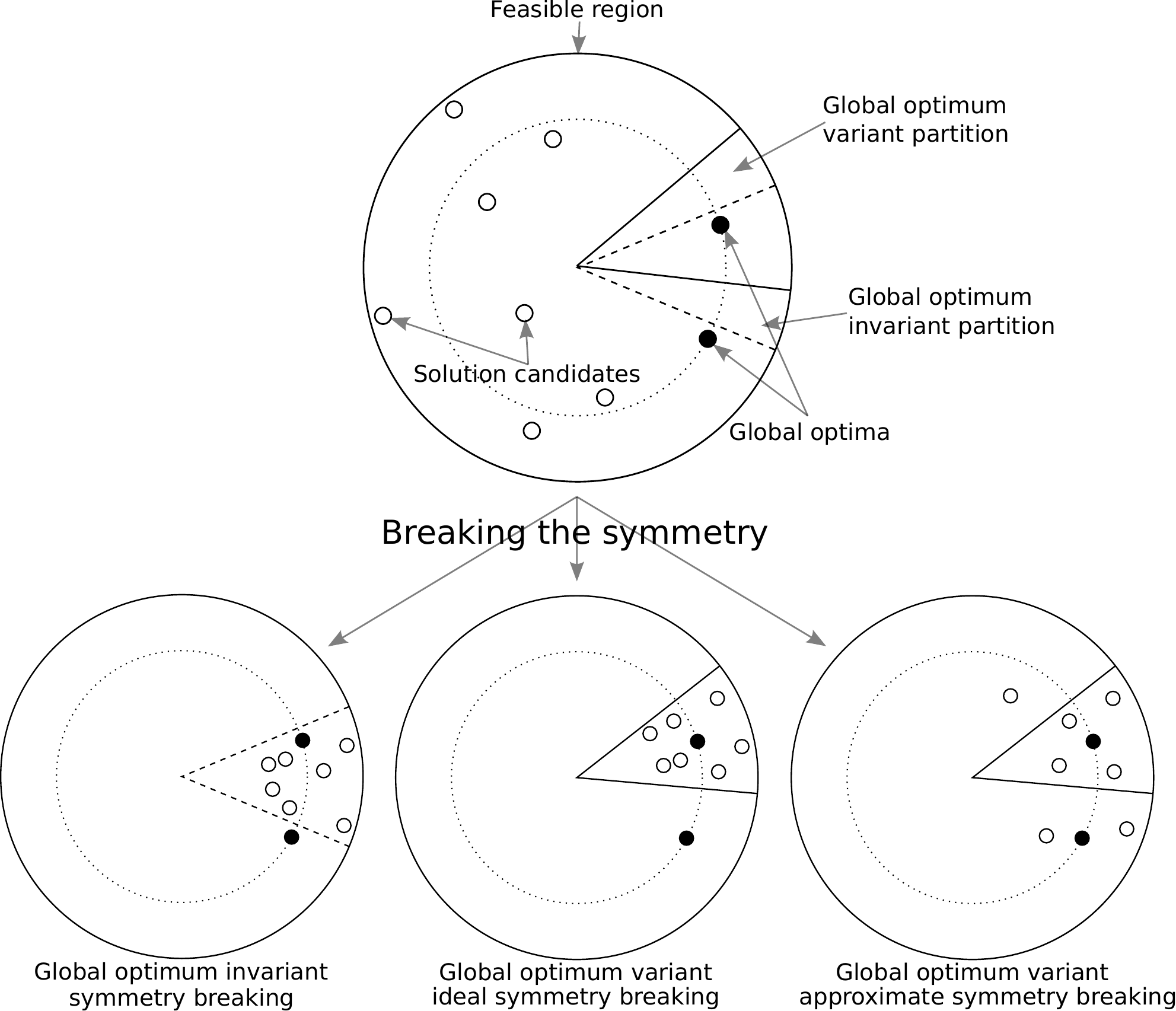}}
      \caption{\label{fig:sb-examples} \it Examples of symmetry breaking methods. Given a distribution of solution candidates as shown in the upper circle, typical outcomes of three different symmetry breaking methods are shown. In the left-bottom case, all solution candidates are mapped into the selected partition, but the global optimum is not necessarily centered within the partition. As a downside, there is a relatively strong influence of the global optimum from the neighbor partition. In the center-bottom case, the selected partition is chosen such that the distance to other symmetric replica of the global optimum are maximized, and all solution candidates are mapped into the selected partition. The right-bottom case shows the proposed approximate global optimum variant symmetry breaking. It equals the center-bottom case, except that the solution candidates are not necessarily mapped into the selected partition, but also to other partitions close to the selected one.}
   \end{center}
\end{figure*}

\subsubsection{DE with symmetry breaking}
The DE method~\cite{Storn95a,Pri96} comprises a population of solution candidates $\bm{\theta}_{i}$, which are iteratively updated and moved towards an optimal solution. We propose to choose the centroid of the population at each iteration as an estimate for the global optimum $\hat{\bm{\theta}}$. 

The DE method extended by the global optimum invariant symmetry breaking~\cite{Thierens96} is denoted by DE-INV-SB, DE extended by the proposed global optimum variant symmetry breaking, described by Algorithm~\ref{alg:heuristic}, is denoted by DE-SB and DE with global optimum variant ideal symmetry breaking using brute force search is denoted by DE-SB-BF.
As shown in Fig.~\ref{fig:de-scheme}, in DE-based symmetry breaking approaches, symmetry breaking is always applied on each solution candidate $\bm{\theta}_{i}$ right after it has been updated for the next iteration. % (see Eqn.~(\ref{eq:DE-trial})) at each iteration prior to the application of Eqn.~(\ref{eq:DE-trial}).
Only in DE-SB, we apply an additional step by increasing the error yield of some solution candidates which are not in the same partition as the selected partition holding $\hat{\bm{\theta}}$. This increases the probability that these solution candidates are updated and moved closer to the selected partition. This is not required for symmetry breaking approaches which map each solution cadidate exactly to the selected partition, such as DE-INV-SB or DE-SB-BF. The DE-SB method is described in Algorithm~\ref{alg:de-sb}. 
\begin{algorithm}[h]
{\small
\caption{\it DE-SB. Algorithm input: population of candidate vectors $\bm{\theta}_j,\ j=1,...,N_p$ and the centroid of the population as the estimate for the global optimum $\hat{\bm{\theta}}$. Effect: modify candidate vectors  $\bm{\theta}_j,\ j=1,...,N_p$ when appropriate.}\label{alg:de-sb}
\begin{algorithmic}
\FOR {all candidate vectors $\bm{\theta}_j,\ j=1,...,N_p$}
    \STATE {apply symmetry breaking on $\bm{\theta}_j$, see Algorithm~\ref{alg:heuristic}}
    \IF {$\bm{\theta}_j$ modified (a symmetry operator was applied) \AND $j<N_p/2$}
        \STATE {multiply the stored error yield of $\bm{\theta}_j$ by factor $100$}
    \ENDIF
\ENDFOR
\end{algorithmic}}
\end{algorithm}
%%%%%%%%%%%%%%%%%%%%%%%%%%%%%%%%%%%%%%%%%%%%%%%%%%%%%%%%%%%%%%%%%%%%%%%%%%%%%%%%%%%%%%%%%%%%%%%%%%%%%%%%%%%%%%%%
% plot
\begin{figure*}[h!]
   \begin{center}
      \scalebox{0.5}{\includegraphics{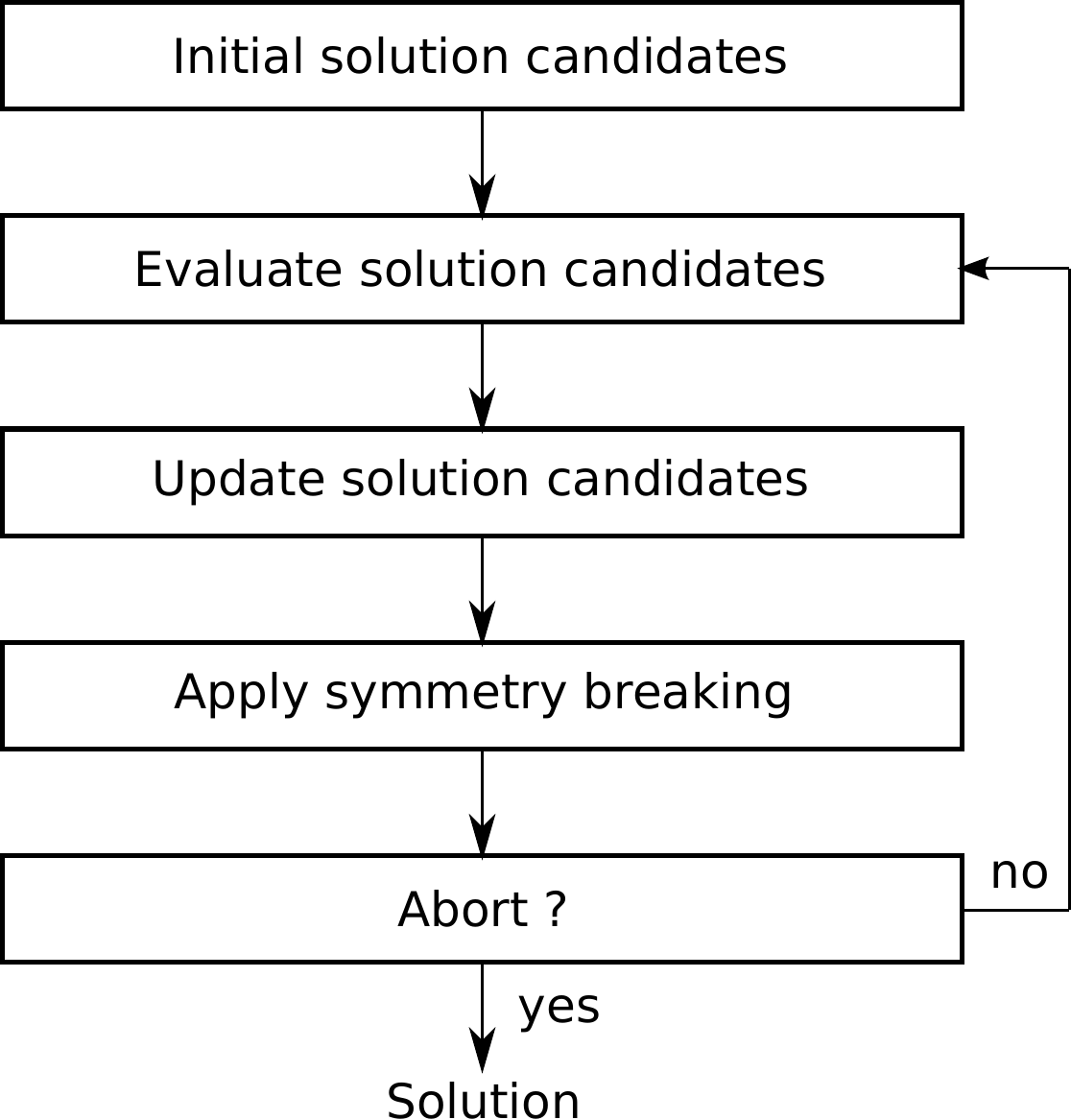}}
      \caption{\label{fig:de-scheme} \it Flowgraph for DE with symmetry breaking.}
   \end{center}
\end{figure*}

% A property of this heuristic is that the symmetry is not completely and not uniquely broken and the resulting modified parameters may belong to different partitions over time. As also $\hat{\bm{\theta}}$ changes over time, the final convergence result will be a random global optimum among all other possible global optima.
\subsubsection{CMA-ES with symmetry breaking}
The CMA-ES method~\cite{Hansen:1996,Hansen:2003} adapts a global step size $\sigma$, the mean $\bf{m}$ and a covariance matrix $C$ at each iteration. According to the Gaussian distribution $\mathcal{N}(\bf{m},\sigma C)$ with mean $\bf{m}$ and covariance matrix $\sigma C$, $N_p$ solution candidate vectors are drawn. After sorting the population by the error each candidate vector yields, the best $N_p/2$ samples are used to update the mean, covariance matrix and the step size for the next iteration. 

In the following discussion, the CMA-ES method extended by the global optimum invariant symmetry breaking~\cite{Thierens96} is denoted by CM-ES-INV-SB, CMA-ES extended by the proposed global optimum variant symmetry breaking, described by Algorithm~\ref{alg:heuristic}, is denoted by CMA-ES-SB and CMA-ES with global optimum variant ideal symmetry breaking using brute force search is denoted by CMA-ES-SB-BF.

In CMA-ES-INV-SB, CMA-ES-SB and CMA-ES-SB-SF, symmetry breaking is applied right after the evaluation of all candidate vectors and prior to updating the parameters of the Gaussian distribution. In CMA-ES-SB, we propose to use the best candidate vector (yielding the smallest error) so far as the estimate for the global optimum, denoted by $\hat{\bm{\theta}}$. In Fig.~\ref{fig:CMAES-SB-scheme}, the flowgraph for CMA-ES-based symmetry breaking approaches is shown. For CMA-ES-SB, the update of the mean is described in Algorithm~\ref{alg:cmaes-sb}. In all other CMA-ES-based methods, the original update formula for the mean is applied.

In CMA-ES, applying symmetry breaking introduces a bias in the mean, which can lead to an excessive increase of the global step size and negatively affect the performance. This bias results from the rotations caused by the symmetry operators. These rotations move solution candidates to the vicinity of one partiton, which typically increases the radius of the population mean, as shown in Fig.~\ref{fig:sb-examples}. In order to prevent such an increase, in all CMA-ES-based symmetry breaking methods, we modify the damping term for the update of the global step size $\sigma$. Let $\bm{s}$ be the shift vector of the centroid of the best $N_p/2$ solution candidates induced by applying symmetry breaking. The regular update formula for $\sigma$
\begin{equation}
 \sigma_{k+1} = \sigma_k \exp(\chi)
\end{equation} 
is changed to
\begin{equation}
 \sigma_{k+1} = \sigma_k \exp\left[\chi \exp(-0.05 D^2 ||\bm{s}||)\right],
\end{equation}
where $k$ is the iteration number and $\chi$ is a term depending on the difference of the previous mean and the current mean, and several other parameters.
\begin{algorithm}[h]
{\small
\caption{\it CMA-ES-SB. Algorithm input: population of candidate vectors $\bm{\theta}_j,\ j=1,...,N_p$, the estimate for the global optimum $\hat{\bm{\theta}}$ and weights $w_j,\ j=1,...,N_p$. Effect: modify candidate vectors  $\bm{\theta}_j,\ j=1,...,N_p$ when appropriate.}\label{alg:cmaes-sb}
\begin{algorithmic}
\FOR {all candidate vectors $\bm{\theta}_j,\ j=1,...,N_p$}
    \STATE {set mean vector $\bm{m}:=\bm{0}$}
    \STATE {apply symmetry breaking on $\bm{\theta}_j$, see Algorithm~\ref{alg:heuristic}}
    \IF {$\bm{\theta}_j$ modified (a symmetry operator was applied)}
        \STATE {add weighted global optimum estimate to mean vector: $\bm{m}:=\bm{m}+w_j \hat{\bm{\theta}}$}
    \ELSE
        \STATE {add weighted candidate vector to mean vector: $\bm{m}:=\bm{m}+w_j \bm{\theta}$} 
    \ENDIF
\ENDFOR
\end{algorithmic}}
\end{algorithm}
%%%%%%%%%%%%%%%%%%%%%%%%%%%%%%%%%%%%%%%%%%%%%%%%%%%%%%%%%%%%%%%%%%%%%%%%%%%%%%%%%%%%%%%%%%%%%%%%%%%%%%%%%%%%%%%%
% plot
\begin{figure*}[h!]
   \begin{center}
      \scalebox{0.5}{\includegraphics{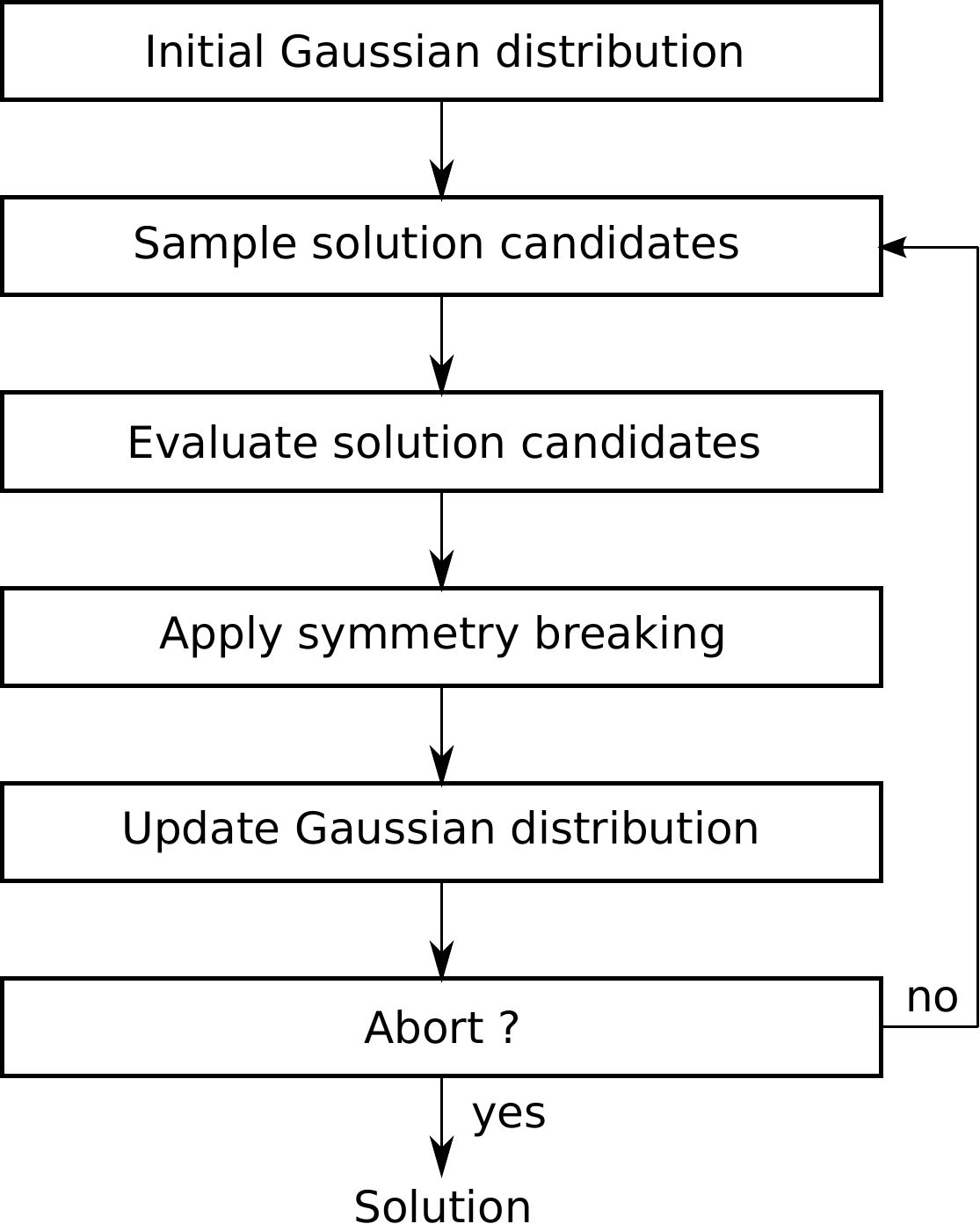}}
      \caption{\label{fig:CMAES-SB-scheme} \it Flowgraph for CMA-ES with symmetry breaking.}
   \end{center}
\end{figure*}
%%%%%%%%%%%%%%%%%%%%%%%%%%%%%%%%%%%%%%%%%%%%%%%%%%%%%%%%%%%%%%%%%%%%%%%%%%%%%%%%%%%%%%%%%%%%%%%%%%
\section{Experiments}\label{sec:experiments}
In this section, we introduce results of experiments to demonstrate the performance improvements by symmetry breaking. The following methods are compared using regression and classification tests. From the DE-family: Differential Evolution (DE), DE with global optimum invariant symmetry breaking (DE-INV-SB), DE with global optimum variant symmetry breaking (DE-SB) and DE with global optimum variant ideal symmetry breaking using brute force search (DE-SB-BF). From the CMA-ES-family: Covariance Matrix Adaptation Evolution Strategies (CMA-ES), CMA-ES with global optimum invariant symmetry breaking (CMA-ES-INV-SB), CMA-ES with global optimum variant symmetry breaking (CMA-ES-SB) and CMA-ES with global optimum variant ideal symmetry breaking using brute force search (CMA-ES-SB-BF). It should be noted that the purpose of this investigation is not to present the best global optimization method for ANN-learning, but to demonstrate the benefits of symmetry breaking.\\
With a $D$-dimensional parameter space, all tests are performed with following settings:
\begin{itemize}
 \item DE, DE-SB, DE-INV-SB and DE-SB-BF settings: $F=0.5$, $C_r=0.9$, initial population is randomly generated in $D$-dim. hypercube $[-1,1]^D$ (uniformly),
 \item CMA-ES, CMA-ES-SB, CMA-ES-INV-SB and CMA-ES-SB-BF settings: we used suggested settings for enhanced global search abilities, mentioned in the C-code reference implementation.
 \item in all experiments, the optimization is finished when a maximum number of ANN-function-evaluations is reached.
%  \item error threshold (min. MSE): $\epsilon_0$, depending on the problem.%=(\sqrt{2}\sigma)^2=5\cdot 10^{-5}$.
\end{itemize}
Given a parameter $\bm{\theta}$ and a data set $(\bm{x}_i,\bm{y}_i)$, we define the Mean Squared Error (MSE) $\epsilon$ according to Eqn.~\eqref{eq:ANN-argmin}:
\begin{equation}\label{eq:errorf}
 \epsilon=\frac{1}{K\cdot q}\sum_{k=1}^K(\bm{y}_k-\Omega(\bm{\theta};\bm{x}_k))^{\top}(\bm{y}_k-\Omega(\bm{\theta};\bm{x}_k)).
\end{equation} 

In order to limit the $D$-dimensional parameter space to a feasible region, we apply a penalty approach. Due to the length-invariance by the symmetry operators as shown in Eqn.~\eqref{eq:length-invar}, the feasible region is defined by a hypersphere. In case of $||\bm{\theta}|| > \sqrt{D}$, the error function~(\ref{eq:errorf}) is evaluated at a rescaled parameter vector $\frac{\bm{\theta}}{||\bm{\theta}||}$ and a penalty term $50(||\bm{\theta}||-\sqrt{D})$ is added to the error $\epsilon$.

In self-generated data sets, we add normal distributed noise with zero mean and variance $\sigma^2$ to the function values $y_i$
\begin{equation}
 y_i=f(\bm{x}_i)+\mu, \ \mu\sim \mathcal{N}(0, \sigma^2),\ \sigma=5\!\!\times\!\! 10^{-3}.
\end{equation}

\subsection{Experimental setup}
In all experiments, data is normalized such that mean is zero and variance is one. The population size $N_p$ used in DE and CMA-ES depends on the problem and the choice of the optimization method. Therefore, it is manually adapted accordingly. For each problem and each optimization method, we conduct 50 independent repetitions of the optimization process and record the error over the number of ANN-evaluations. To test for statistical significance of the obtained results, first the Kruskal-Wallis test~\cite{Hollander-nonpara} for the hypothesis that all performance means are equal is applied. In case this hypothesis is rejected, the Wilcoxon rank sum test~\cite{Wilcoxon45} is applied to all pairs of means to identify significantly different results. All tests are based on a significance level of $0.05$.
% In all experiments, following procedures are applied:
% \begin{itemize}
%  \item Data is normalized such that mean is zero and variance is one.
% \item The population size $N_p$ depends on the problem and is manually adapted accordingly.
% \item For each problem and each optimization method, 50 independent repetitions are conducted.
% \item 
% \end{itemize}
In Table~\ref{tab:summary-res}, normalized training set errors for the regression and the autoencoding problems, and normalized test set errors for the classification problems are shown.
\begin{table}[]\label{tab:summary-res}
\caption{\it Normalized training set errors for the regression and the autoencoding problems, and normalized test set errors for the classification problems. The best results are printed in {\bf boldface}. For each problem and method, errors are normalized by the maximum error from within the corresponding regular method, its extension by global optimization invariant symmetry breaking and its extension by global optimization variant symmetry breaking.}
\center
{\tiny
\begin{tabular}{l|c|c|c||c|c|c}
 & DE & DE-INV-SB & DE-SB & CMA-ES & CMA-ES-INV-SB & CMA-ES-SB\\\noalign{\hrule height 1.25pt}
{\bf syn5}  & 0.958 $\pm$ 0.079 & 1.000 $\pm$ 0.186 & {\bf 0.949 $\pm$ 0.039} & 1.000 $\pm$ 4.446 & 0.386 $\pm$ 0.469 & {\bf 0.093 $\pm$ 0.007}\\\hline
{\bf sinc} & 1.000 $\pm$ 0.859 & 0.412 $\pm$ 0.166 & {\bf 0.114 $\pm$ 0.008} & 0.459 $\pm$ 0.271 & 1.000 $\pm$ 0.784 & {\bf 0.139 $\pm$ 0.051}\\\hline
{\bf inc-sinc} & 1.000 $\pm$ 0.963 & 0.337 $\pm$ 0.155 & {\bf 0.089 $\pm$ 0.016} & 0.287 $\pm$ 0.336 & 1.000 $\pm$ 0.707 & {\bf 0.082 $\pm$ 0.035}\\\hline
{\bf sinc2d} & 1.000 $\pm$ 0.387 & 0.995 $\pm$ 0.094 & {\bf 0.875 $\pm$ 0.029} & 0.975 $\pm$ 0.139 & 1.000 $\pm$ 0.241 & {\bf 0.089 $\pm$ 0.253}\\\hline
{\bf sinc3d} & 0.622 $\pm$ 0.029 & 1.000 $\pm$ 0.572 & {\bf 0.603 $\pm$ 0.033} & 1.000 $\pm$ 1.401 & 0.090 $\pm$ 0.013 & {\bf 0.043 $\pm$ 0.021}\\\noalign{\hrule height 1.25pt}
{\bf autoenc-circle} & 0.057 $\pm$ 0.082 & 1.000 $\pm$ 1.850 & {\bf 0.020 $\pm$ 0.030} & 1.000 $\pm$ 0.295 & 0.626 $\pm$ 0.548 & {\bf 0.077 $\pm$ 0.164}\\\hline
{\bf autoenc-spiral} & 0.341 $\pm$ 0.545 & 1.000 $\pm$ 0.932 & {\bf 0.116 $\pm$ 0.308} & 0.248 $\pm$ 0.232 & 1.000 $\pm$ 0.882 & {\bf 0.030 $\pm$ 0.024}\\\hline
{\bf autoenc-sphere} & 0.554 $\pm$ 0.321 & 1.000 $\pm$ 0.064 & {\bf 0.022 $\pm$ 0.012}
 & 0.050 $\pm$ 0.012 & 1.000 $\pm$ 0.416 & {\bf 0.032 $\pm$ 0.008}\\\noalign{\hrule height 1.25pt}
{\bf two-circles} & 0.450 $\pm$ 0.225 & 1.000 $\pm$ 0.182 & {\bf 0.269 $\pm$ 0.074} & 0.635 $\pm$ 0.368 & 1.000 $\pm$ 0.284 & {\bf 0.326 $\pm$ 0.169}\\\hline
{\bf two-spirals} & 0.918 $\pm$ 0.260 & 1.000 $\pm$ 0.213 & {\bf 0.426 $\pm$ 0.197}
 & 1.000 $\pm$ 0.228 & 0.930 $\pm$ 0.201 & {\bf 0.683 $\pm$ 0.293}\\\hline
{\bf digits} & 0.325 $\pm$ 0.087 & 1.000 $\pm$ 0.111 & {\bf 0.272 $\pm$ 0.062} & 1.000 $\pm$ 0.352 & 0.805 $\pm$ 0.113 & {\bf 0.668 $\pm$ 0.099}
\end{tabular}}
\end{table}

\clearpage
\subsection{Regression problems}\label{sec:regressionproblems}
As in~\cite{castillo06very}, we apply learning only on a training set to compare the performance of the introduced methods. In the following, the regression problems are introduced and corresponding results are shown.
\subsubsection{Dataset {\bf syn5}}
The {\bf syn5} dataset is generated by the fourth-degree polynome $(x-0.5)^2\cdot(0.1+(x+0.65)^2)$ with uniformly distributed random input values $x_i\in(-1,1)$. We use a 1-3-1 net and 200 data samples. The population size for all DE-based methods is $N_p=80$, and $N_p=48$ for all CMA-ES-based methods. Fig.~\ref{fig:syn5-res} shows the resulting convergence curves and box plots for the learning process.
%%%%%%%%%%%%%%%%%%%%%%%%%%%%%%%%%%%%%%%%%%%%%%%%%%%%%%%%%%%%%%%%%%%%%%%%%%%%%%%%%%%%%%%%%%%%%%%%%%%%%%%%%%%%%%%%
% plot
\begin{figure*}[h!]
   \begin{center}
      \scalebox{0.3}{\includegraphics{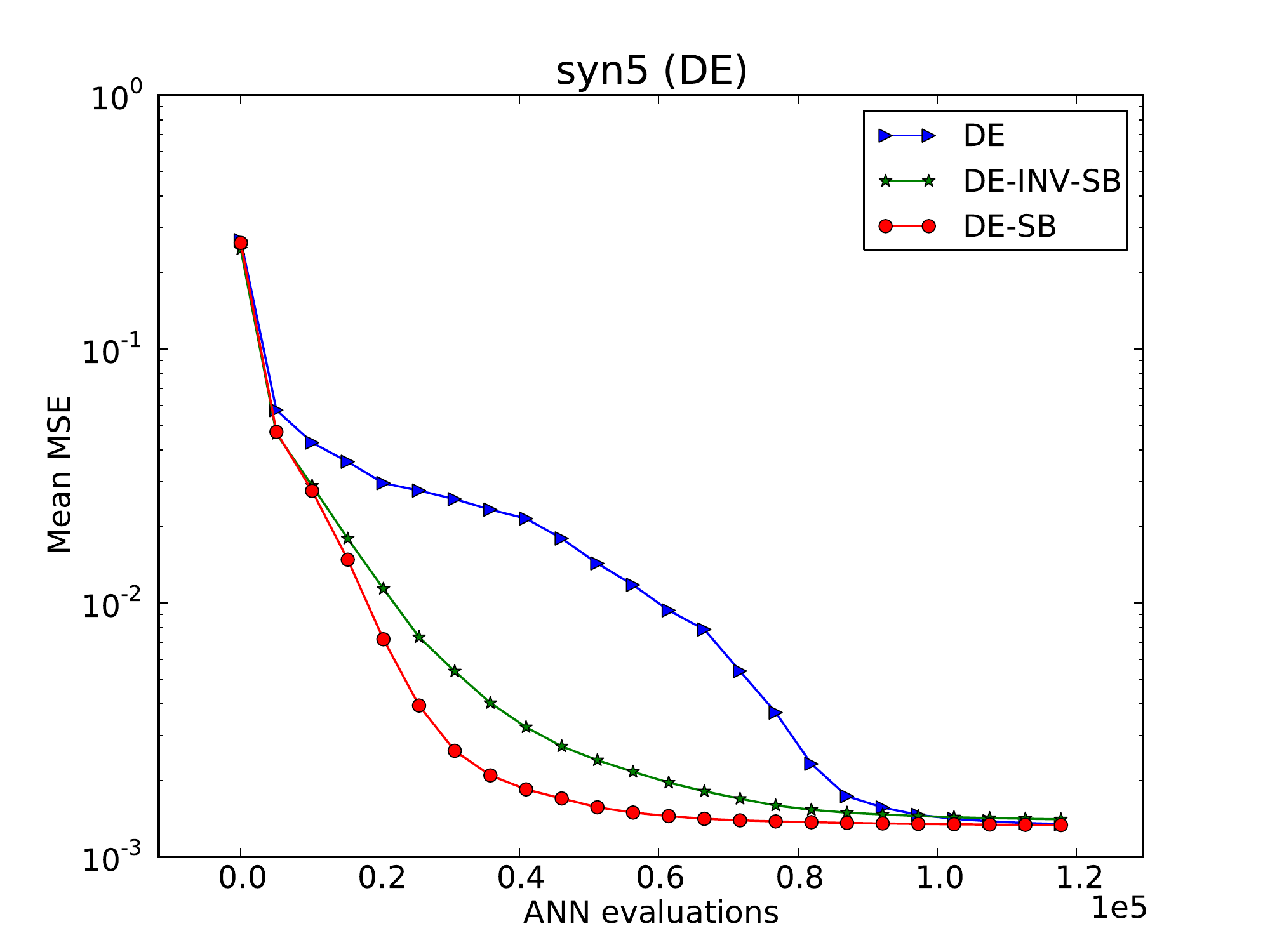}}
      \scalebox{0.3}{\includegraphics{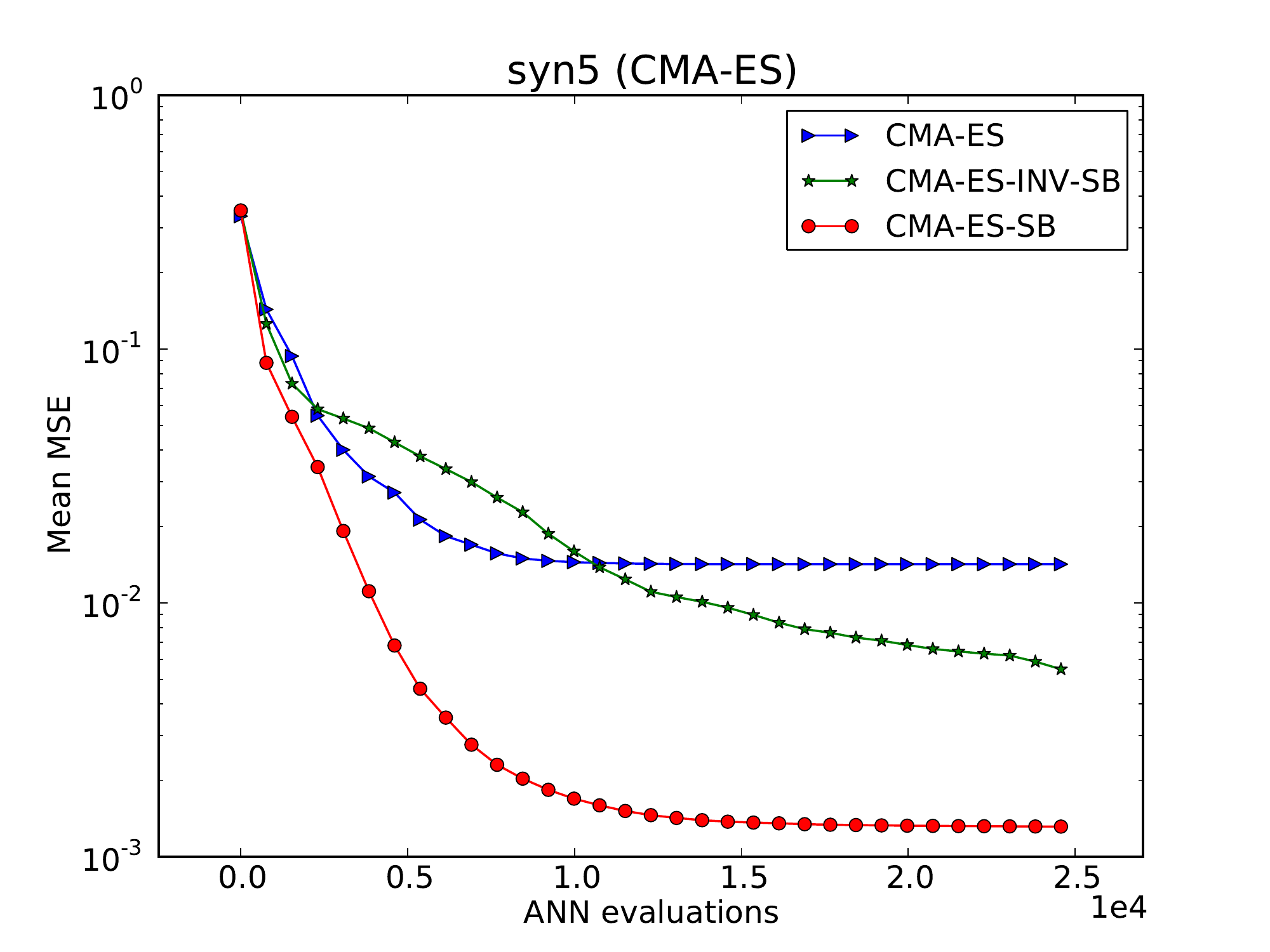}}
      \scalebox{0.3}{\includegraphics{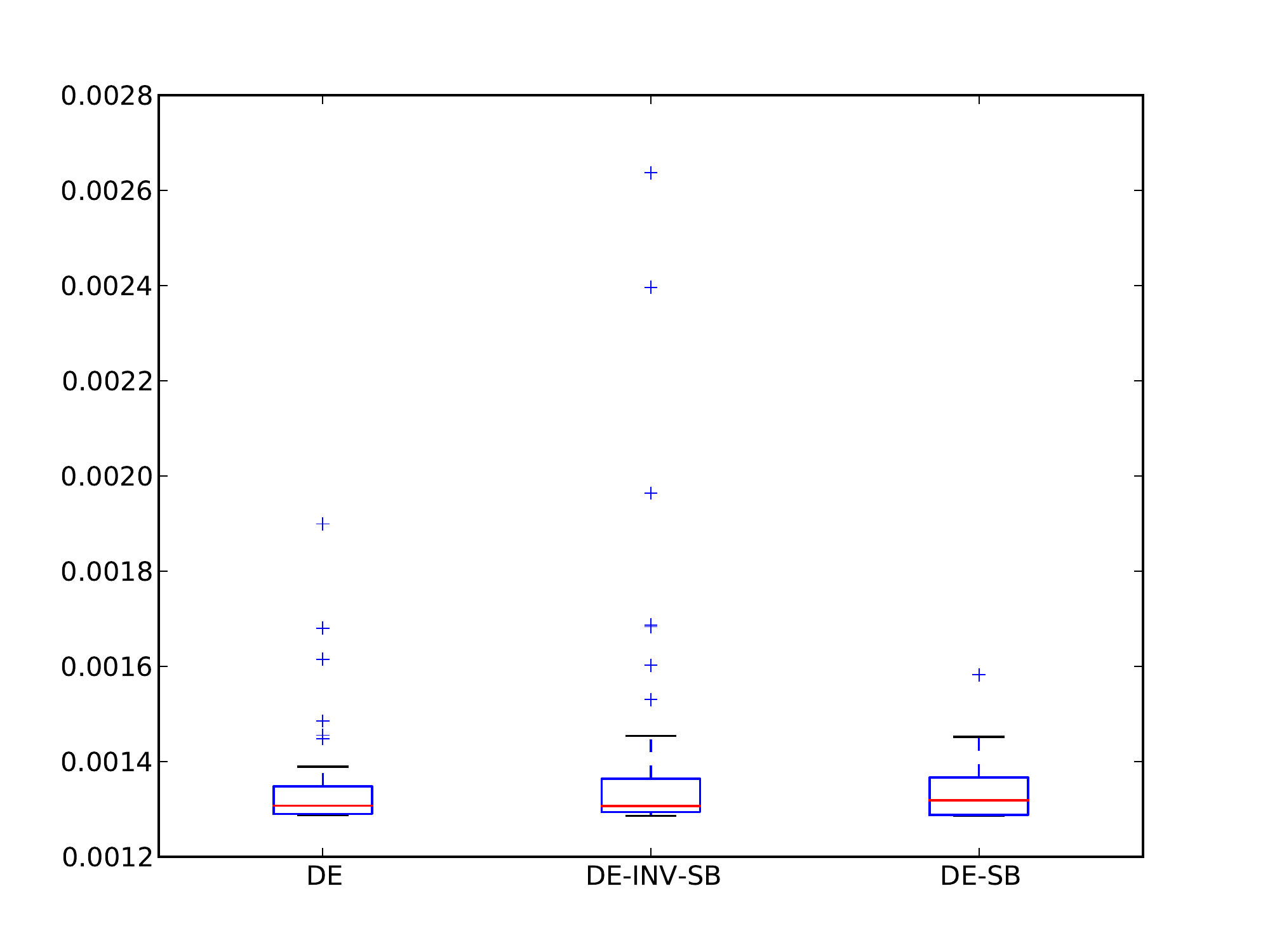}}
      \scalebox{0.3}{\includegraphics{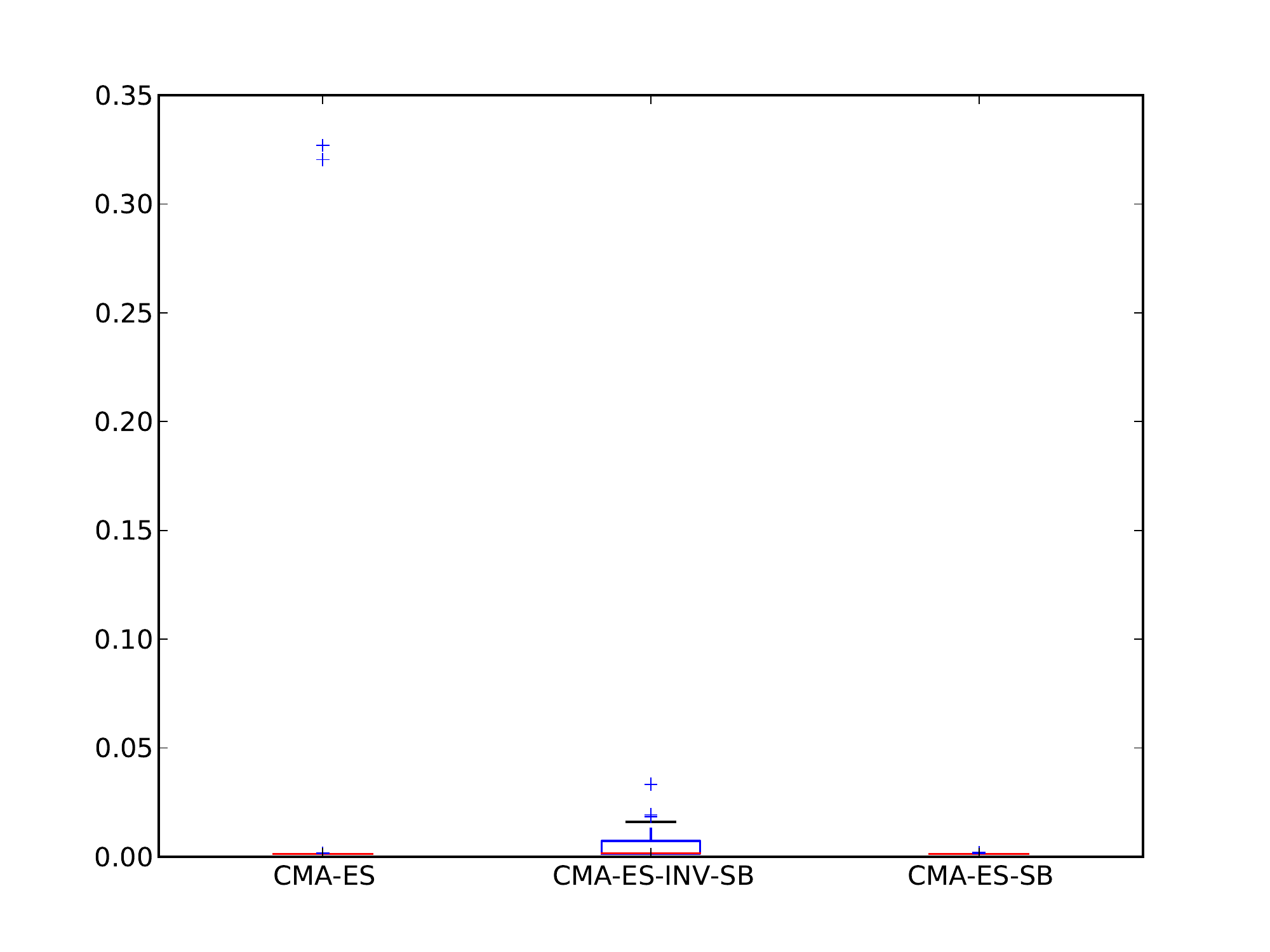}}
      \caption{\label{fig:syn5-res} \it Convergence curves for regression by DE (left) and CMA-ES (right) using the {\bf syn5} dataset.}
   \end{center}
\end{figure*}
For the DE-family, the Kruskal-Wallis test showed no significant difference in means. In contrast, according to the Wilcoxon tests, the inequality of means of CMA-ES and CMA-ES-SB is rejected by a narrow margin, with a corresponding p-value of $0.08$. The other means are significantly different. All DE variants reach the same low-error, where DE-SB shows the fastest decrease in error. As for the CMA-ES variants, CMA-ES fails to reach a low error in a few runs, which leads to a larger mean error in average. In contrast, CMA-ES-SB proves to be more robust and reaches a relatively low error in all runs.
\clearpage
\subsubsection{Dataset {\bf sinc}}
The {\bf sinc} dataset is generated by the function $\frac{\sin(10x)}{10x}$ with uniformly distributed random input values $x_i\in(-1,1)$. We use a 1-5-1 net and 200 data samples. The population size for all DE-based methods is $N_p=120$, and $N_p=400$ for all CMA-ES-based methods. Fig.~\ref{fig:sinc-res} shows the resulting convergence curves and box plots for the learning process.
%%%%%%%%%%%%%%%%%%%%%%%%%%%%%%%%%%%%%%%%%%%%%%%%%%%%%%%%%%%%%%%%%%%%%%%%%%%%%%%%%%%%%%%%%%%%%%%%%%%%%%%%%%%%%%%%
% plot
\begin{figure*}[h!]
   \begin{center}
      \scalebox{0.35}{\includegraphics{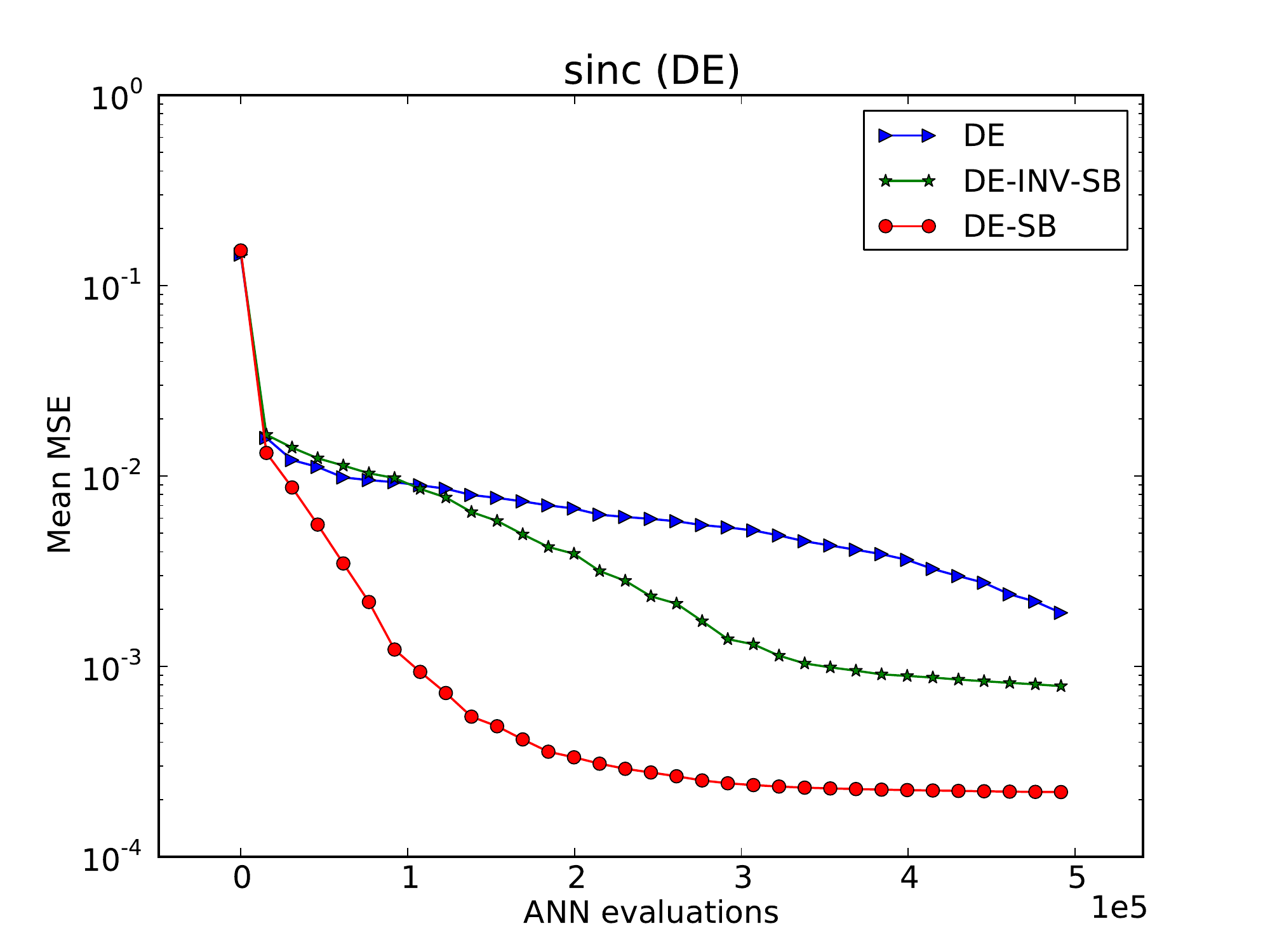}}
      \scalebox{0.35}{\includegraphics{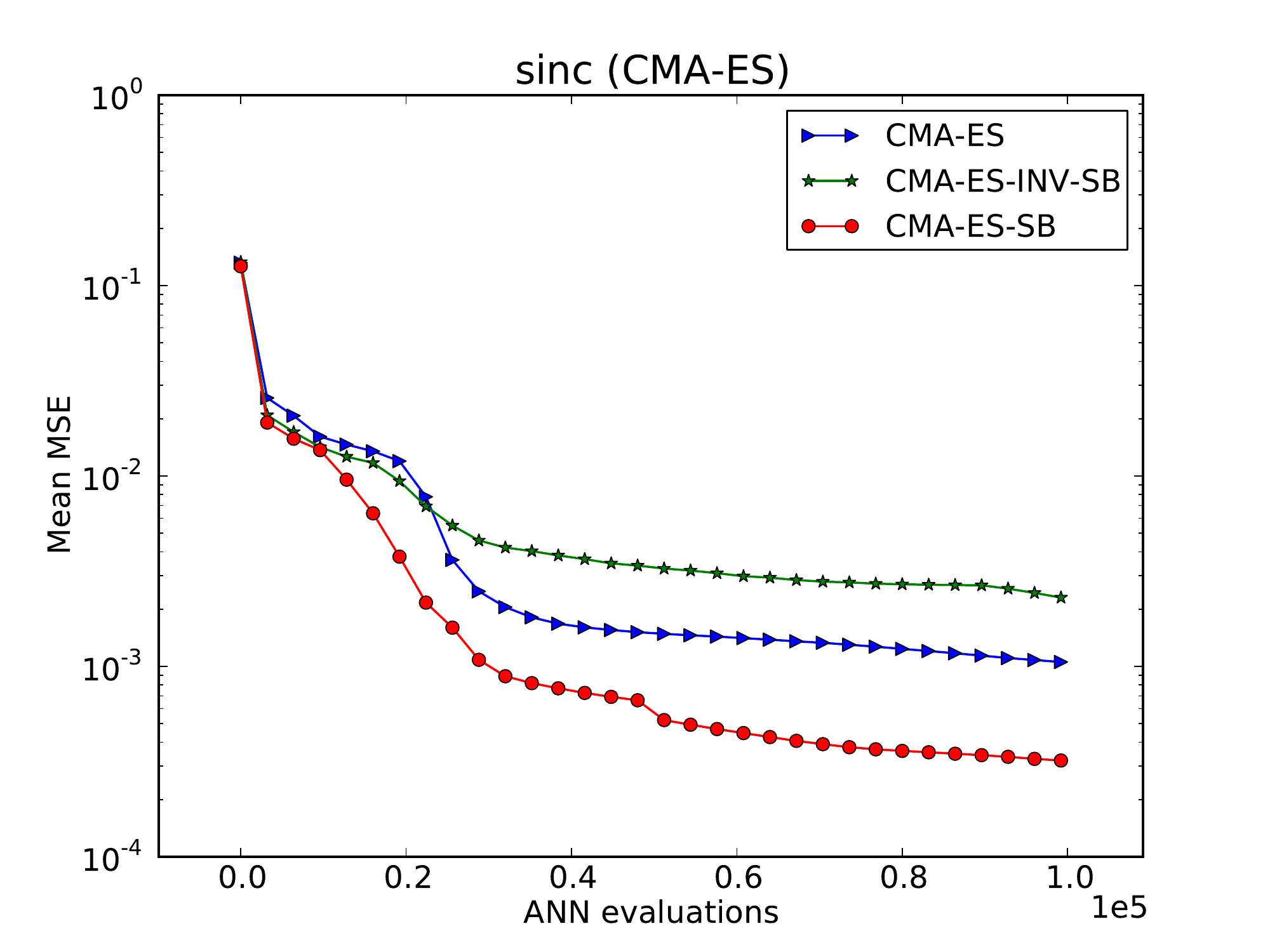}}
      \scalebox{0.35}{\includegraphics{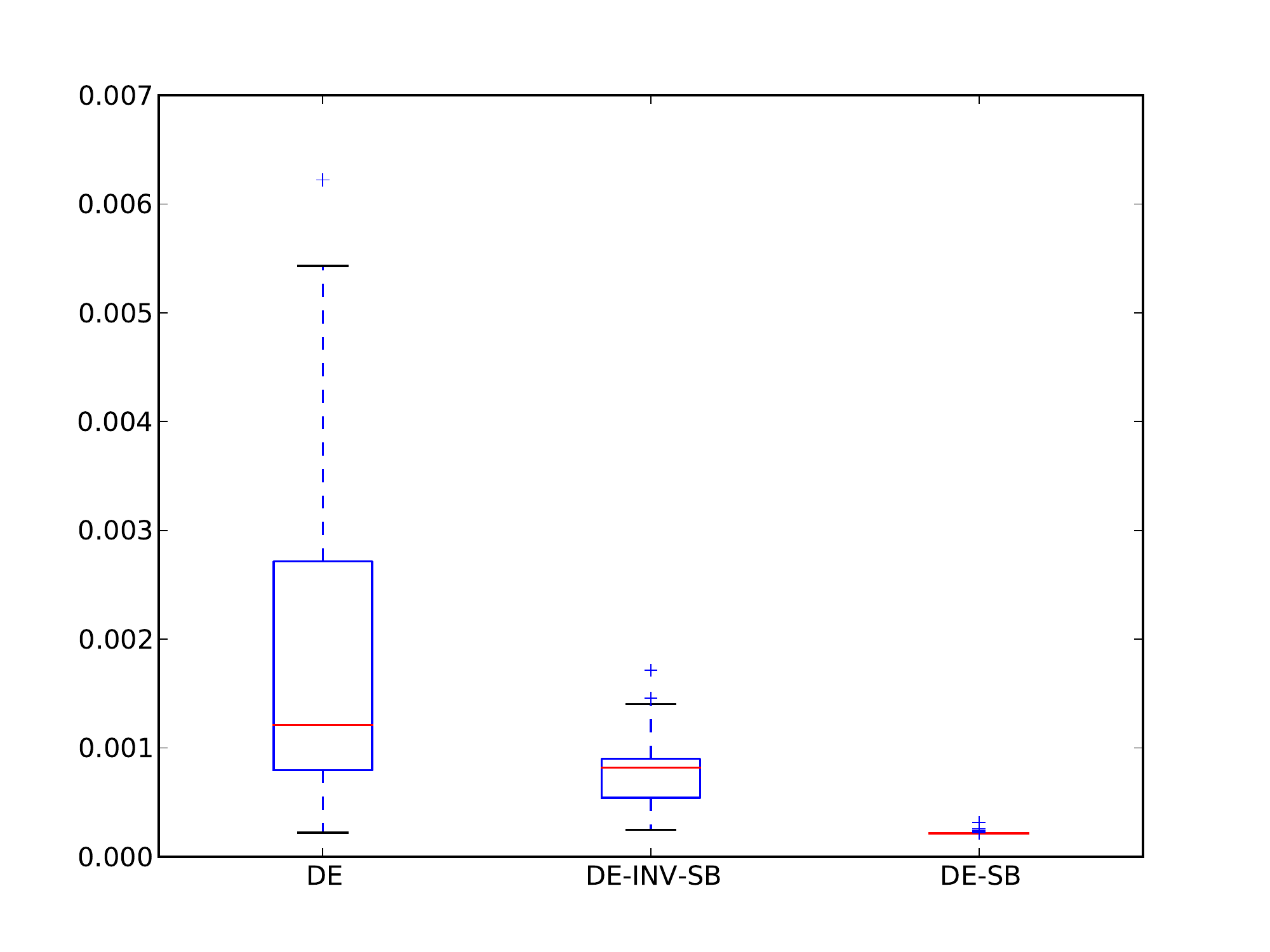}}
      \scalebox{0.35}{\includegraphics{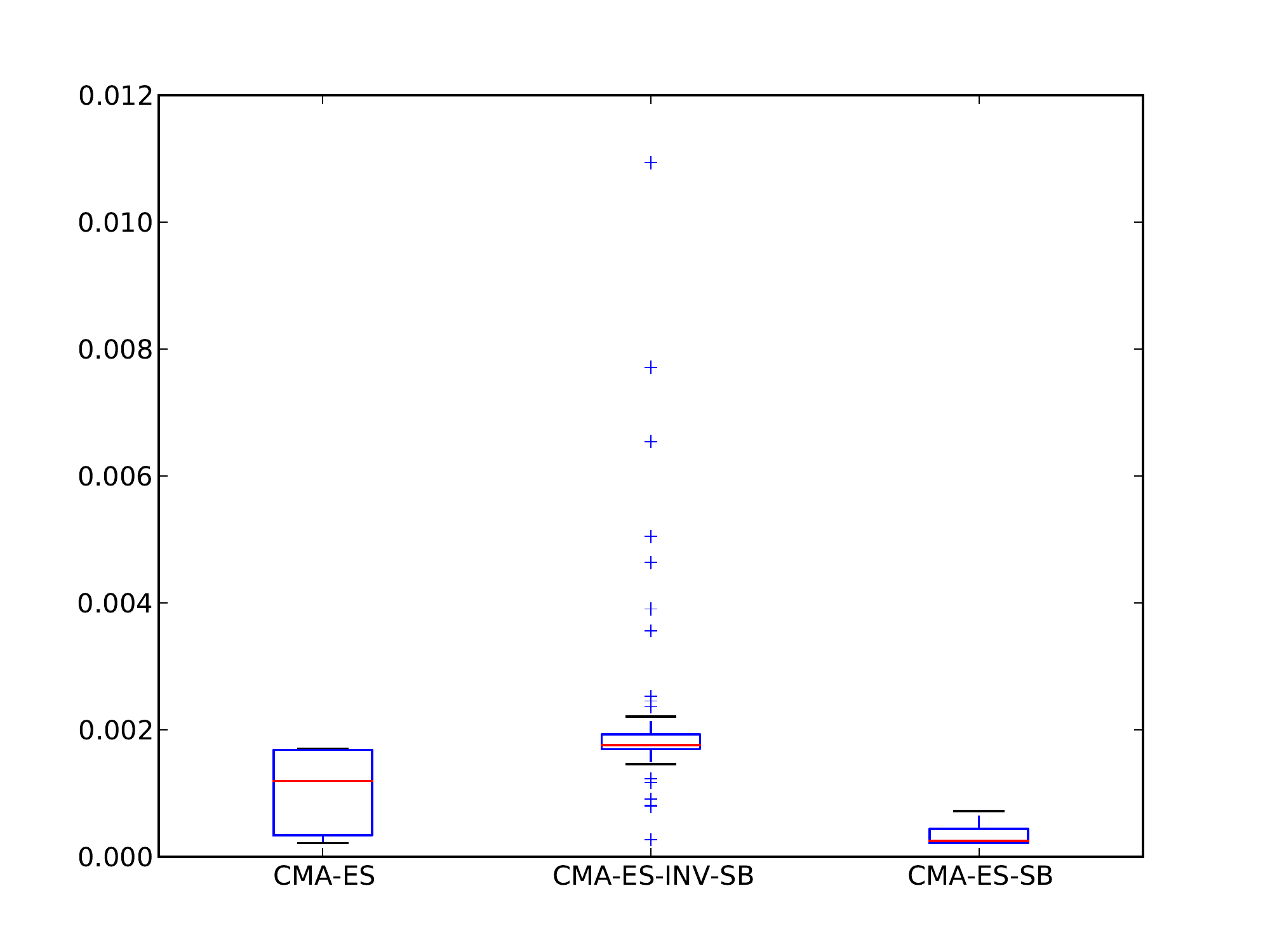}}
      \caption{\label{fig:sinc-res} \it Convergence curves for regression by DE (left) and CMA-ES (right) using the {\bf sinc} dataset.}
   \end{center}
\end{figure*}
According to the Wilcoxon tests, all pairwise differences are significant. DE-SB clearly outperforms DE and DE-INV-SB. Similarly, CMA-ES-SB is the fastest among the CMA-ES-based methods.

\clearpage
\subsubsection{Dataset {\bf inc-sinc}}
The {\bf inc-sinc} dataset is generated by the function $\frac{x}{2}+\frac{\sin(10x)}{10x}$ with uniformly distributed random input values $x_i\in(-1,1)$. We use a 1-5-1 net and 200 data samples. The population size for all DE-based methods is $N_p=144$, and $N_p=400$ for all CMA-ES-based methods. Fig.~\ref{fig:inc-sinc-res} shows the resulting convergence curves and box plots for the learning process. %%%%%%%%%%%%%%%%%%%%%%%%%%%%%%%%%%%%%%%%%%%%%%%%%%%%%%%%%%%%%%%%%%%%%%%%%%%%%%%%%%%%%%%%%%%%%%%%%%%%%%%%--1-5-1-convergence%%%%%%%%%
% plot
\begin{figure*}[h!]
   \begin{center}
      \scalebox{0.35}{\includegraphics{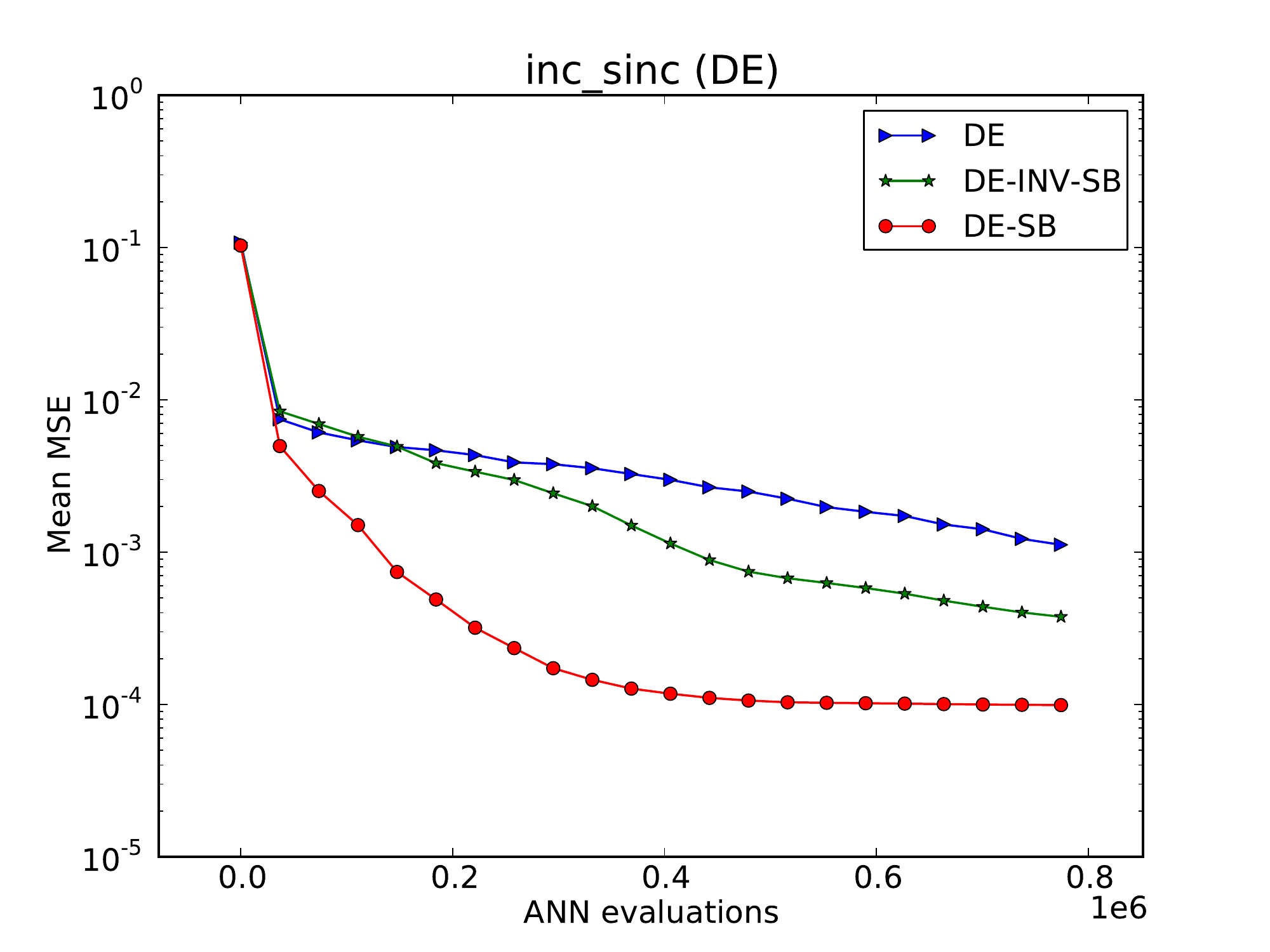}}
      \scalebox{0.35}{\includegraphics{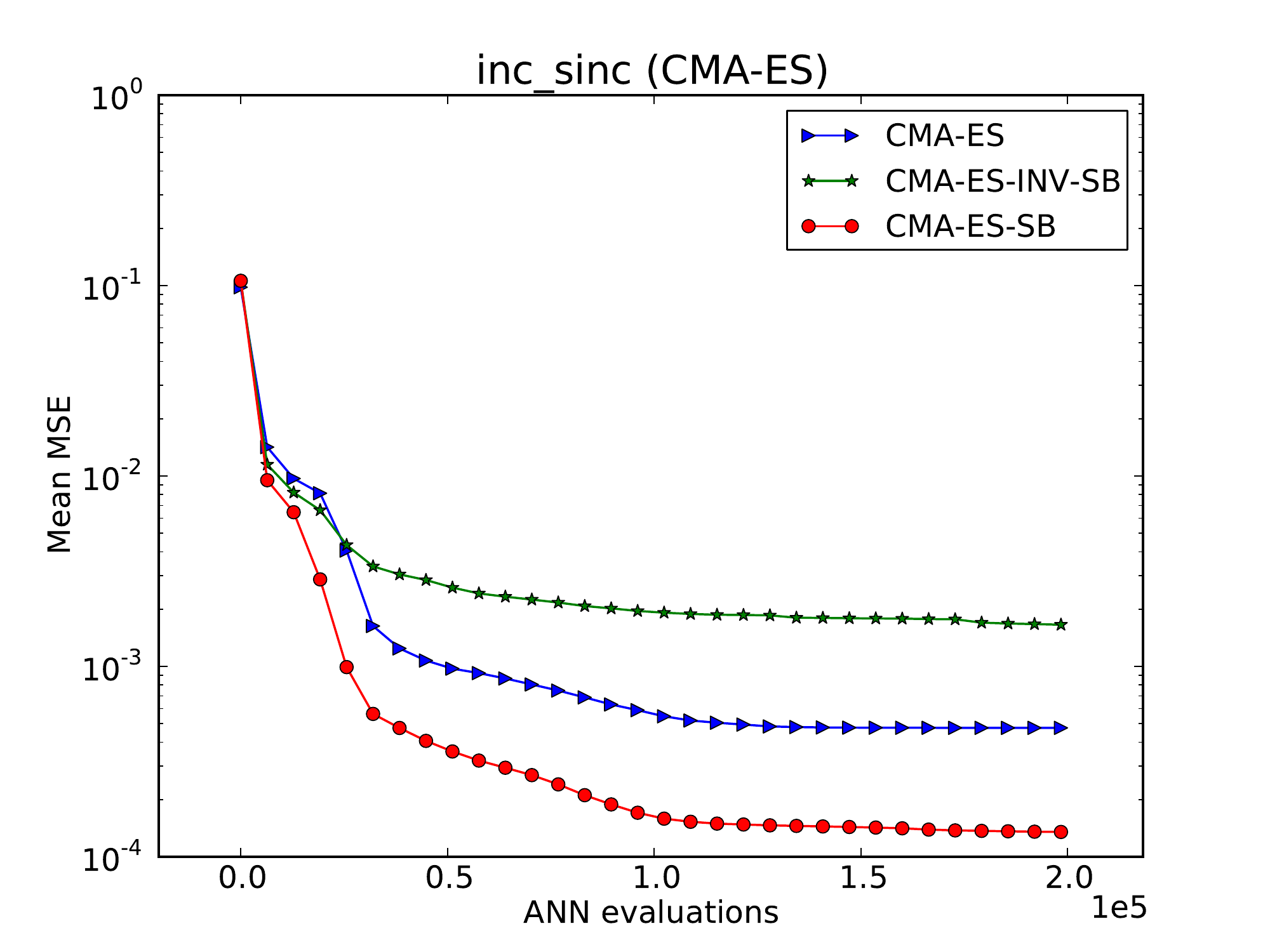}}
      \scalebox{0.35}{\includegraphics{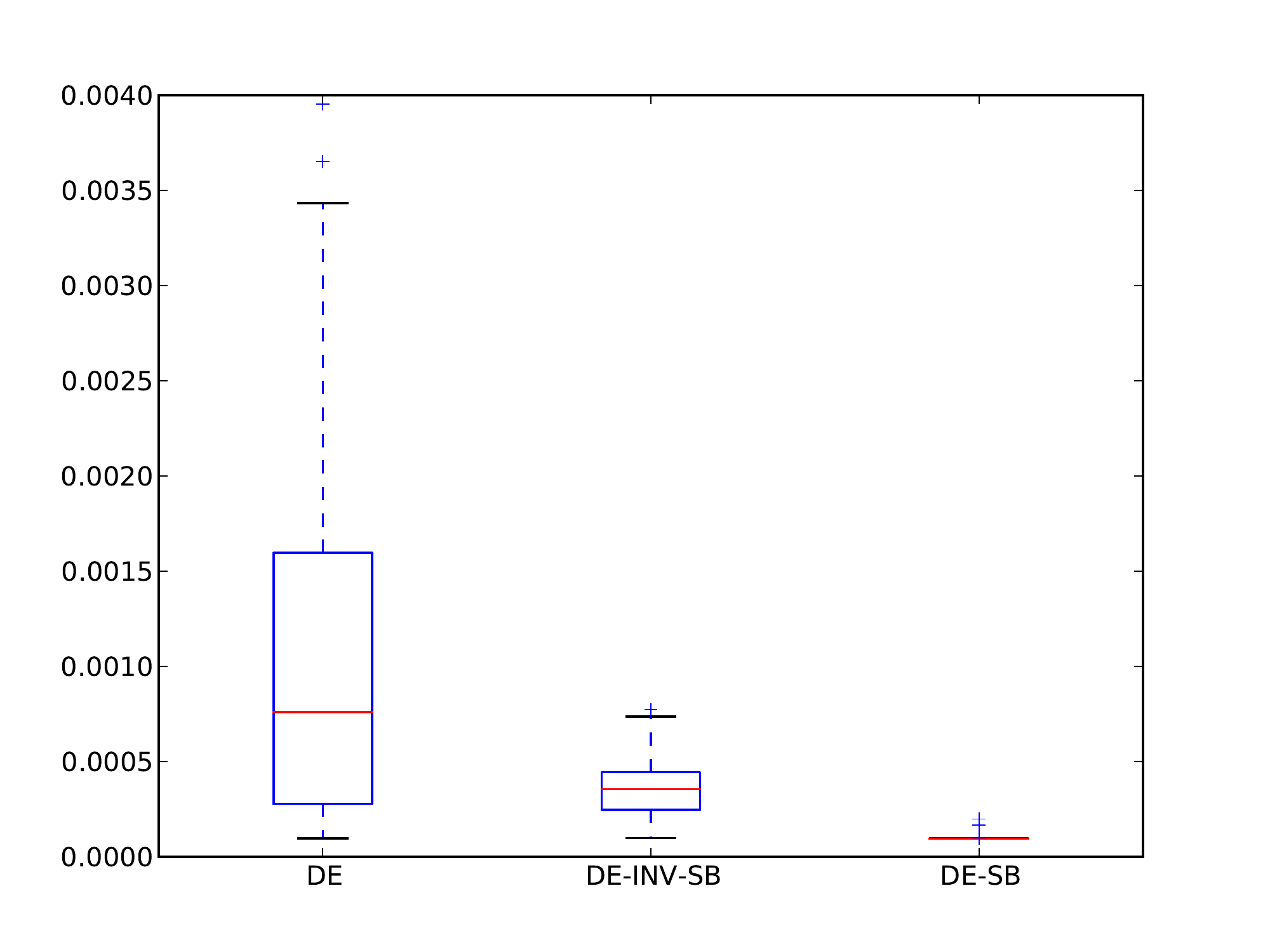}}
      \scalebox{0.35}{\includegraphics{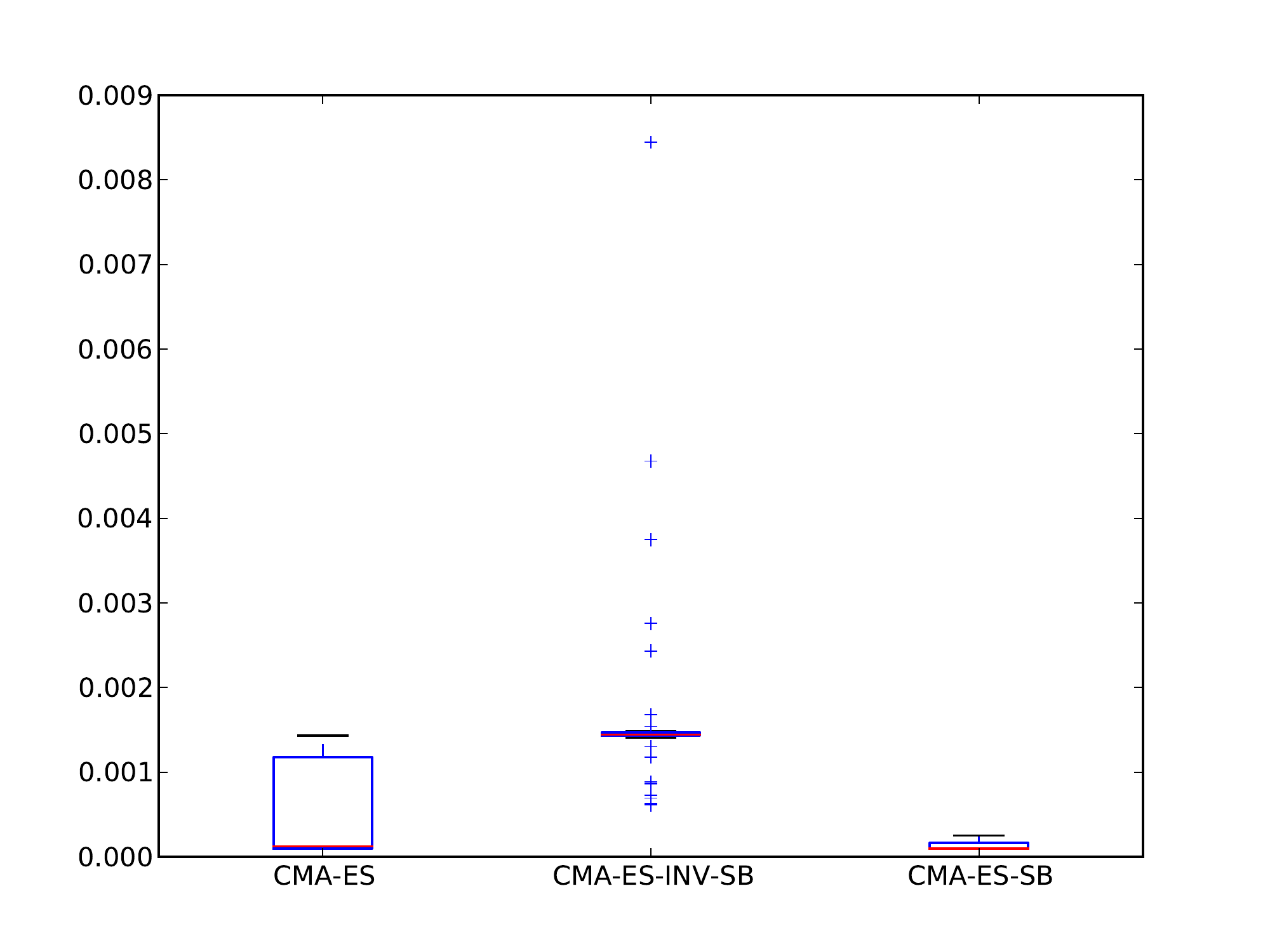}}
      \caption{\label{fig:inc-sinc-res} \it Convergence curves for regression by DE (left) and CMA-ES (right) using the {\bf inc-sinc} dataset.}
   \end{center}
\end{figure*}
According to the Wilcoxon tests, all pairwise differences are significant. Interestingly, the global optimum invariant symmetry breaking approach leads to in improvement for DE (DE-INV-SB), but shows inferior performance on CMA-ES (CMA-ES-INV-SB). This proves that symmetry breaking approaches should be specific to the selected global optimization method. Again, DE-SB and CMA-ES-SB are the fastest methods.
\clearpage
\subsubsection{Dataset {\bf sinc2d}}
The {\bf sinc2d} dataset is generated by the function $\frac{\sin(5||x||)}{15||x||}$ with uniformly distributed random input values $x_i\in(-1,1)^2$. We use a 2-3-1-3-1 net and 1000 data samples. The population size for all DE-based methods is $N_p=96$, and $N_p=1000$ for all CMA-ES-based methods. Fig.~\ref{fig:sinc2d-res} shows the resulting convergence curves and box plots for the learning process.
%%%%%%%%%%%%%%%%%%%%%%%%%%%%%%%%%%%%%%%%%%%%%%%%%%%%%%%%%%%%%%%%%%%%%%%%%%%%%%%%%%%%%%%%%%%%%%%%%%%%%%%%%%%%%%%%
% plot
\begin{figure*}[h!]
   \begin{center}
      \scalebox{0.35}{\includegraphics{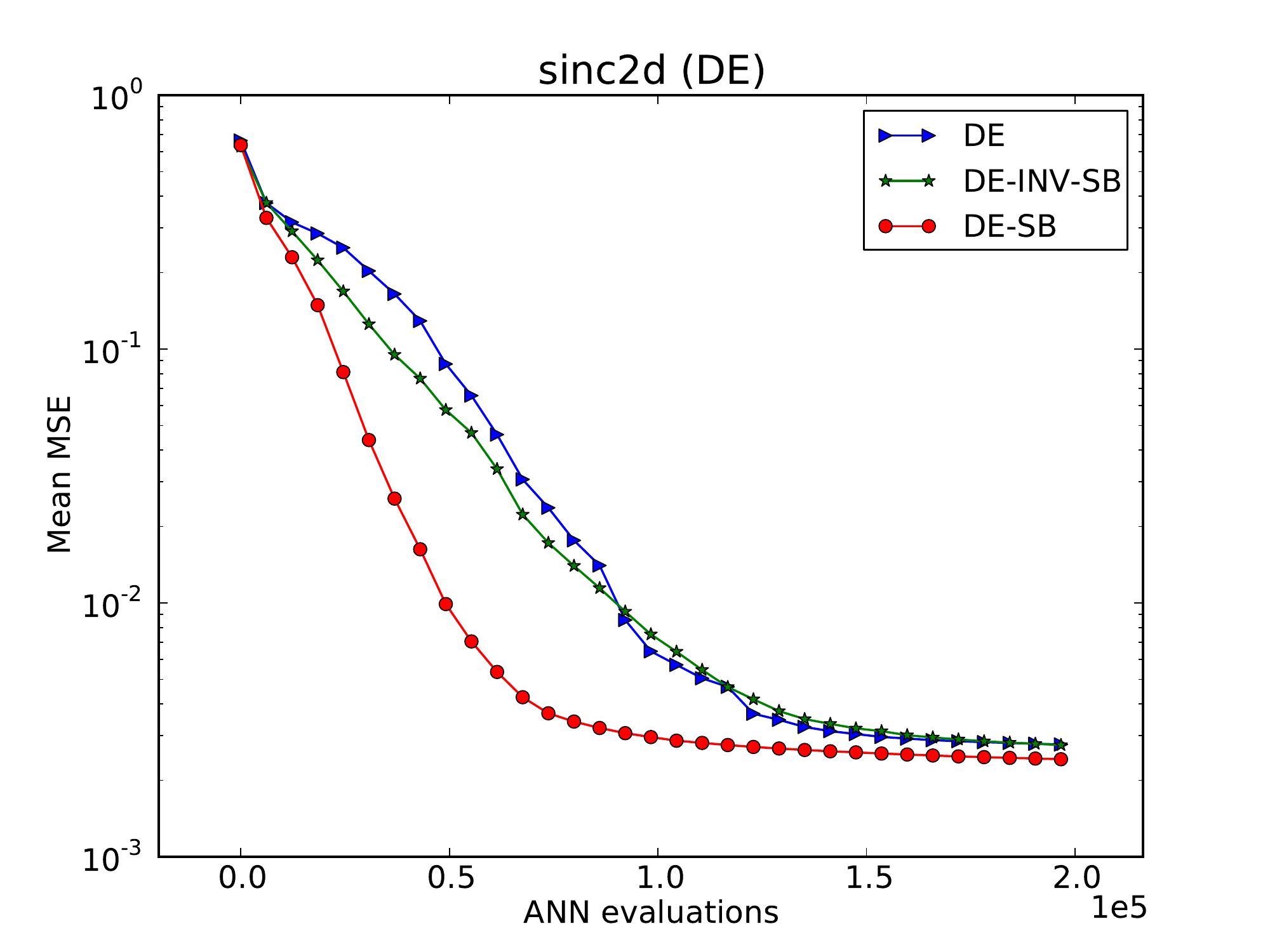}}
      \scalebox{0.35}{\includegraphics{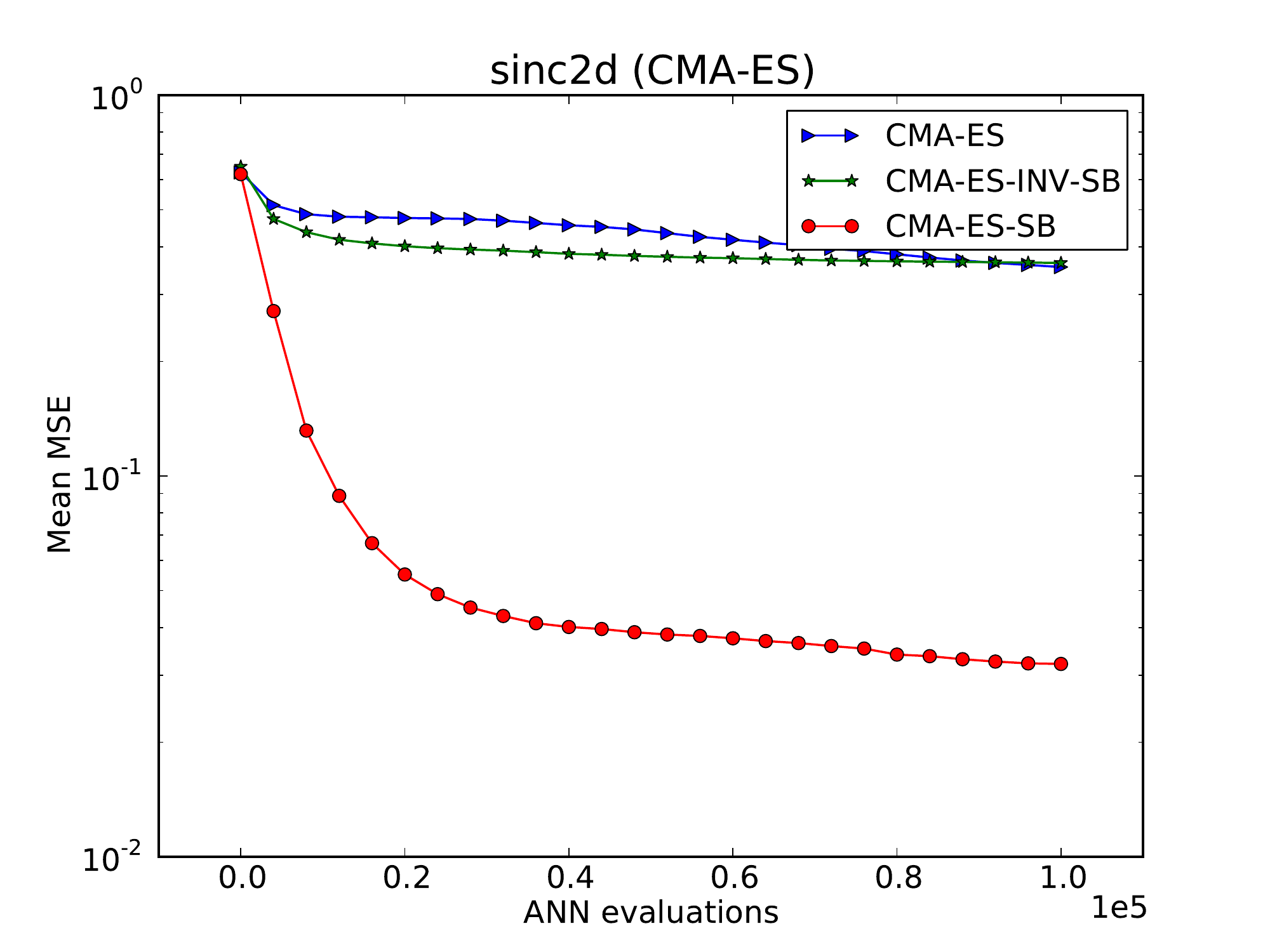}}
      \scalebox{0.35}{\includegraphics{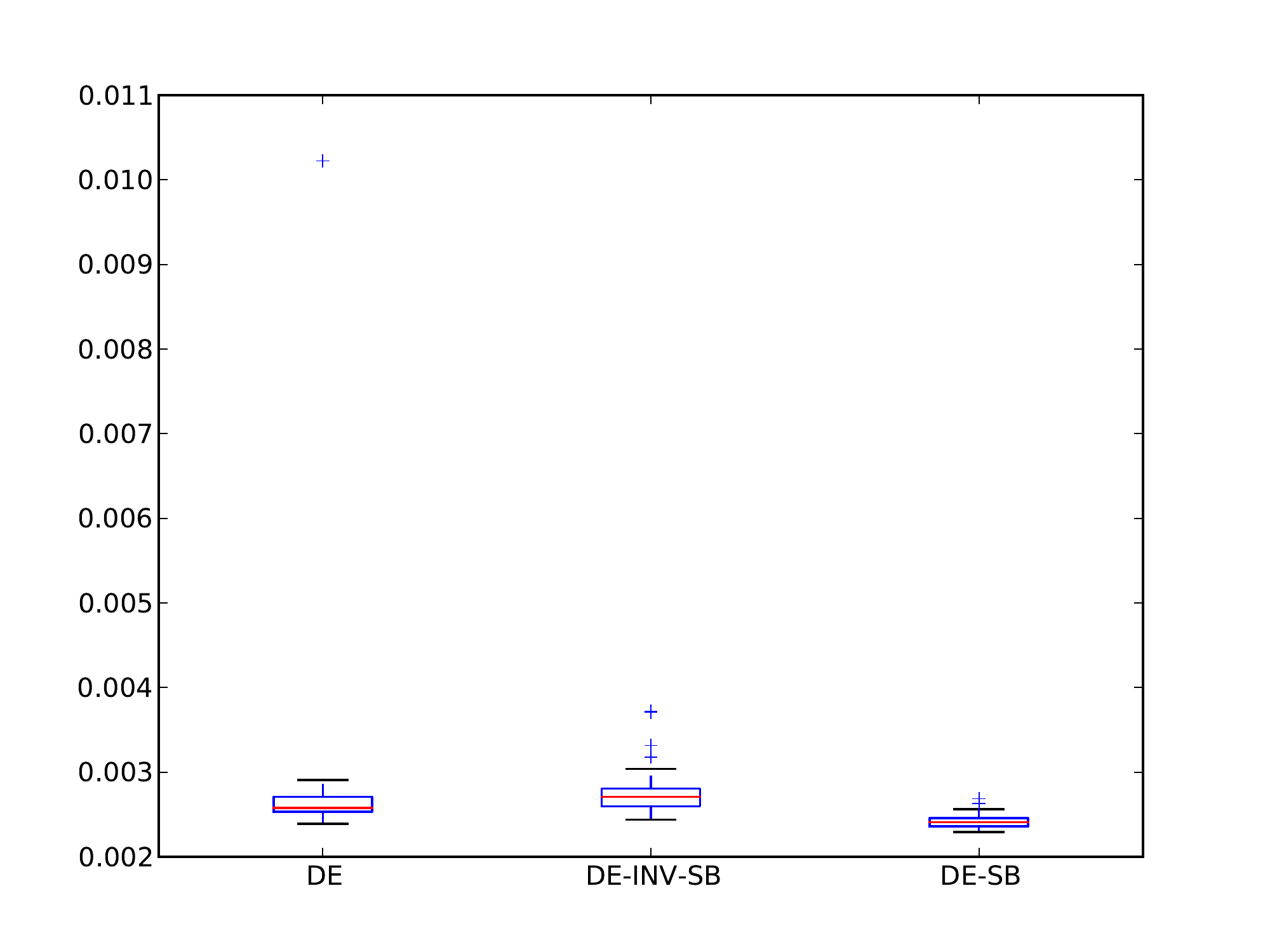}}
      \scalebox{0.35}{\includegraphics{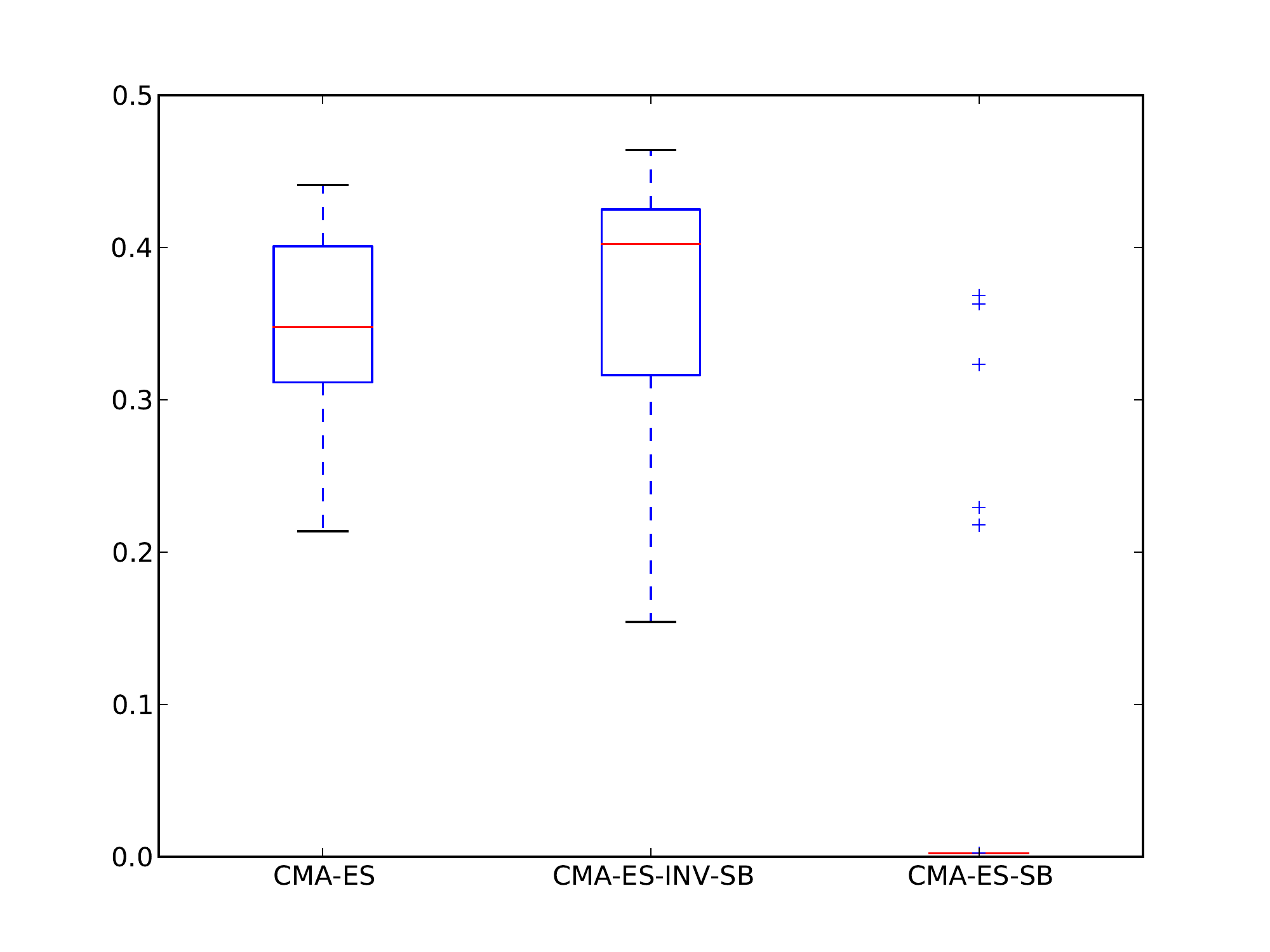}}
      \caption{\label{fig:sinc2d-res} \it Convergence curves for regression by DE (left) and CMA-ES (right) using the {\bf sinc2d} dataset.}
   \end{center}
\end{figure*}
According to the Wilcoxon tests, all pairwise differences are significant, except the difference between CMA-ES and CMA-ES-INV-SB.
The proposed symmetry breaking approach shows a very clear impact on the CMA-ES-variants. While CMA-ES and CMA-ES-INV-SB fail to solve this problem completely, CMA-ES-SB successfully trains the ANN in the majority of the 50 runs.
\clearpage
\subsubsection{Dataset {\bf sinc3d}}
The {\bf sinc3d} dataset is generated by the function $\frac{\sin(5||x||)}{15||x||}$ with uniformly distributed random input values $x_i\in(-1,1)^3$. We use a 3-4-1-4-1 net and 1000 data samples. The population size for all DE-based methods is $N_p=120$, and $N_p=1000$ for all CMA-ES-based methods. Fig.~\ref{fig:sinc3d-res} shows the resulting convergence curves and box plots for the learning process.
%%%%%%%%%%%%%%%%%%%%%%%%%%%%%%%%%%%%%%%%%%%%%%%%%%%%%%%%%%%%%%%%%%%%%%%%%%%%%%%%%%%%%%%%%%%%%%%%%%%%%%%%%%%%%%%%
% plot
\begin{figure*}[h!]
   \begin{center}
      \scalebox{0.35}{\includegraphics{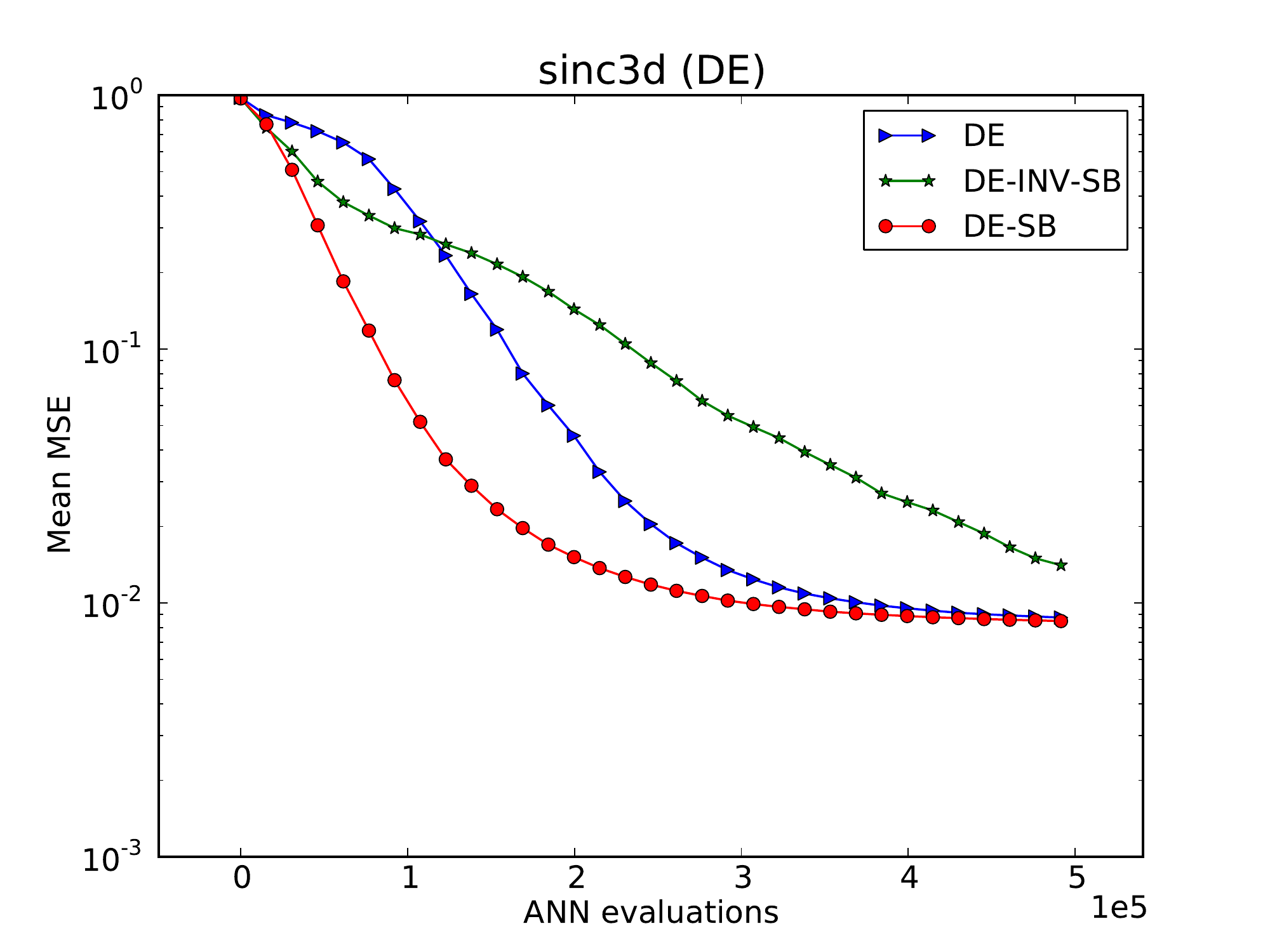}}
      \scalebox{0.35}{\includegraphics{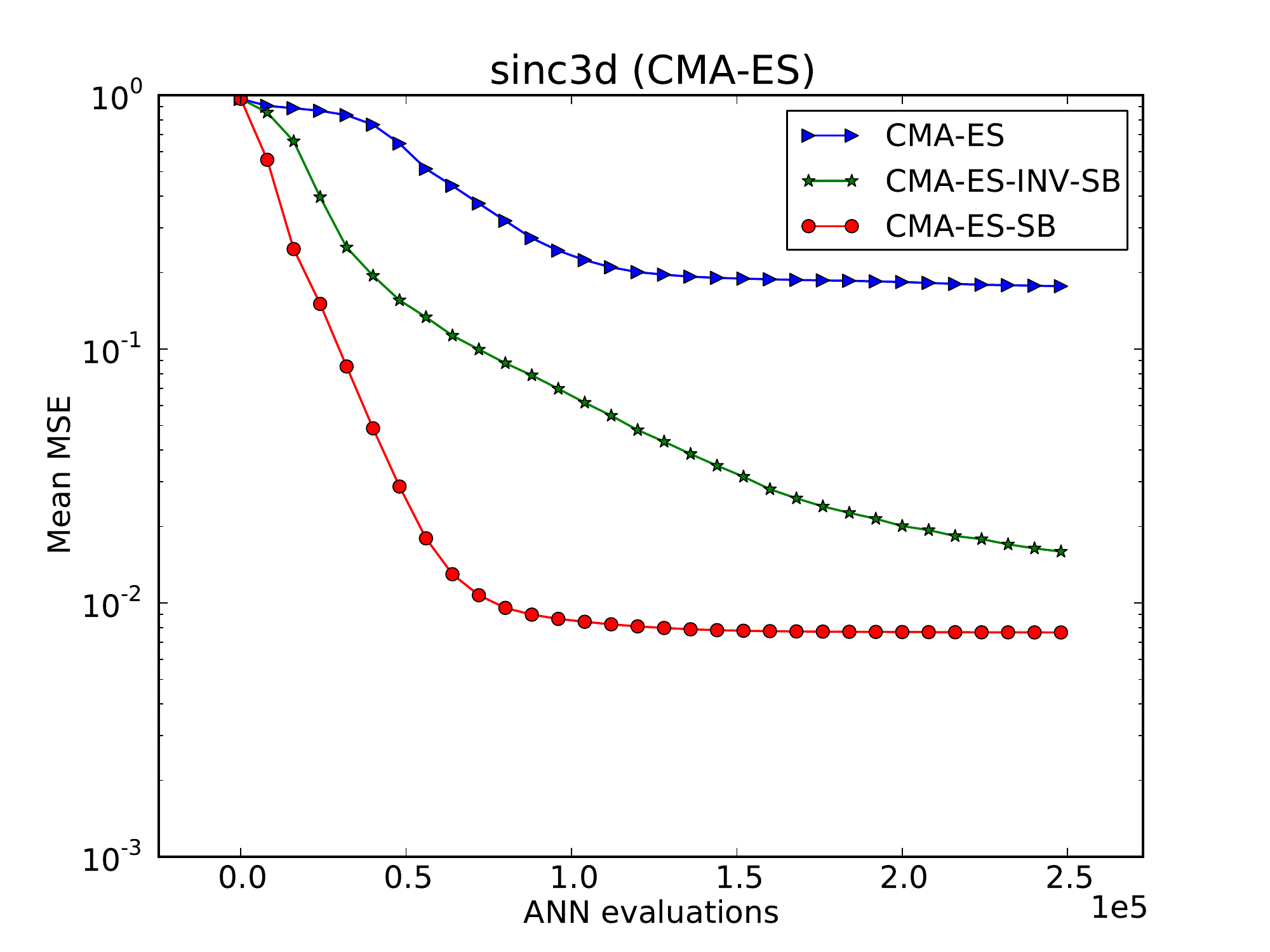}}
      \scalebox{0.35}{\includegraphics{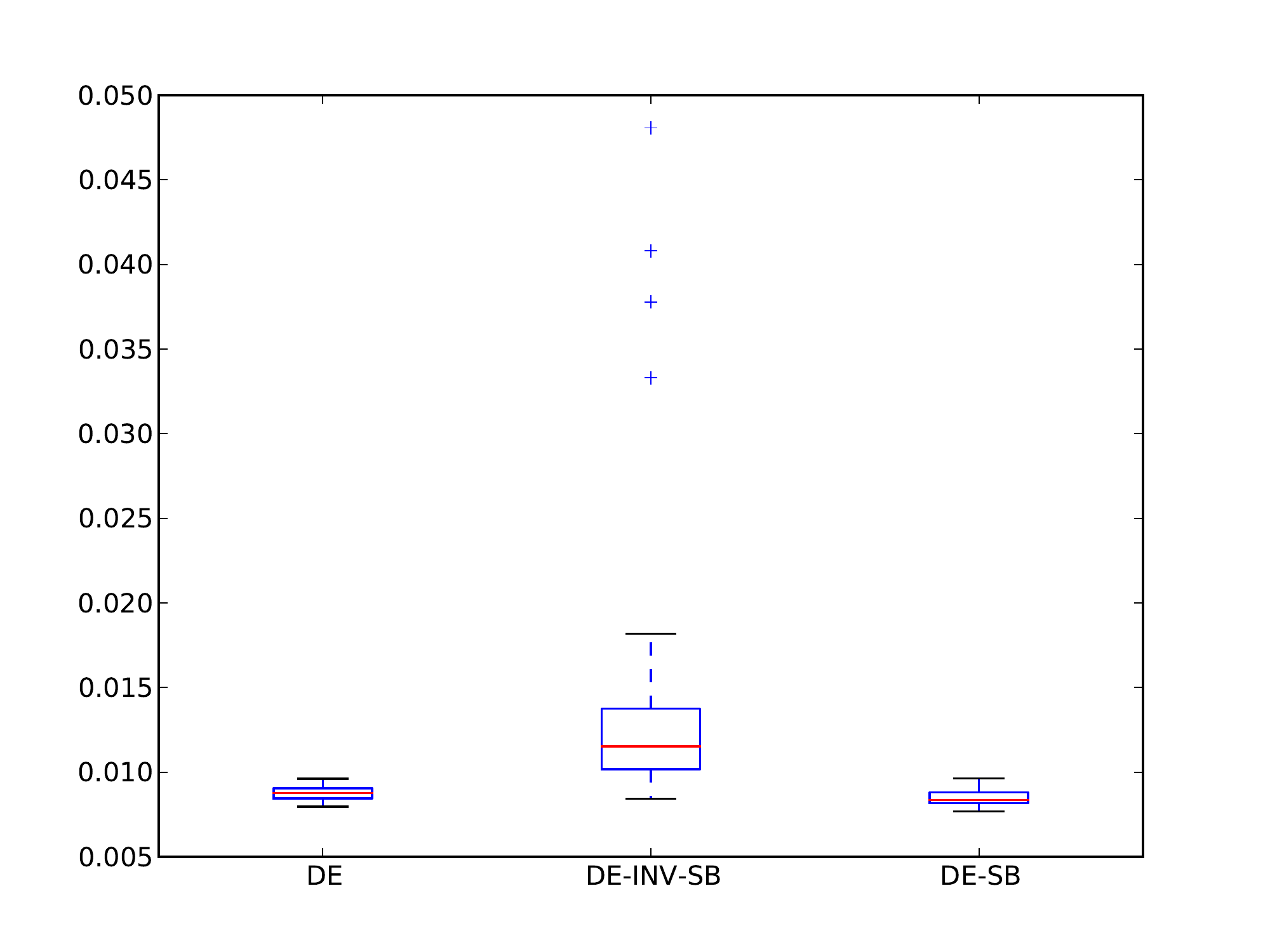}}
      \scalebox{0.35}{\includegraphics{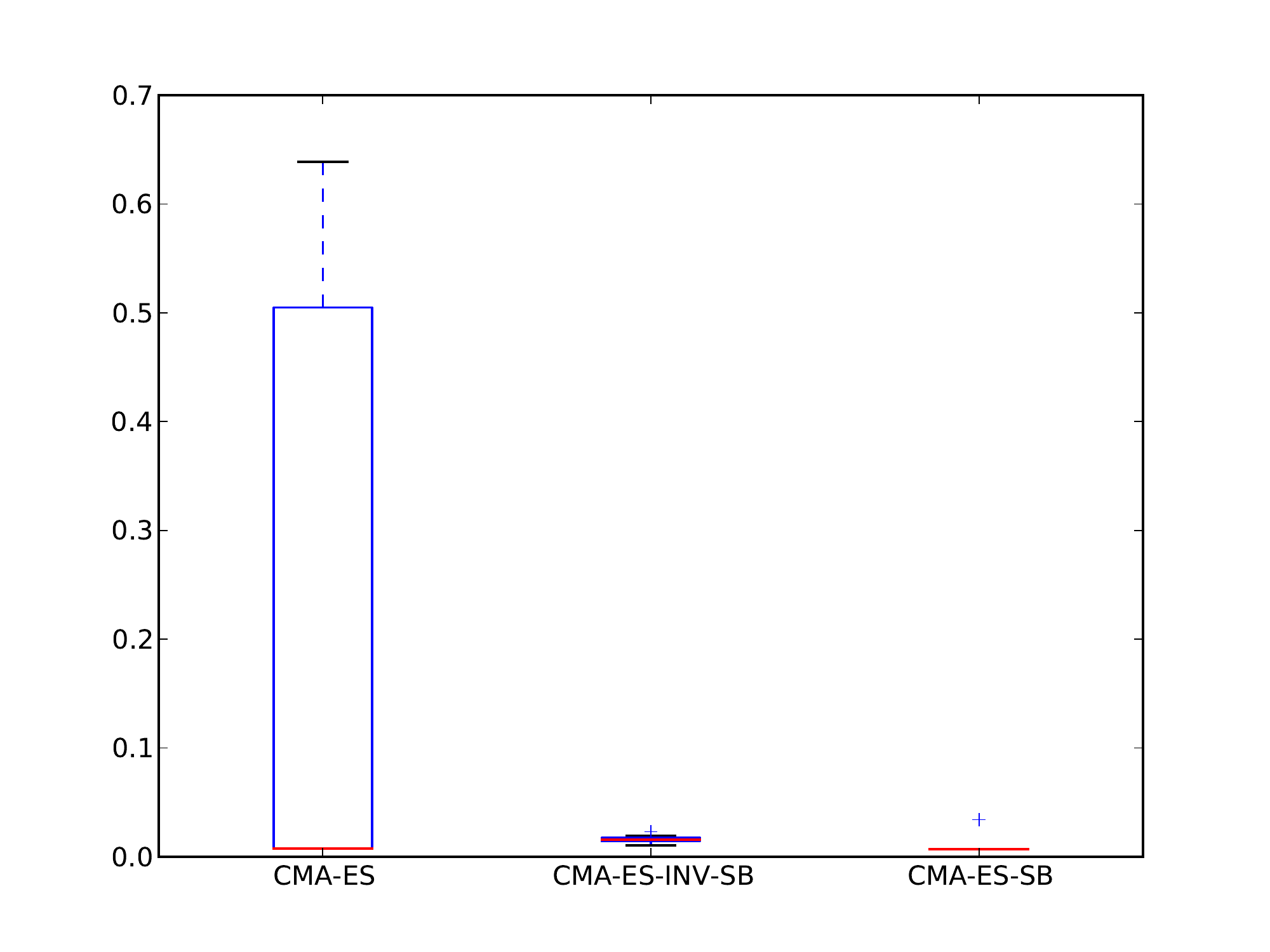}}
      \caption{\label{fig:sinc3d-res} \it Convergence curves for regression by DE (left) and CMA-ES (right) using the {\bf sinc3d} dataset.}
   \end{center}
\end{figure*}
According to the Wilcoxon tests, all pairwise differences are significant. Again, DE-SB and CMA-ES-SB are the fastest methods. This time, in contrast to previous experiments, CMA-ES-INV-SB clearly outperforms CMA-ES.

\subsection{Autoencoding problems}\label{sec:autoencproblems}
In this section, all $d$-dimensional data samples lie on a $s$-dimensional set, where $s<d$. As a result, the data can be described, or 'encoded' by an $s$-dimensional subset. On the other hand, there is also a $s$-D to $d$-D mapping to 'decode' the data. The task is to approximate both the encoding and decoding mapping by an ANN. As in the case of the regression problems, the performance is compared only on the training using a training set.
\clearpage
\subsubsection{Dataset {\bf autoenc-circle}}
In this problem, the data samples lie on a 2-D circle centered at the origin with radius one. We use a 2-5-3-2-1-2-3-5-2 net and 200 data samples to encode from 2-D to 1-D and decode back to 2-D. The population size for all DE-based methods is $N_p=64$, and $N_p=4000$ for all CMA-ES-based methods. Fig.~\ref{fig:autoenc-circle-res} shows the resulting convergence curves and box plots for the learning process.
%%%%%%%%%%%%%%%%%%%%%%%%%%%%%%%%%%%%%%%%%%%%%%%%%%%%%%%%%%%%%%%%%%%%%%%%%%%%%%%%%%%%%%%%%%%%%%%%%%%%%%%%%%%%%%%%
% plot
\begin{figure*}[h!]
   \begin{center}
      \scalebox{0.35}{\includegraphics{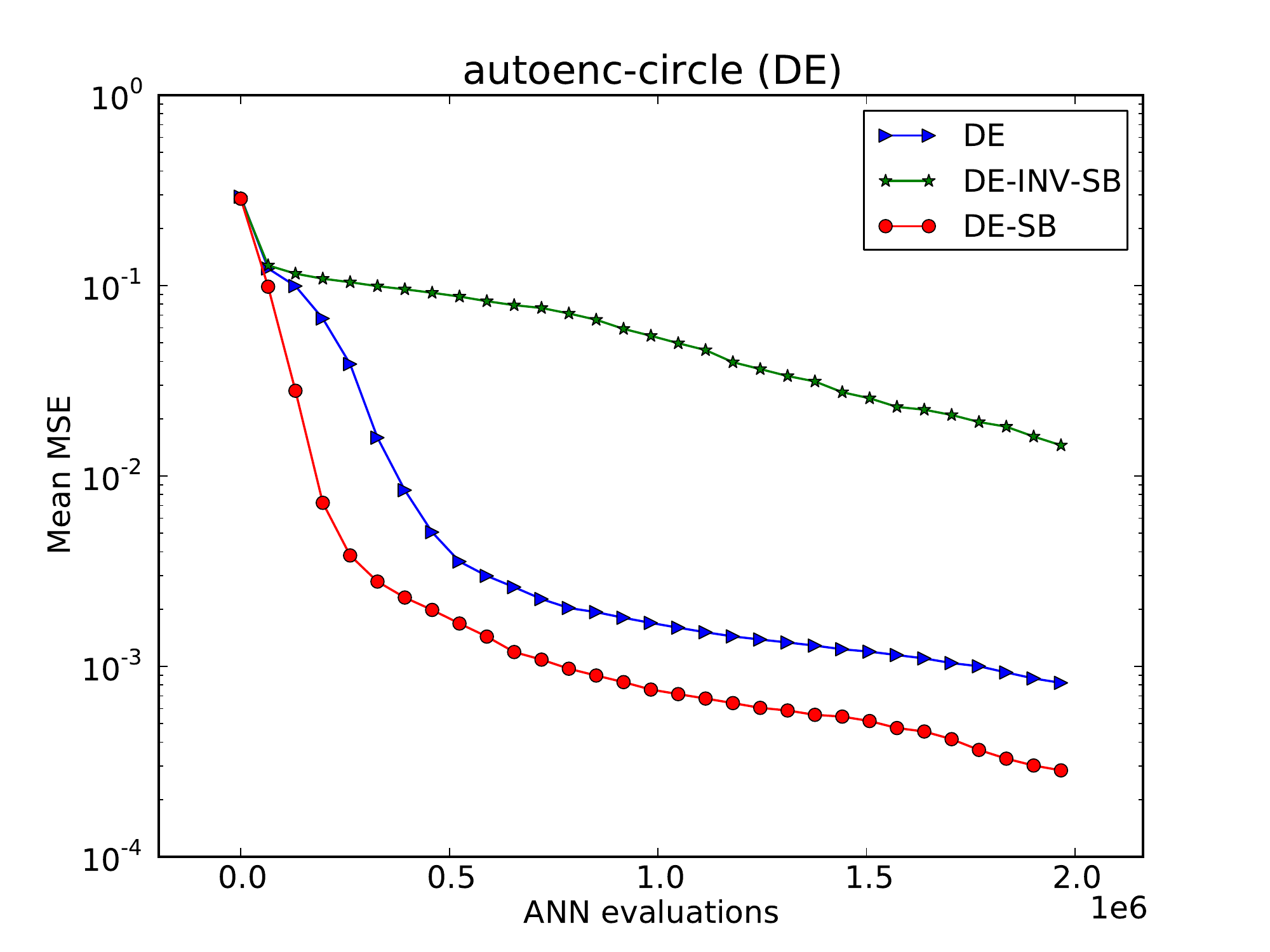}}
      \scalebox{0.35}{\includegraphics{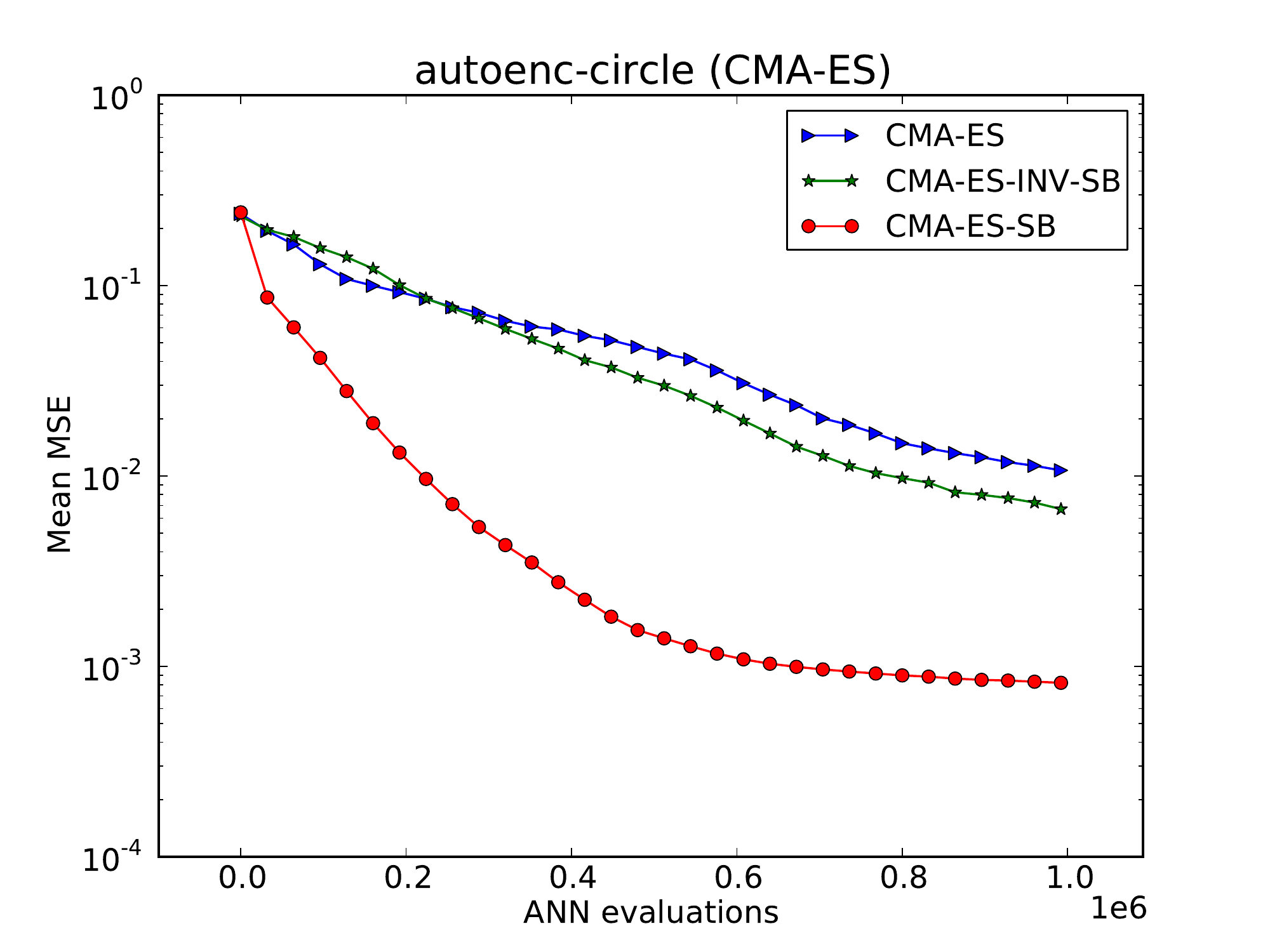}}
      \scalebox{0.35}{\includegraphics{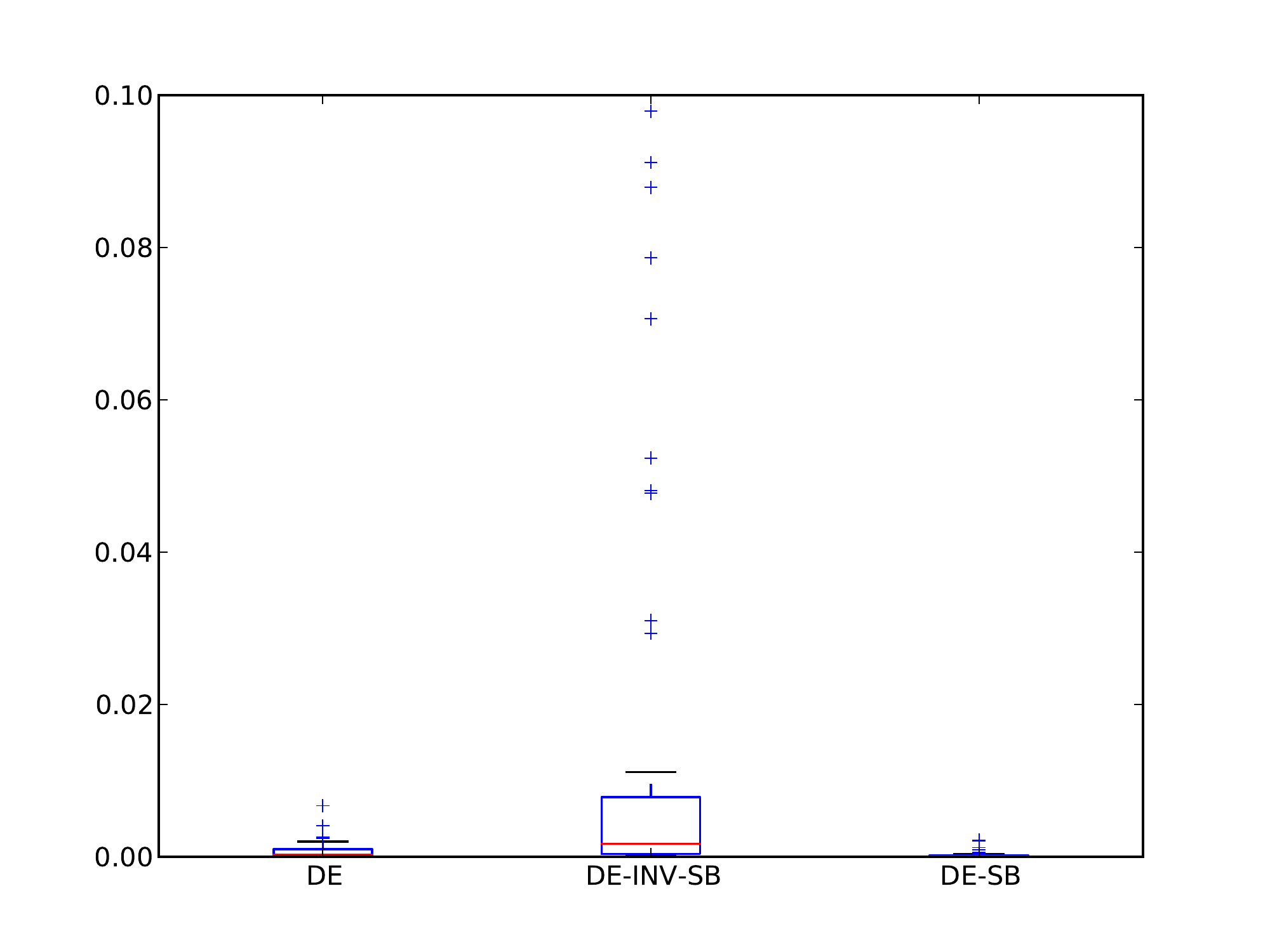}}
      \scalebox{0.35}{\includegraphics{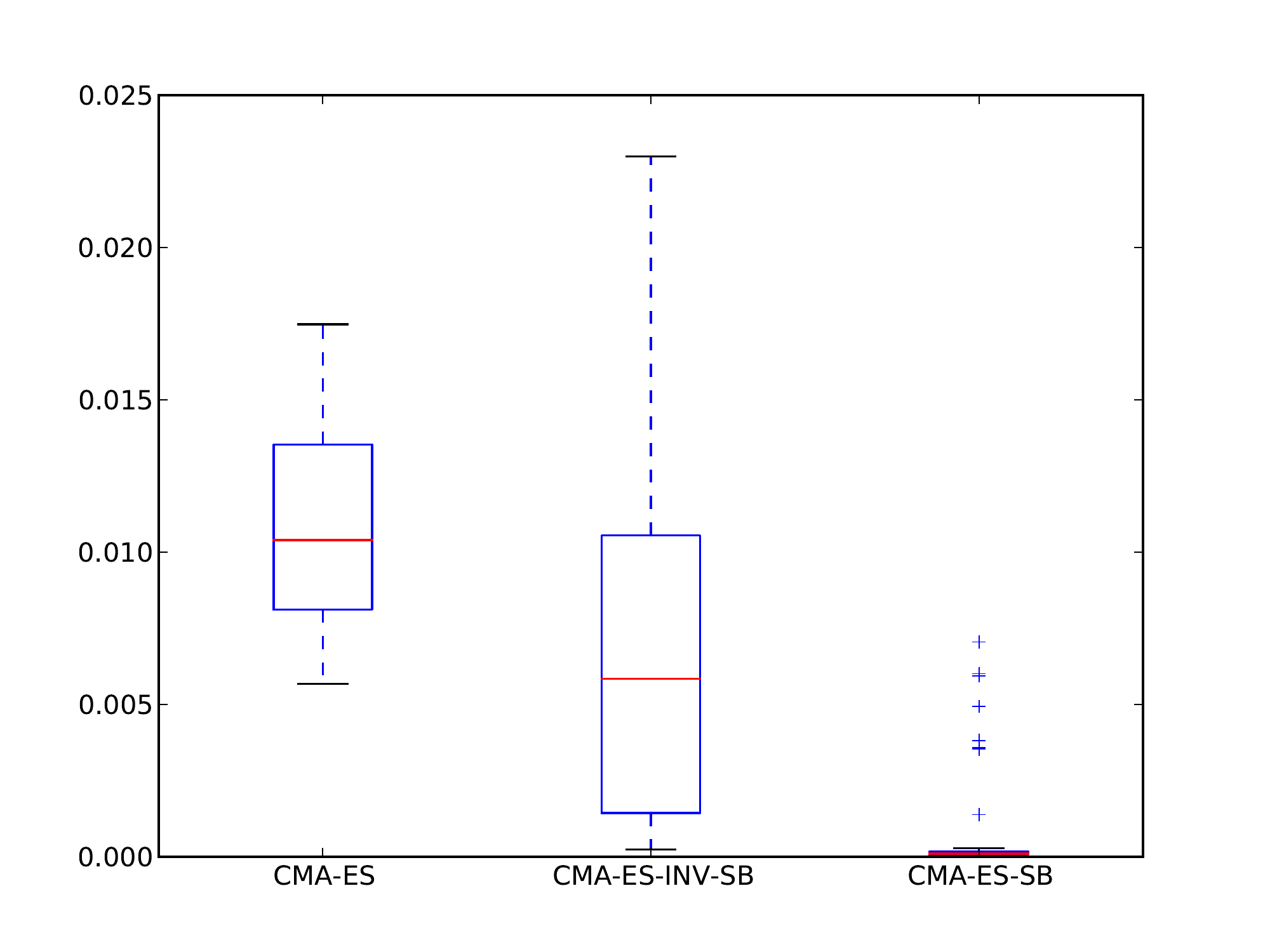}}
      \caption{\label{fig:autoenc-circle-res} \it Convergence curves for regression by DE (left) and CMA-ES (right) using the {\bf sinc3d} dataset.}
   \end{center}
\end{figure*}
All pairwise differences prove to be statistically significant. The proposed symmetry beraking approach improves the training in both methods. On CMA-ES-SB, the difference turns out to be quite significant.
\clearpage
\subsubsection{Dataset {\bf autoenc-spiral}}
In this problem, the data samples lie on a 3-D spiral with radius one, defined by $$(\cos(\phi), \sin(\phi), \phi), \phi\in[0,6\pi].$$ We use a 3-1-3-4-7-3 net and 1000 data samples to encode from 3-D to 1-D and decode back to 3-D. The population size for all DE-based methods is $N_p=80$, and $N_p=400$ for all CMA-ES-based methods. Fig.~\ref{fig:autoenc-spiral-res} shows the resulting convergence curves and box plots for the learning process.
%%%%%%%%%%%%%%%%%%%%%%%%%%%%%%%%%%%%%%%%%%%%%%%%%%%%%%%%%%%%%%%%%%%%%%%%%%%%%%%%%%%%%%%%%%%%%%%%%%%%%%%%%%%%%%%%
% plot
\begin{figure*}[h!]
   \begin{center}
      \scalebox{0.35}{\includegraphics{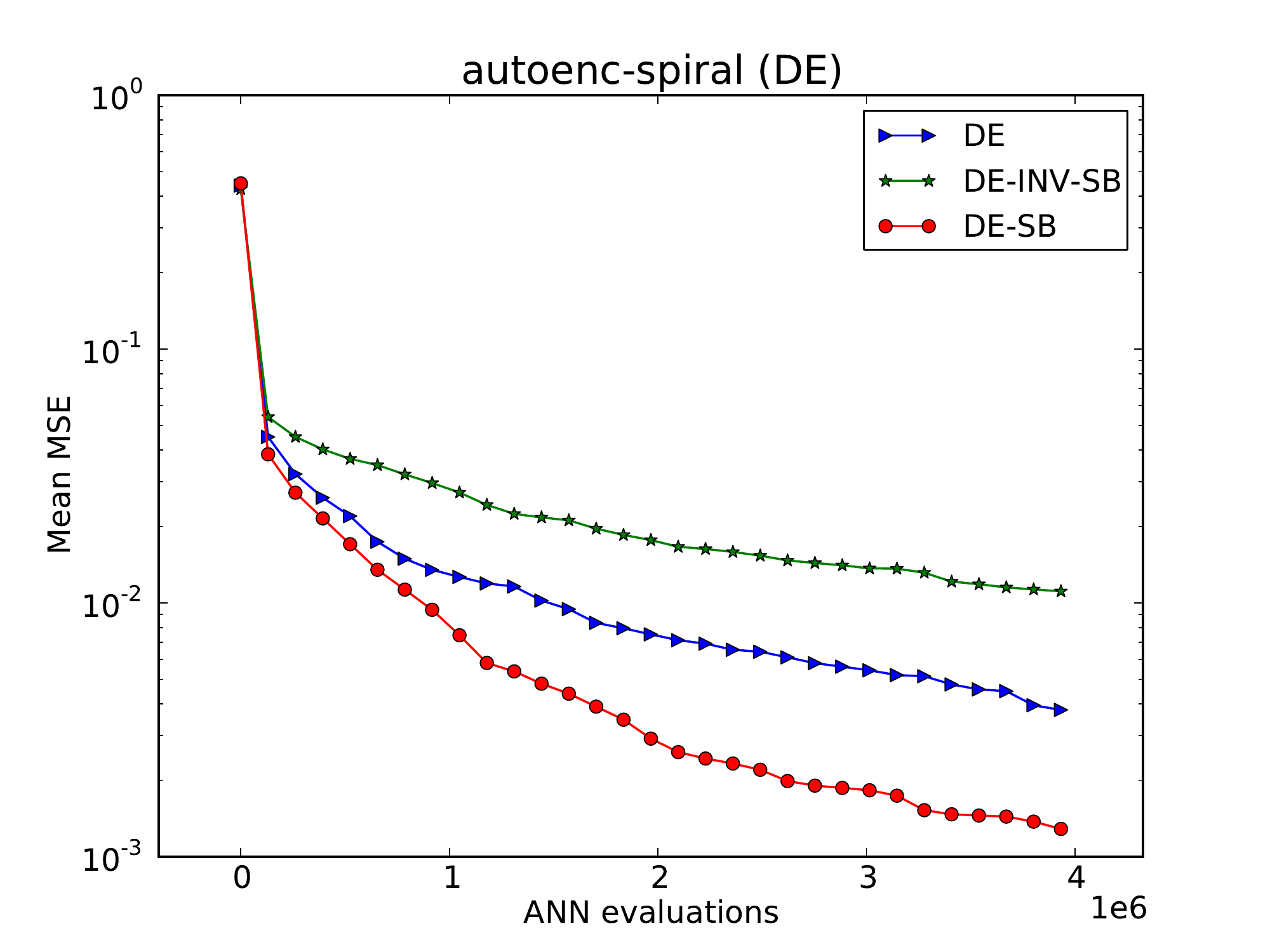}}
      \scalebox{0.35}{\includegraphics{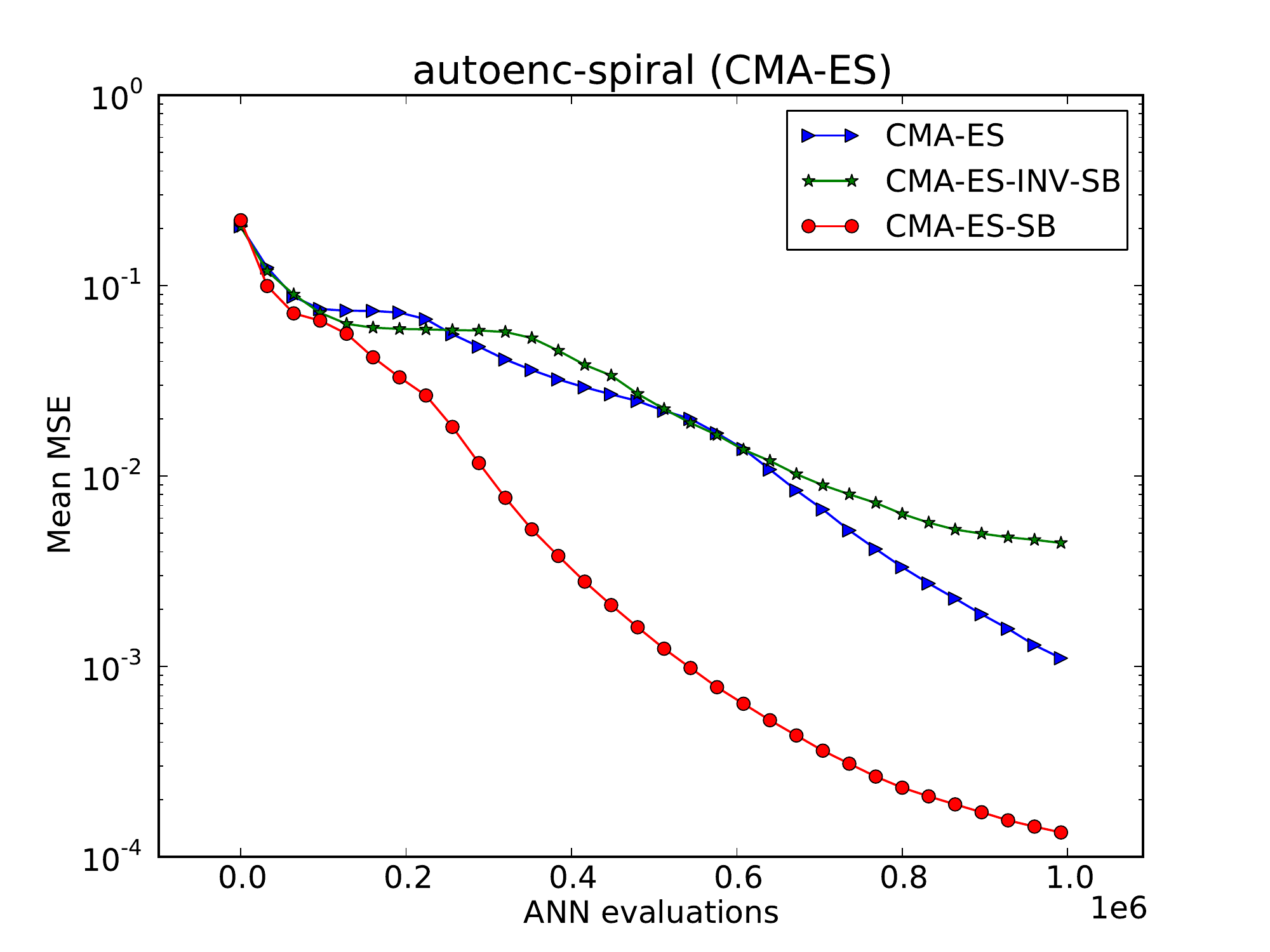}}
      \scalebox{0.35}{\includegraphics{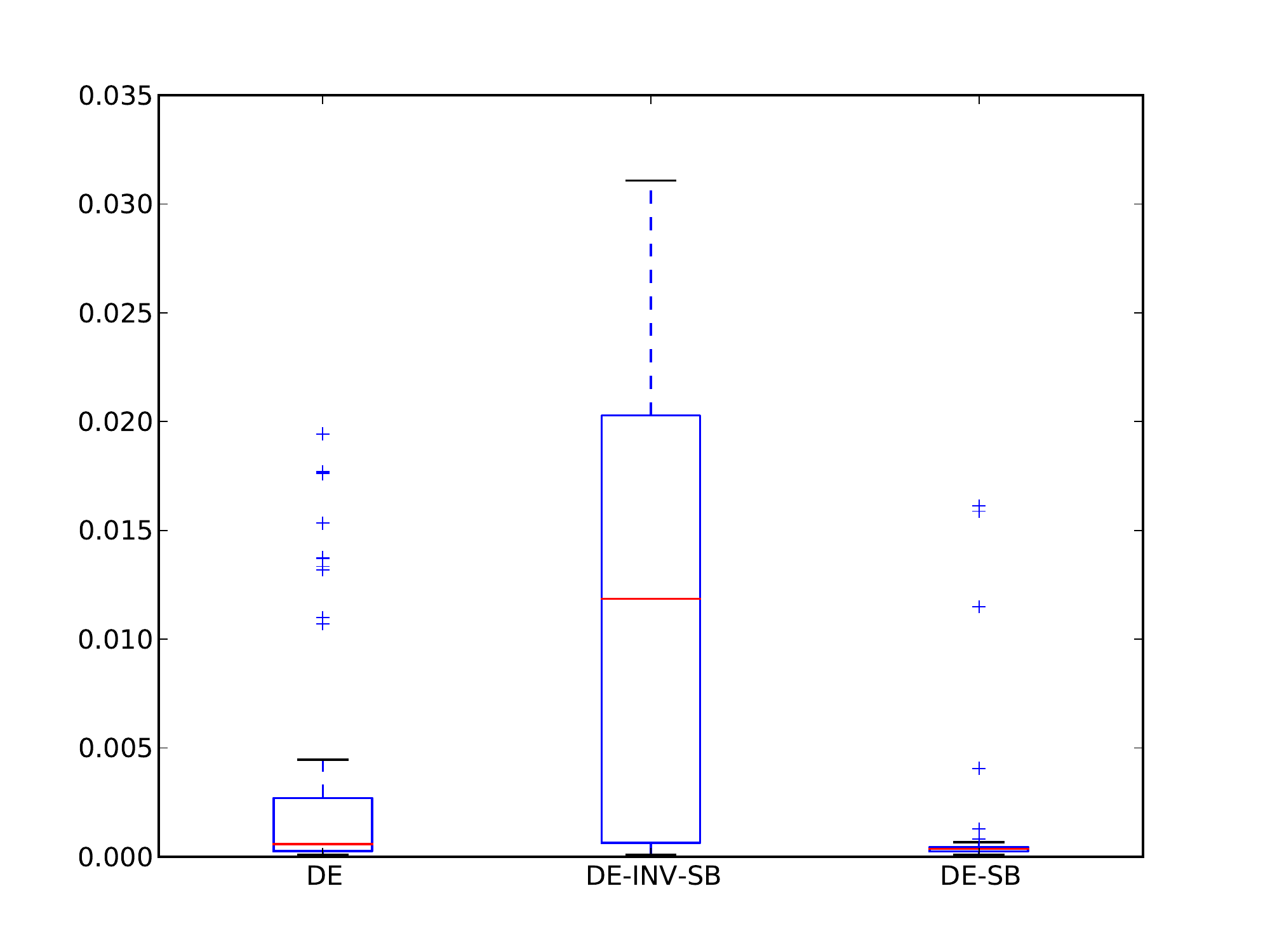}}
      \scalebox{0.35}{\includegraphics{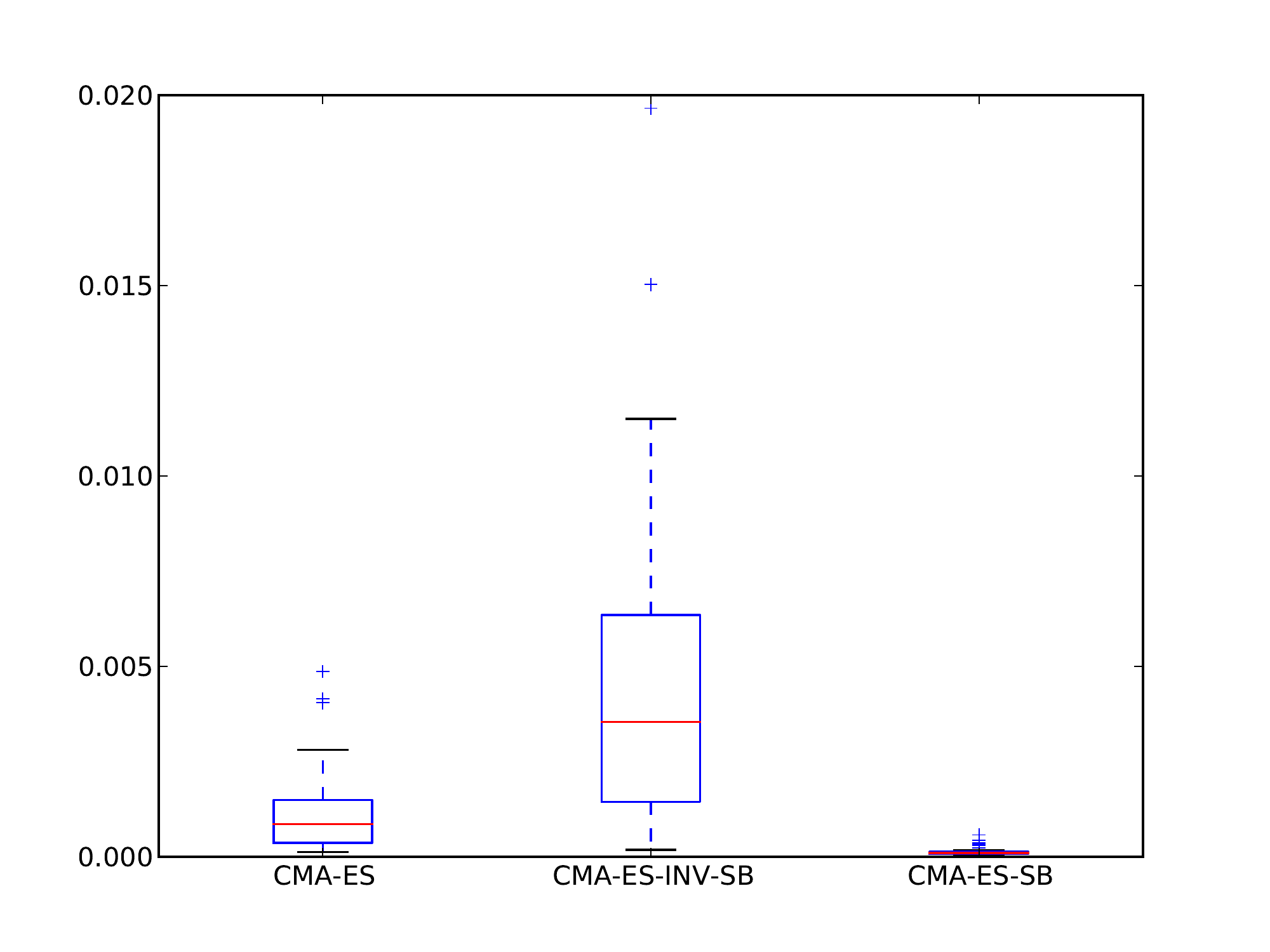}}
      \caption{\label{fig:autoenc-spiral-res} \it Convergence curves for regression by DE (left) and CMA-ES (right) using the {\bf sinc3d} dataset.}
   \end{center}
\end{figure*}
All pairwise differences prove to be statistically significant.
\clearpage
\subsubsection{Dataset {\bf autoenc-sphere}}
In this problem, the data samples lie on a 3-D sphere centered at the origin with radius one. We use a 3-8-5-2-5-8-3 net and 1000 data samples to encode from 3-D to 2-D and decode back to 3-D. The population size for all DE-based methods is $N_p=96$, and $N_p=1000$ for all CMA-ES-based methods. Fig.~\ref{fig:autoenc-sphere-res} shows the resulting convergence curves and box plots for the learning process.
%%%%%%%%%%%%%%%%%%%%%%%%%%%%%%%%%%%%%%%%%%%%%%%%%%%%%%%%%%%%%%%%%%%%%%%%%%%%%%%%%%%%%%%%%%%%%%%%%%%%%%%%%%%%%%%%
% plot
\begin{figure*}[h!]
   \begin{center}
      \scalebox{0.35}{\includegraphics{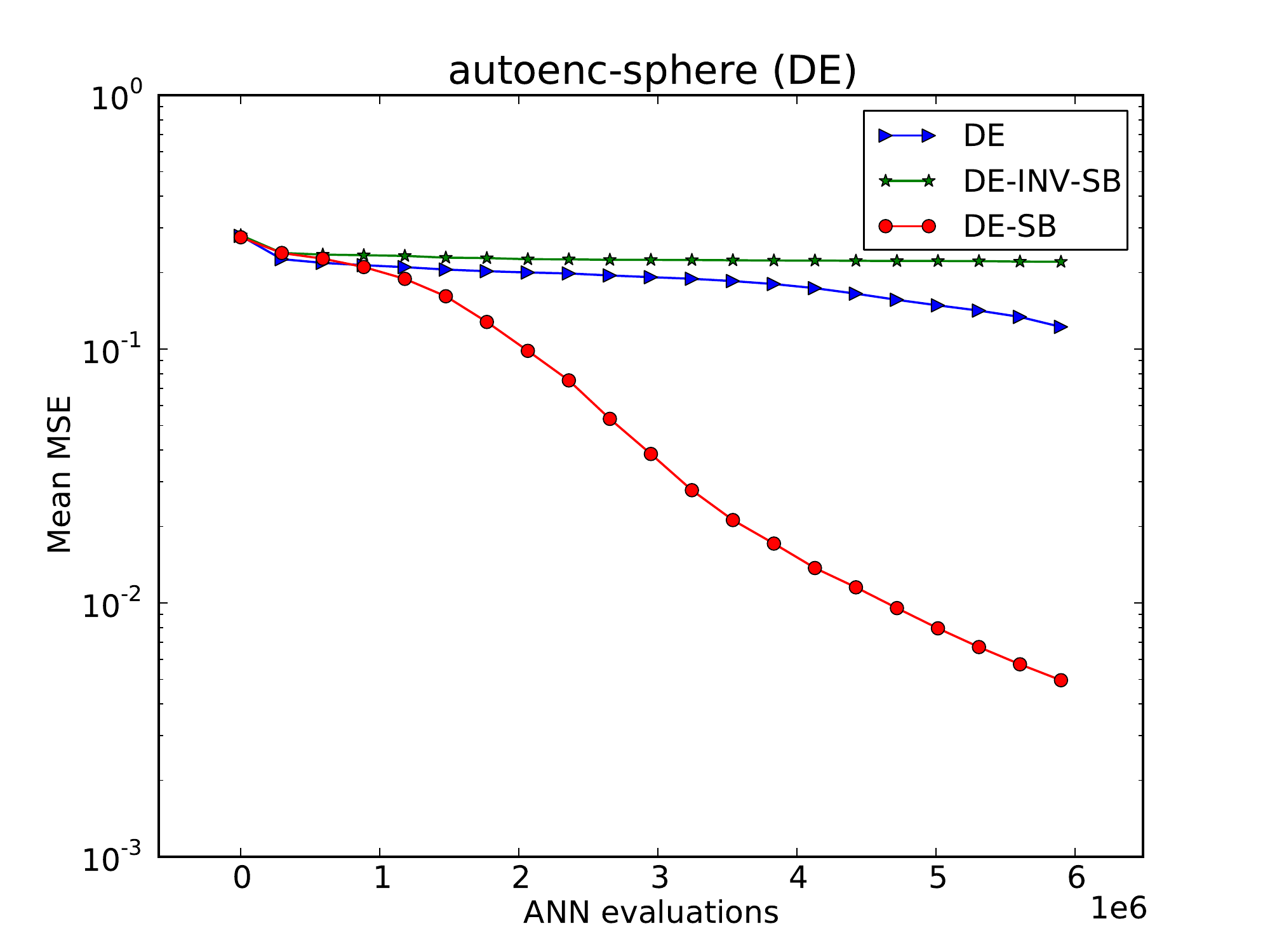}}
      \scalebox{0.35}{\includegraphics{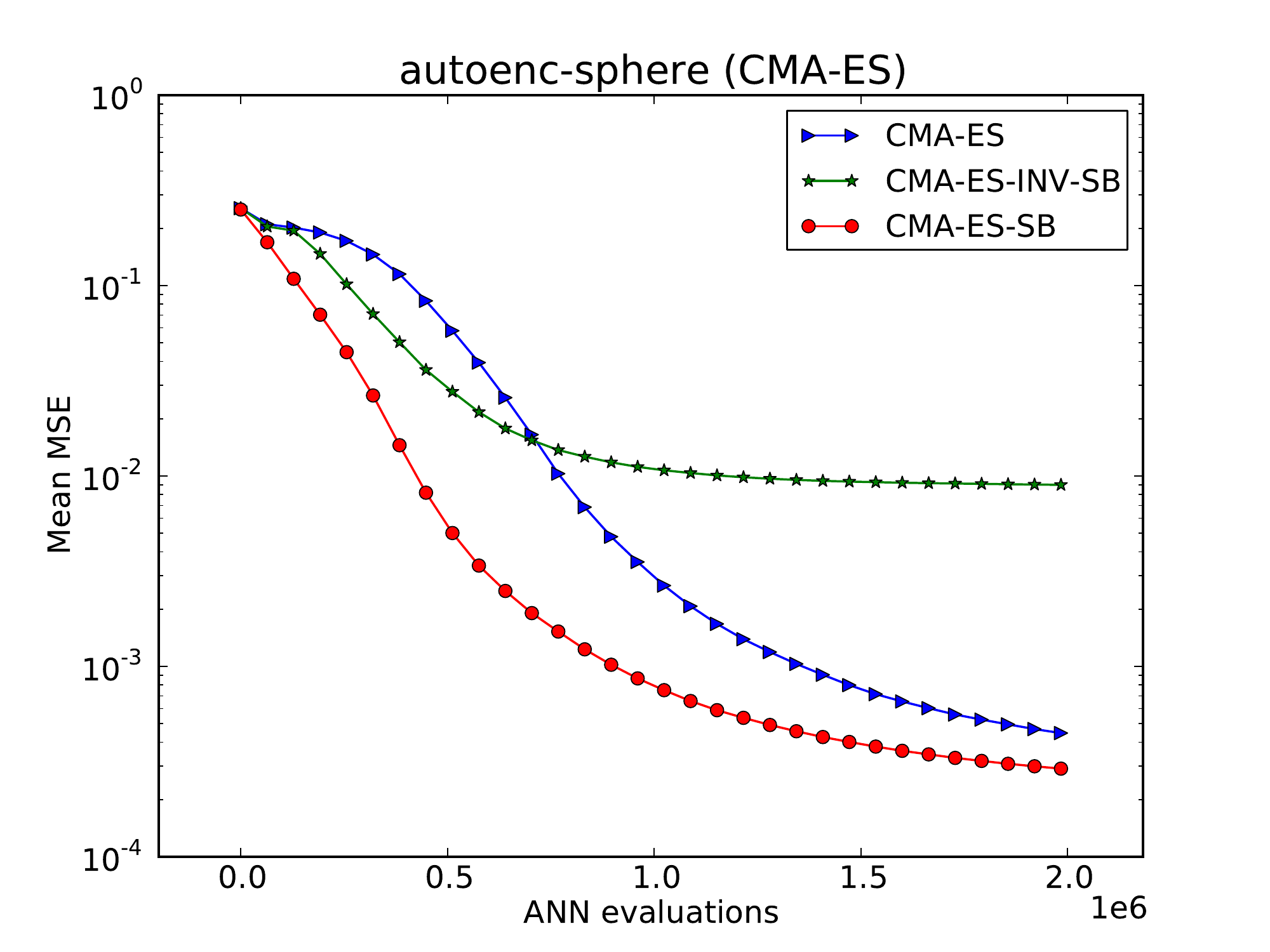}}
      \scalebox{0.35}{\includegraphics{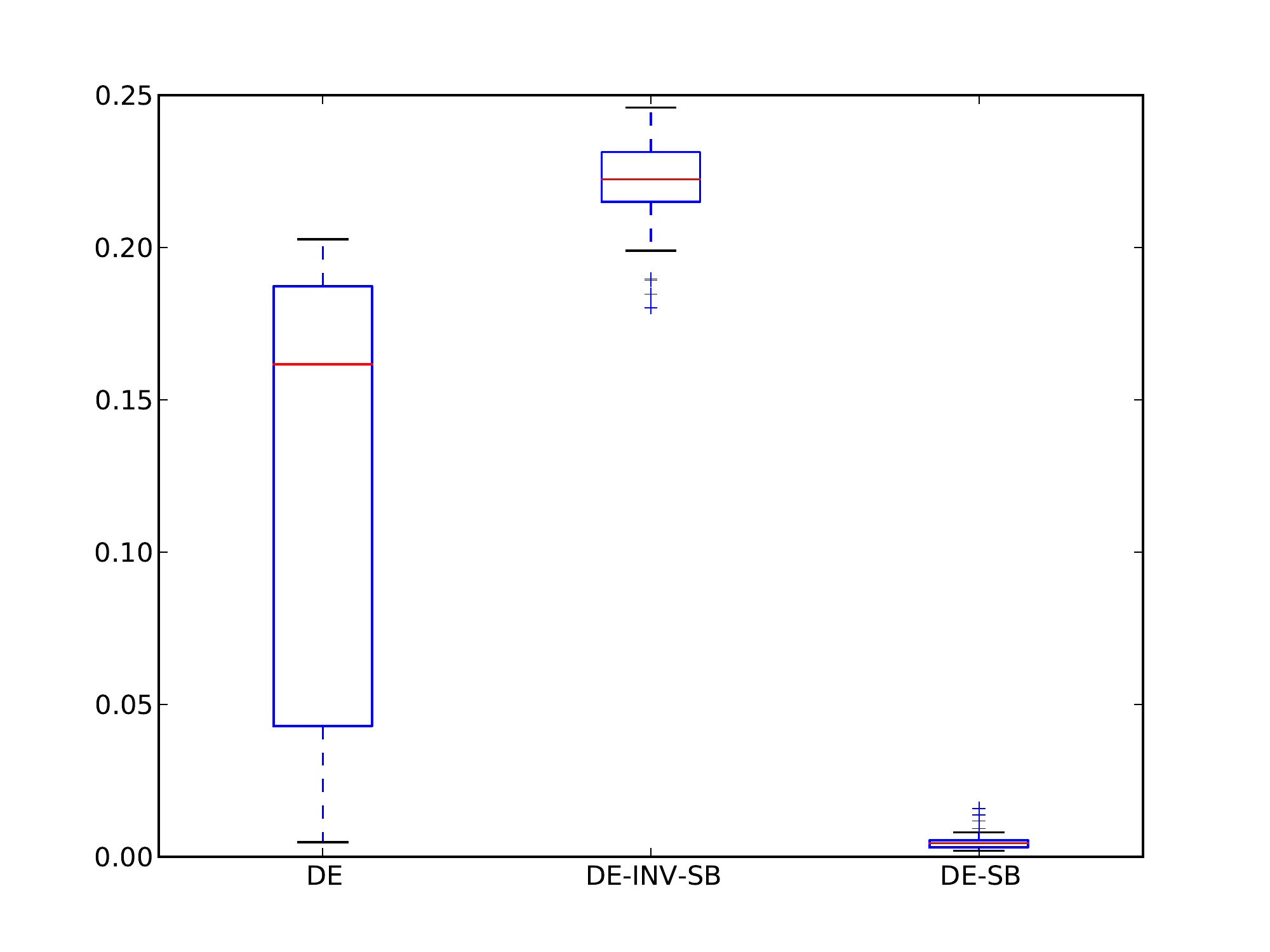}}
      \scalebox{0.35}{\includegraphics{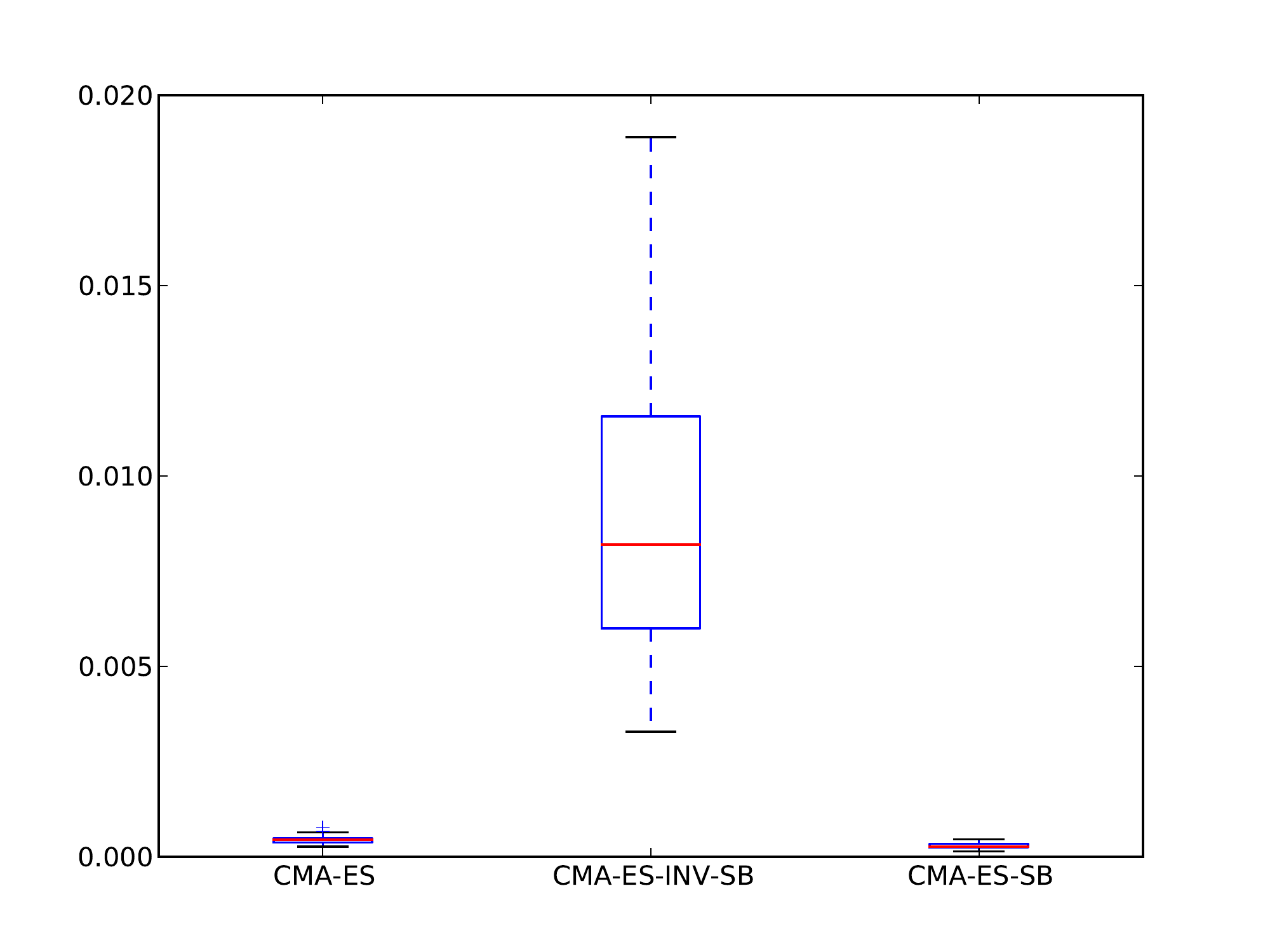}}
      \caption{\label{fig:autoenc-sphere-res} \it Convergence curves for regression by DE (left) and CMA-ES (right) using the {\bf autoenc-sphere} dataset.}
   \end{center}
\end{figure*}
All pairwise differences prove to be statistically significant. Clearly, DE-SB and CMA-ES-SB are significantly faster then the other methods. 

\subsection{Classification problems}\label{sec:classproblems}
In classification problems, data samples are divided into a training set, a validation set and a test set. All three sets are generated by random selection of samples. A winner-takes-all scheme is applied to distinguish different classes, i.e., given an input, the ANN-output component with the greatest value determines the class. In order to improve generalization, classification performance measures on the training and test set are updated only on each improvement of the validation set classification performance.

\clearpage
\subsubsection{Dataset: {\bf Two-Circles}}
In this problem, the 2-D data domain $[-1,1]^2$ is divided into two parts, where one part is given by the union area of two circles and the remaining part is the disjunct space. Hence, there are two classes: samples which lie inside any circle and samples which lie outside of both circles. One circle is specified by center $(0.5,0.5)$ and radius $r_1=0.39894$, and the other circle by center $(-0.5,-0.5)$ and same radius $r_2=0.39894$. We use a 2-4-2-4-2 net with 400 samples for each training, validation and test set, having a total of 1200 samples. The population size for all DE-based methods is $N_p=80$, and $N_p=400$ for all CMA-ES-based methods. Fig.~\ref{fig:autoenc-spiral-res} shows the resulting convergence curves and box plots for the learning process.
%%%%%%%%%%%%%%%%%%%%%%%%%%%%%%%%%%%%%%%%%%%%%%%%%%%%%%%%%%%%%%%%%%%%%%%%%%%%%%%%%%%%%%%%%%%%%%%%%%%%%%%%%%%%%%%%
% plot
\begin{figure*}[h!]
   \begin{center}
      \scalebox{0.26}{\includegraphics{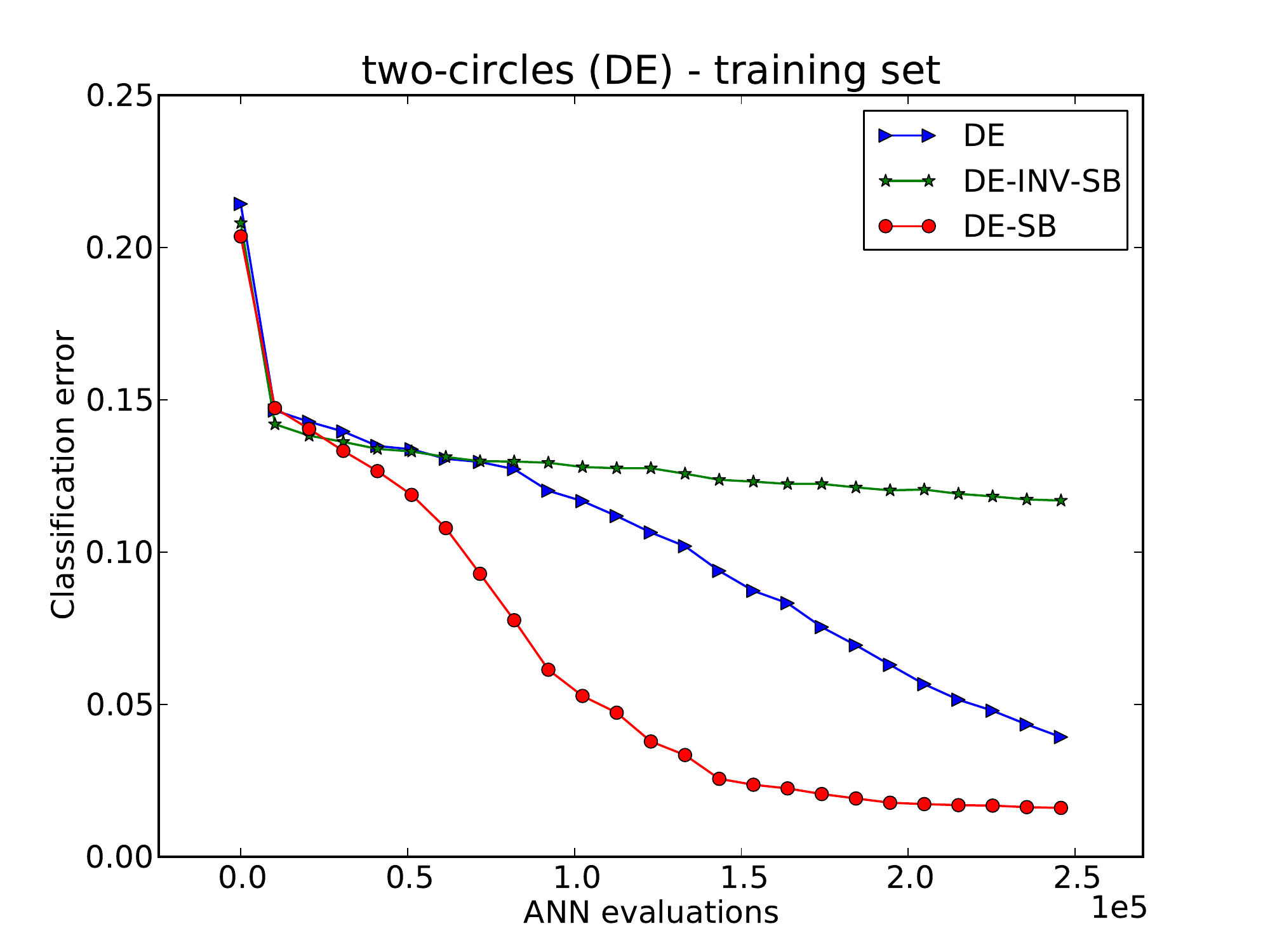}}\hspace{-0.6cm}
      \scalebox{0.26}{\includegraphics{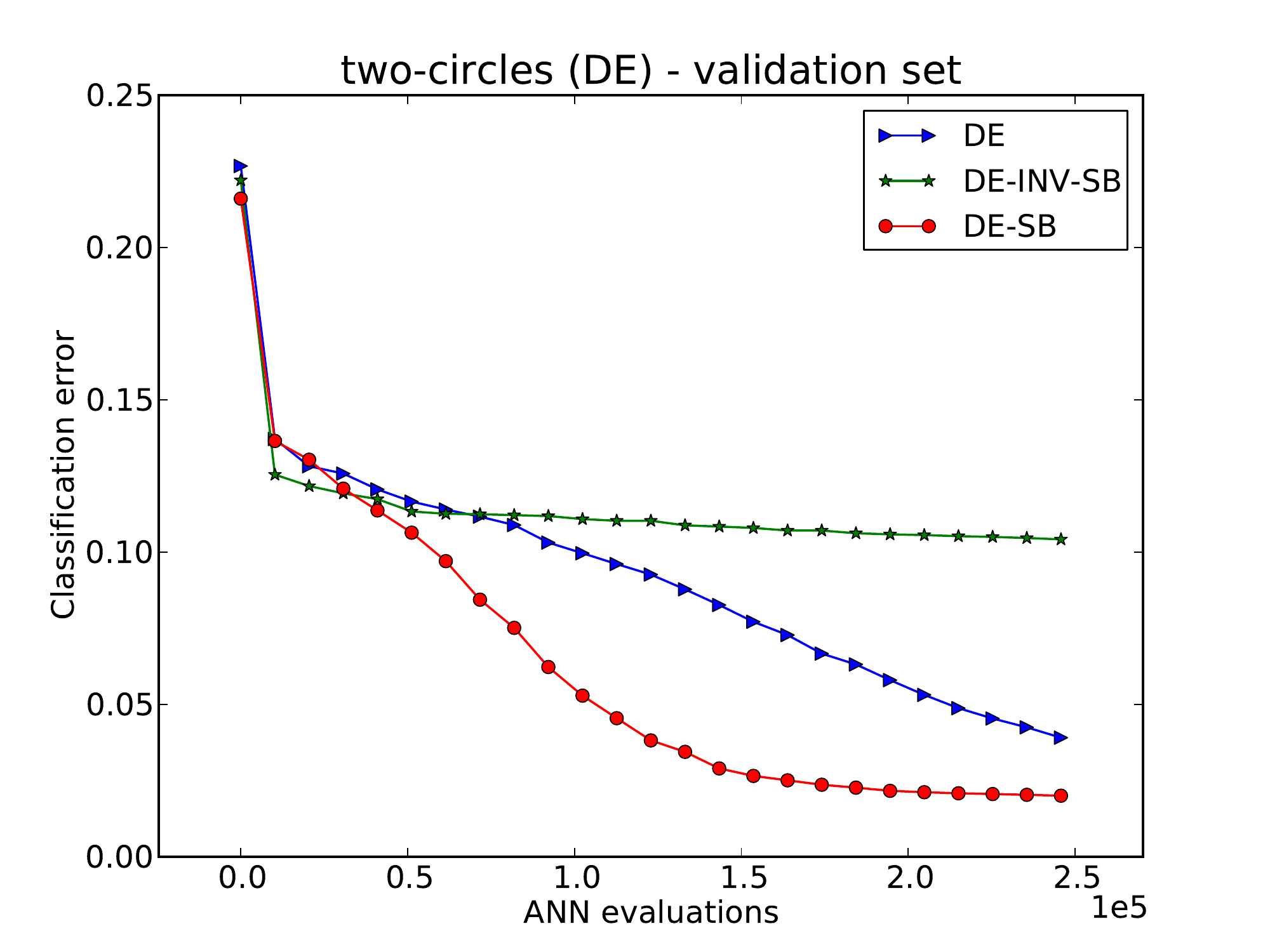}}\hspace{-0.6cm}
      \scalebox{0.26}{\includegraphics{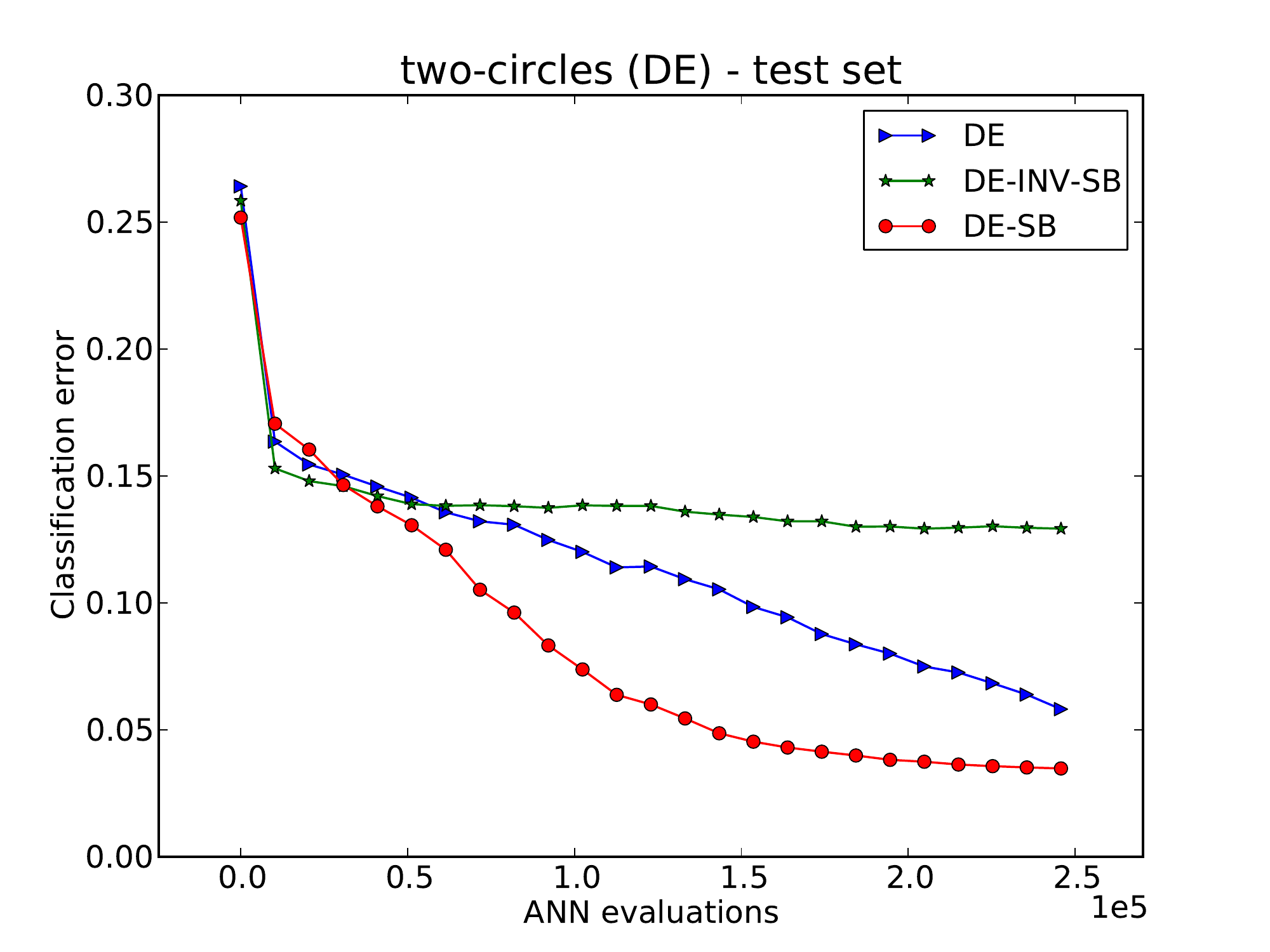}}\\
      \scalebox{0.26}{\includegraphics{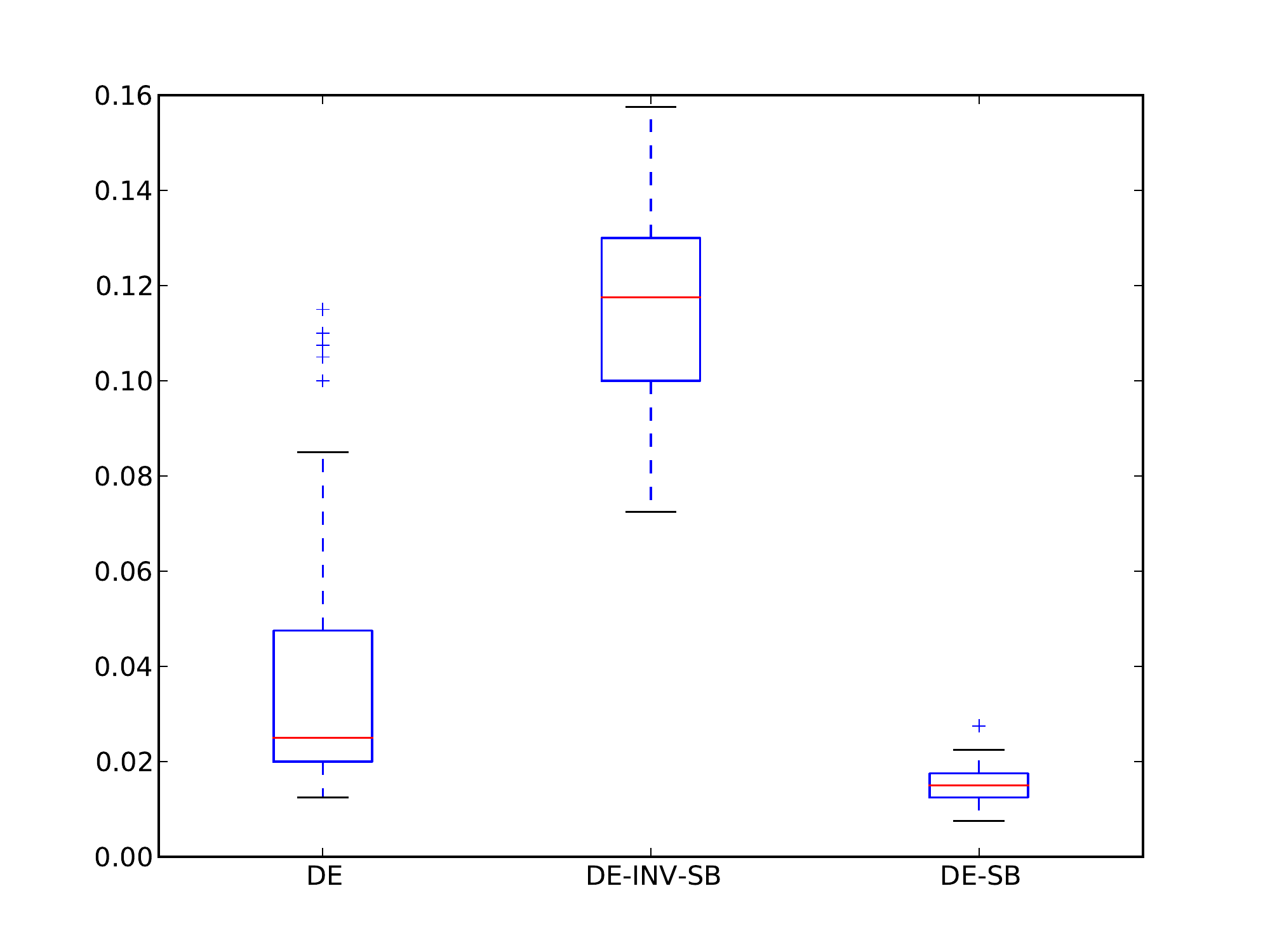}}\hspace{-0.6cm}
      \scalebox{0.26}{\includegraphics{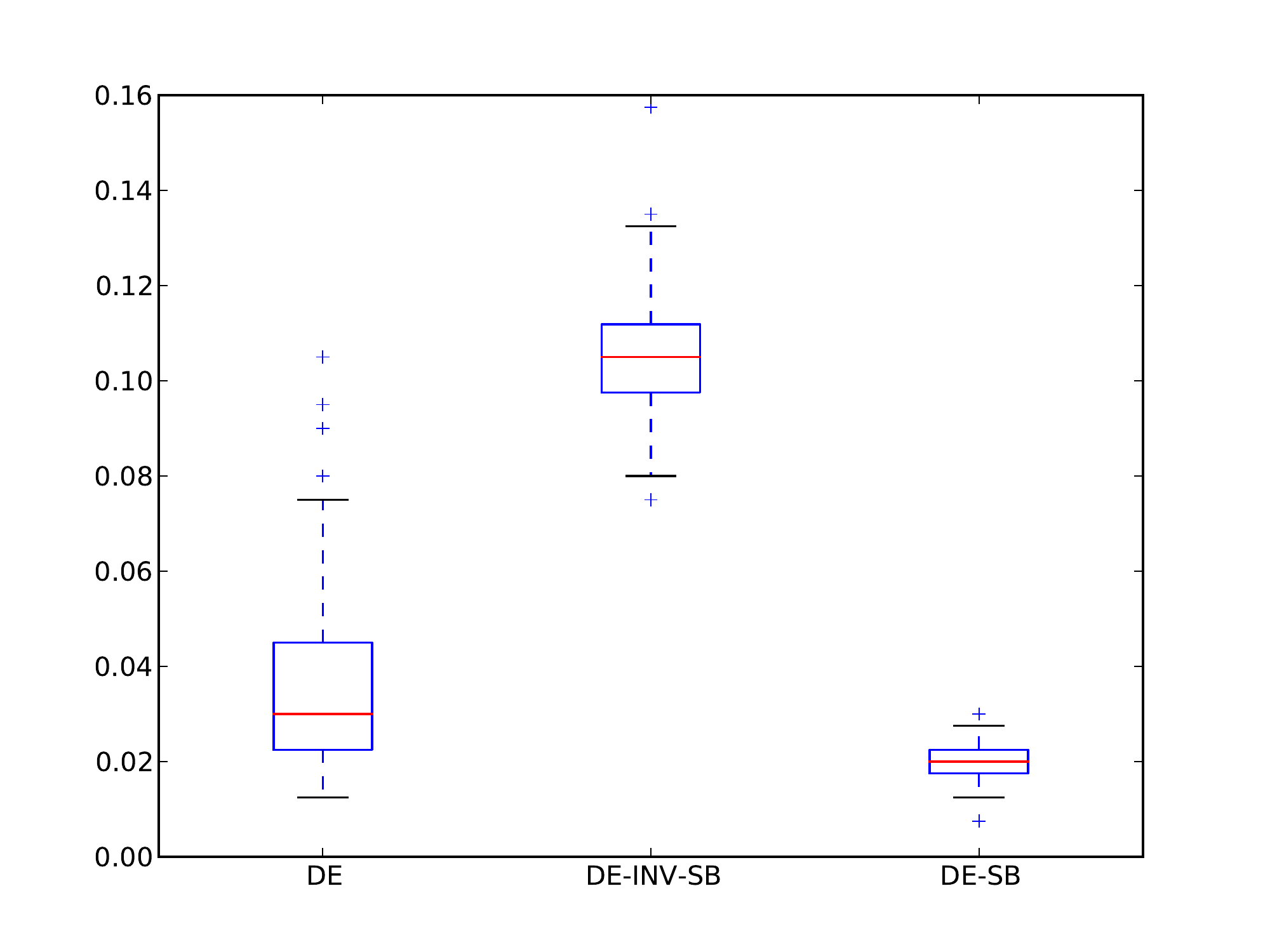}}\hspace{-0.6cm}
      \scalebox{0.26}{\includegraphics{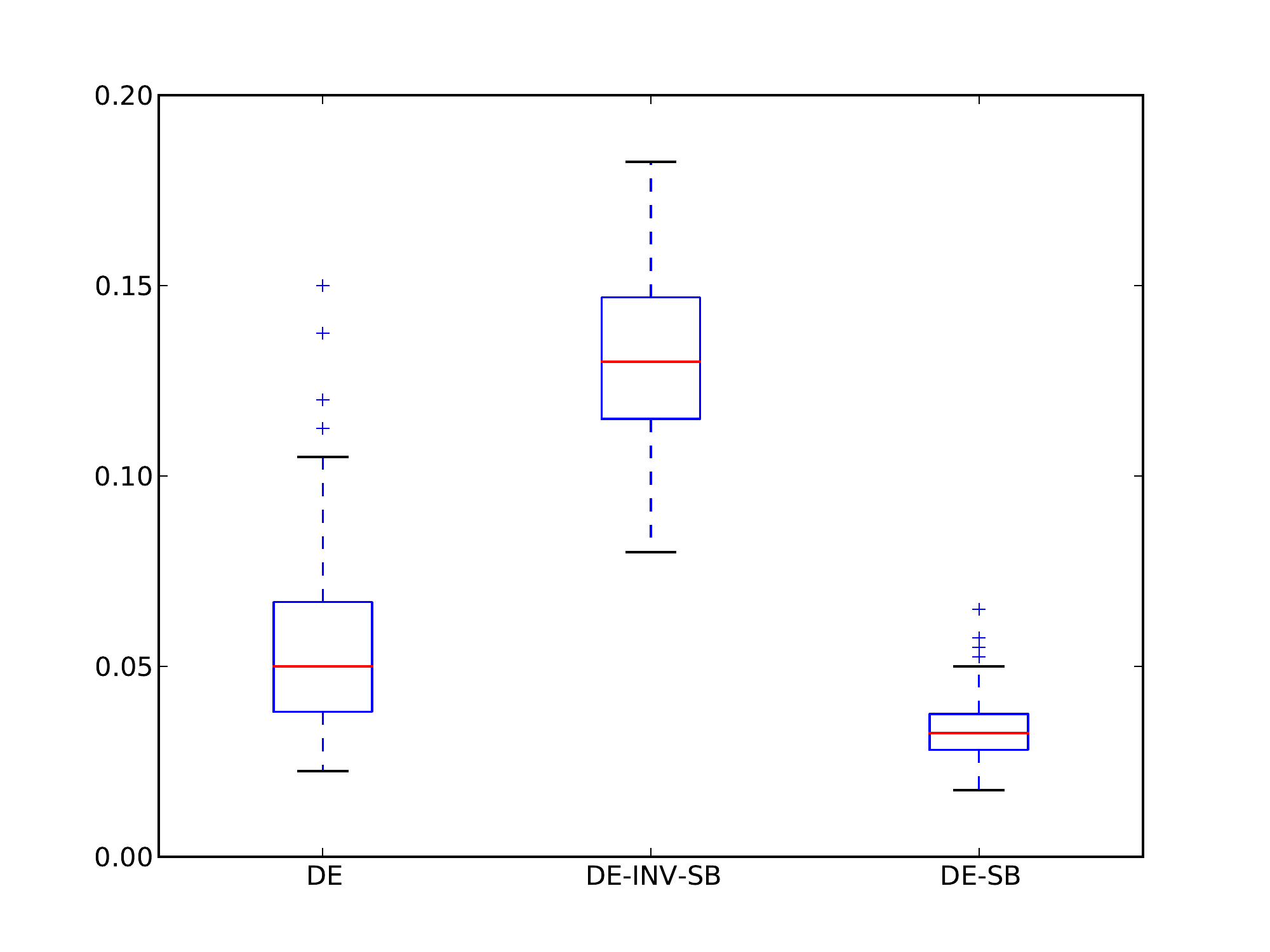}}
      \caption{\label{fig:two-circles-res-DE} \it Classification error rates over ANN-evaluations on the {\bf Two-Circles} dataset using the DE-variants.}
   \end{center}
\end{figure*}
%%%%%%%%%%%%%%%%%%%%%%%%%%%%%%%%%%%%%%%%%%%%%%%%%%%%%%%%%%%%%%%%%%%%%%%%%%%%%%%%%%%%%%%%%%%%%%%%%%%%%%%%%%%%%%%%
% plot
\begin{figure*}[h!]
   \begin{center}
      \scalebox{0.26}{\includegraphics{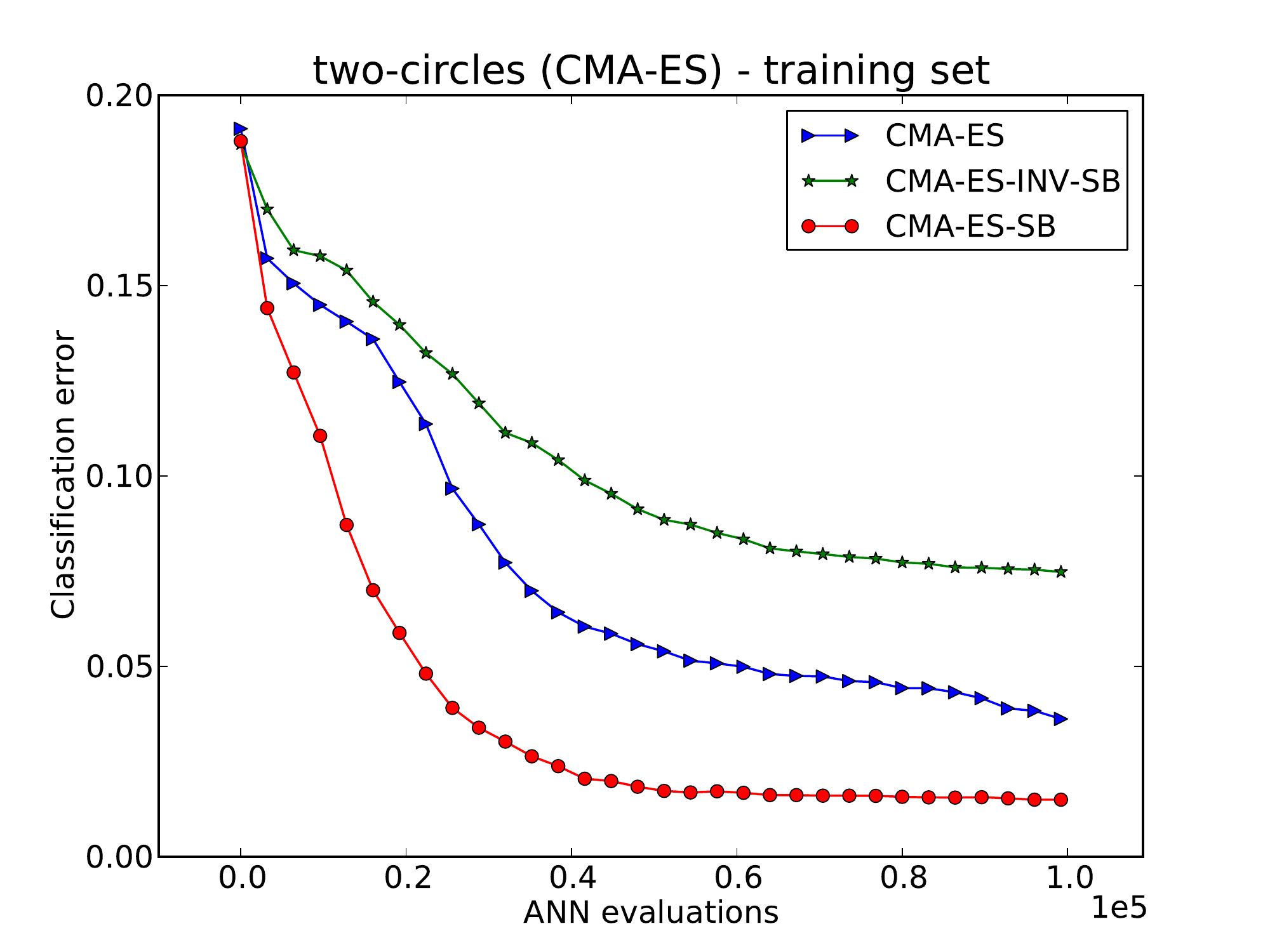}}\hspace{-0.6cm}
      \scalebox{0.26}{\includegraphics{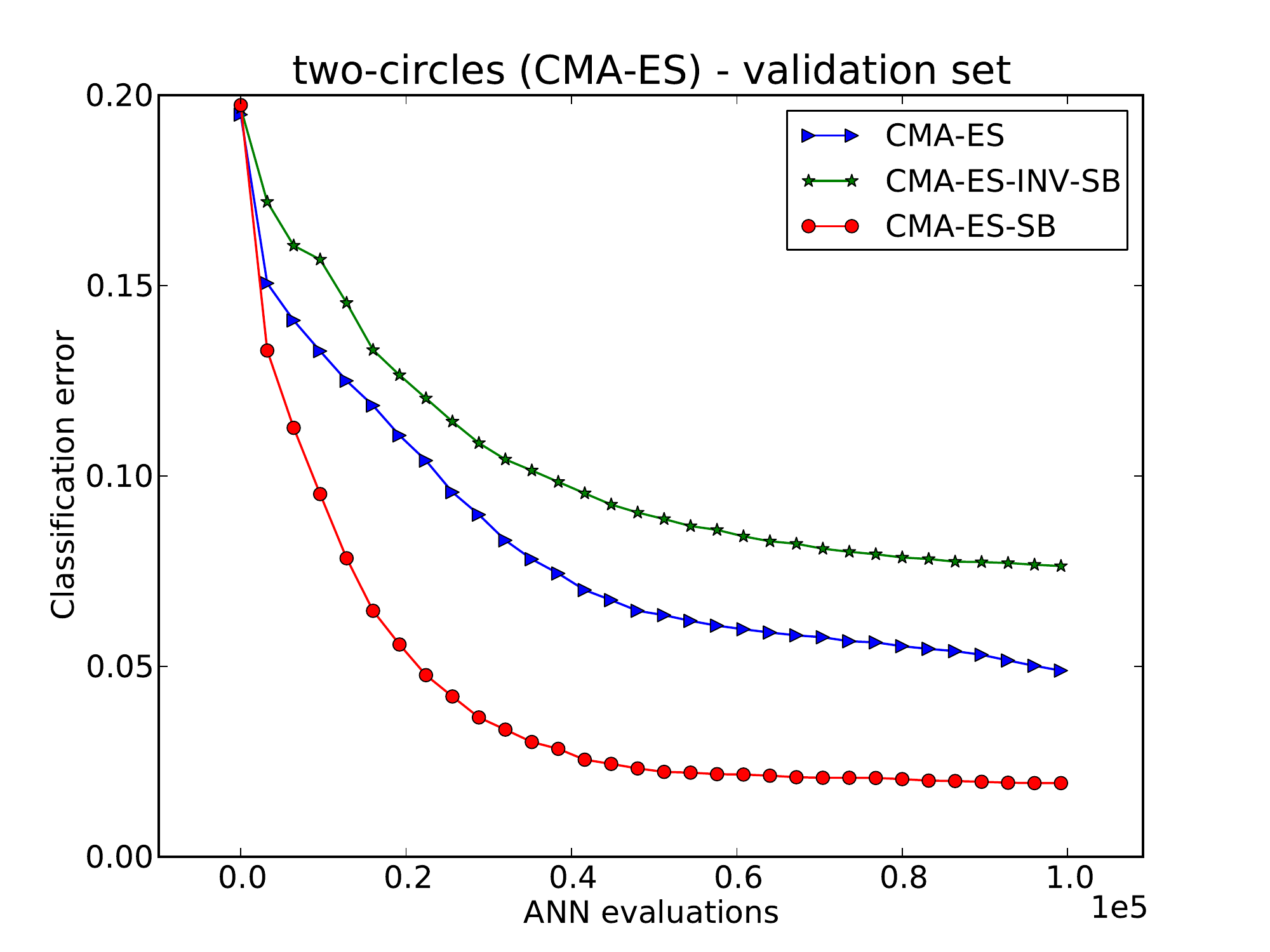}}\hspace{-0.6cm}
      \scalebox{0.26}{\includegraphics{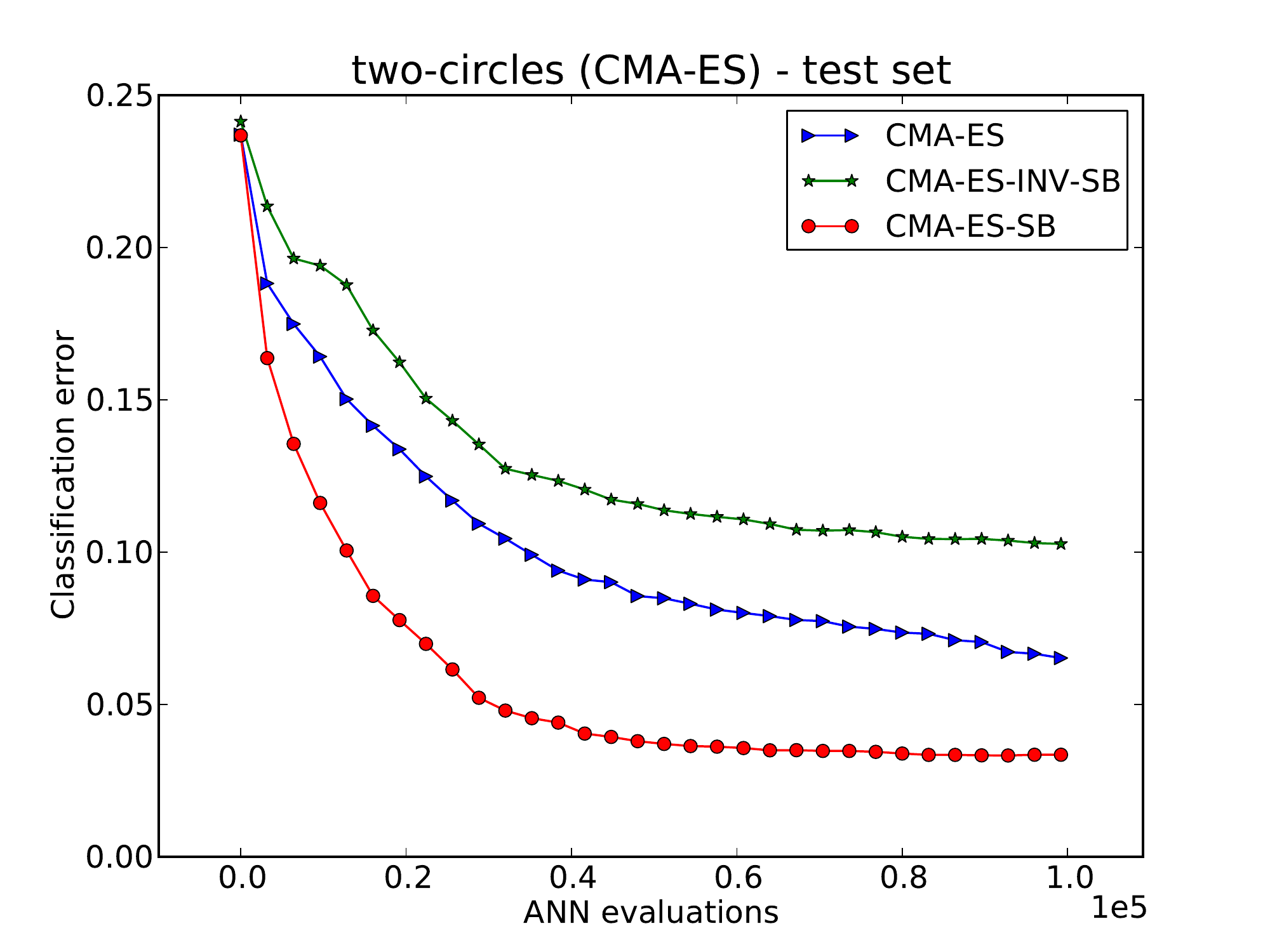}}\\
      \scalebox{0.26}{\includegraphics{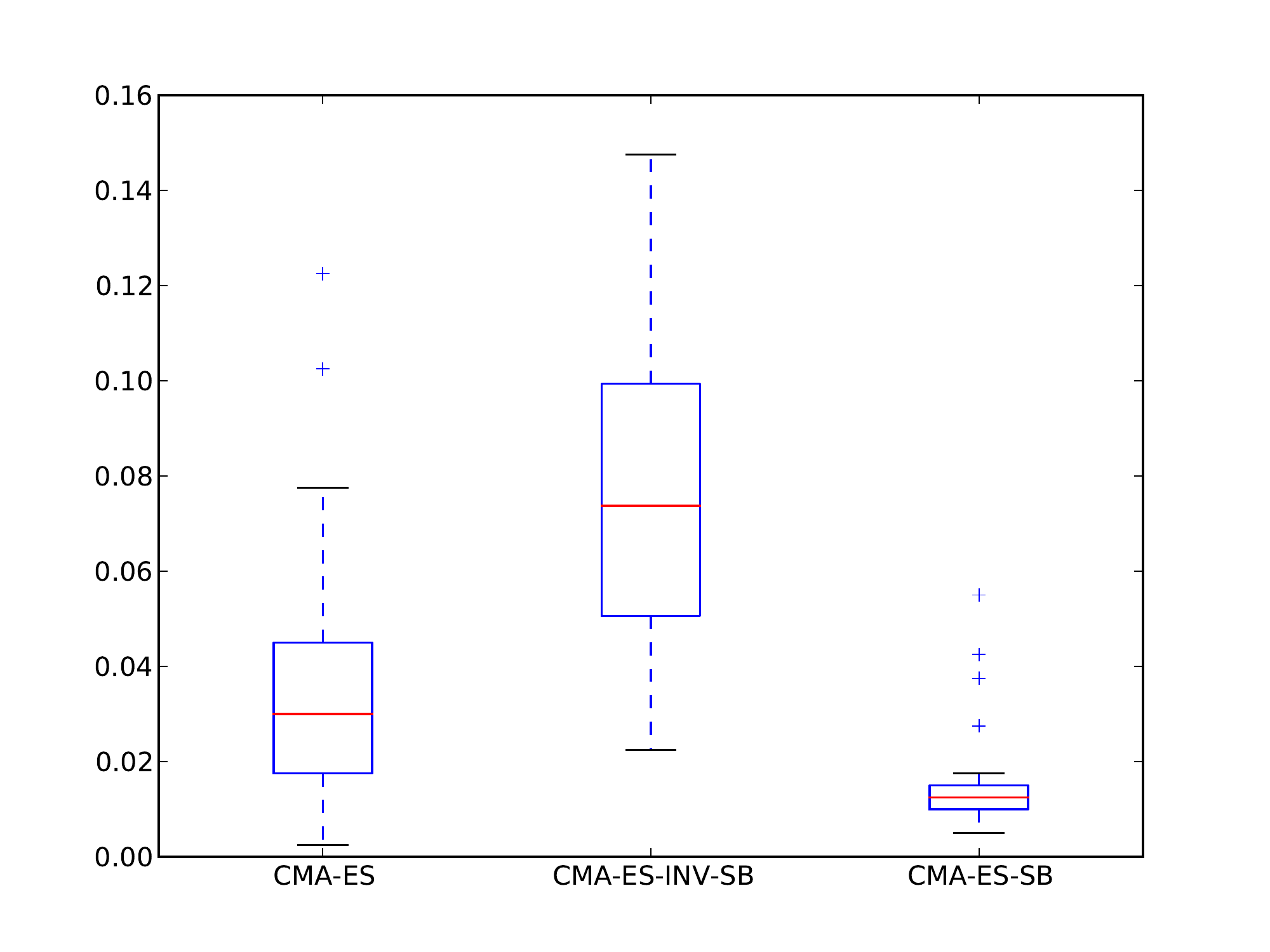}}\hspace{-0.6cm}
      \scalebox{0.26}{\includegraphics{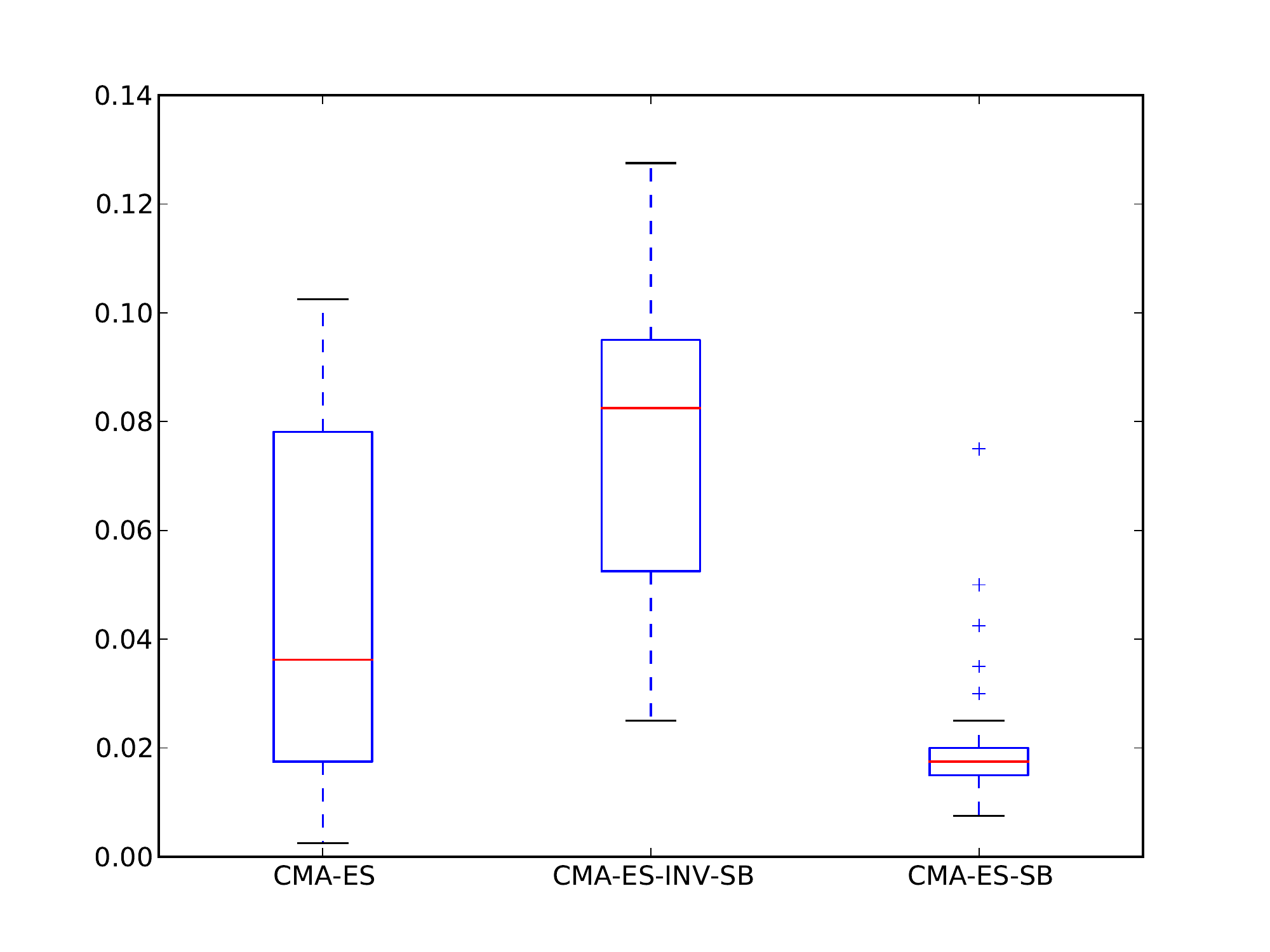}}\hspace{-0.6cm}
      \scalebox{0.26}{\includegraphics{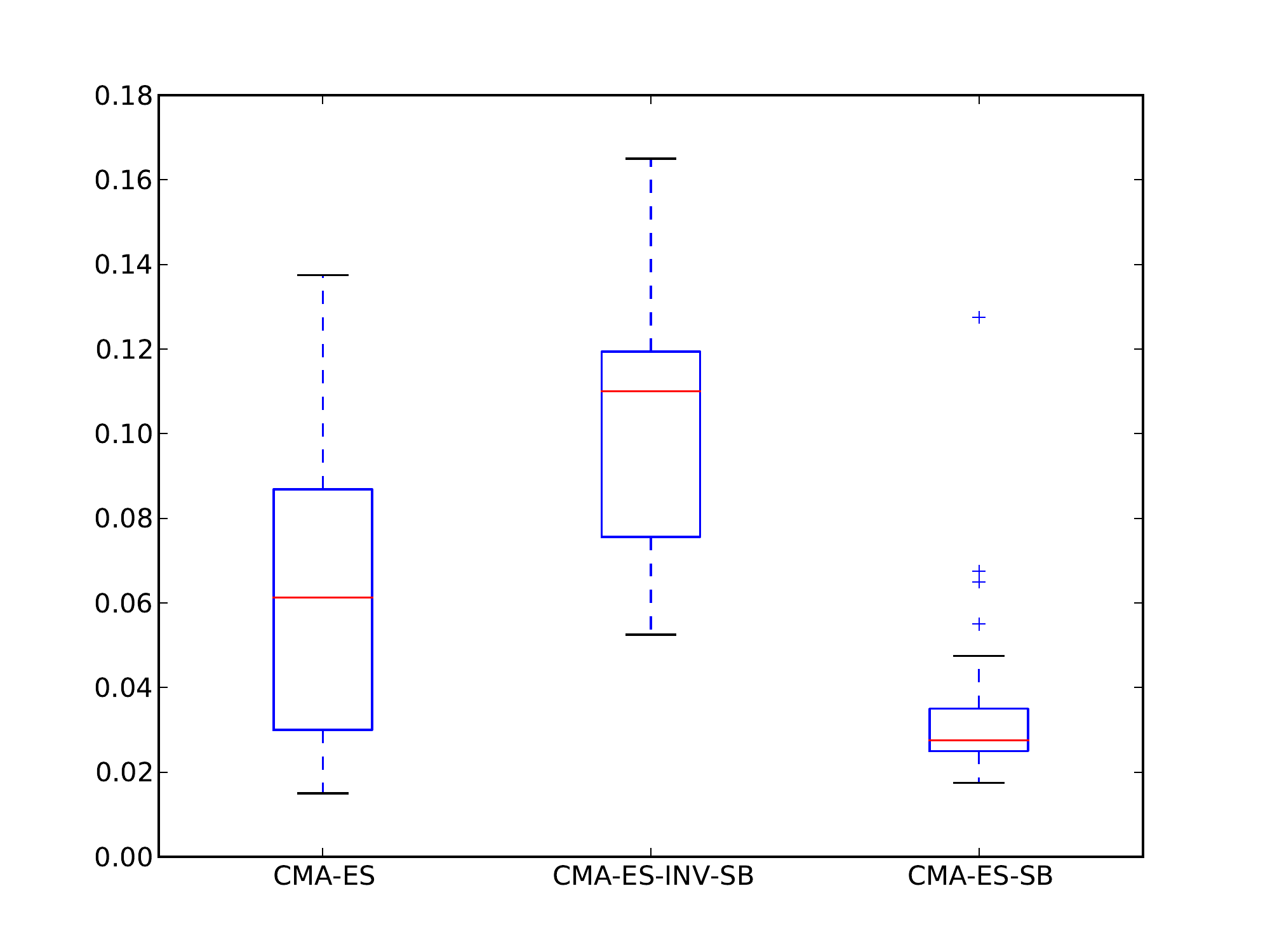}}
      \caption{\label{fig:two-circles-res-CMAES} \it Classification error rates over ANN-evaluations on the {\bf Two-Circles} dataset using the CMA-ES-variants.}
   \end{center}
\end{figure*}
All pairwise differences prove to be statistically significant. It can be seen that again DE-SB and CMA-ES-SB dominate the performances.

\clearpage
\subsubsection{Dataset: {\bf Two-Spirals}}
This problem~\cite{Lang1988} contains 2-D data-samples from two spirals on the plane, both starting at the origin and going around each other. The task is to classify each data sample by deciding to which spiral it belongs to. We use a 2-8-3-1-3-8-2 net, 114 samples for the training set, 40 samples for the validation set and 40 samples for the test set. The population size for all DE-based methods is $N_p=120$, and $N_p=1000$ for all CMA-ES-based methods. Fig.~\ref{fig:two-spirals-res-DE} and~\ref{fig:two-spirals-res-CMAES} show the resulting convergence curves and box plots for the learning process.
%%%%%%%%%%%%%%%%%%%%%%%%%%%%%%%%%%%%%%%%%%%%%%%%%%%%%%%%%%%%%%%%%%%%%%%%%%%%%%%%%%%%%%%%%%%%%%%%%%%%%%%%%%%%%%%%
% plot
\begin{figure*}[h!]
   \begin{center}
      \scalebox{0.26}{\includegraphics{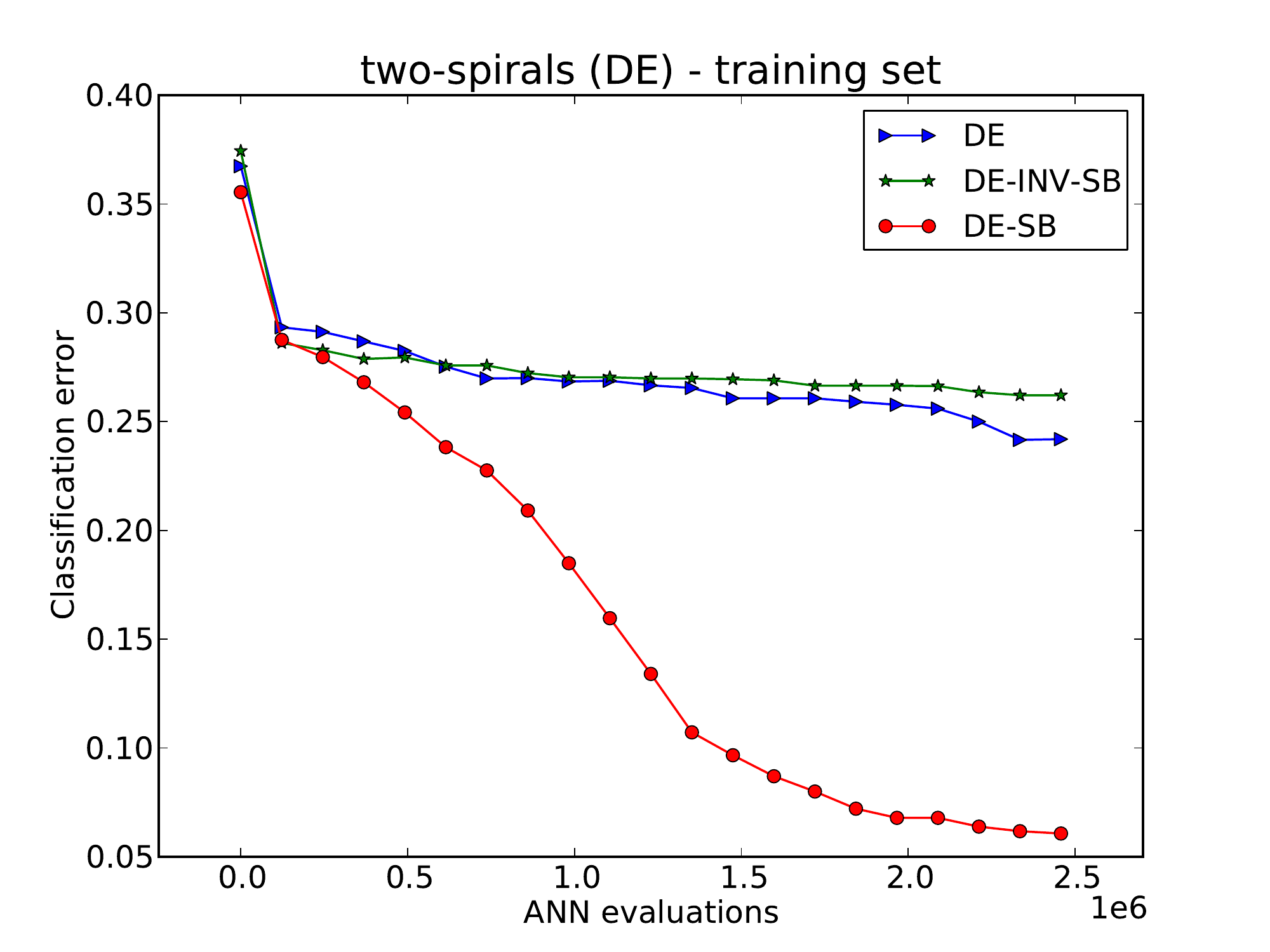}}\hspace{-0.6cm}
      \scalebox{0.26}{\includegraphics{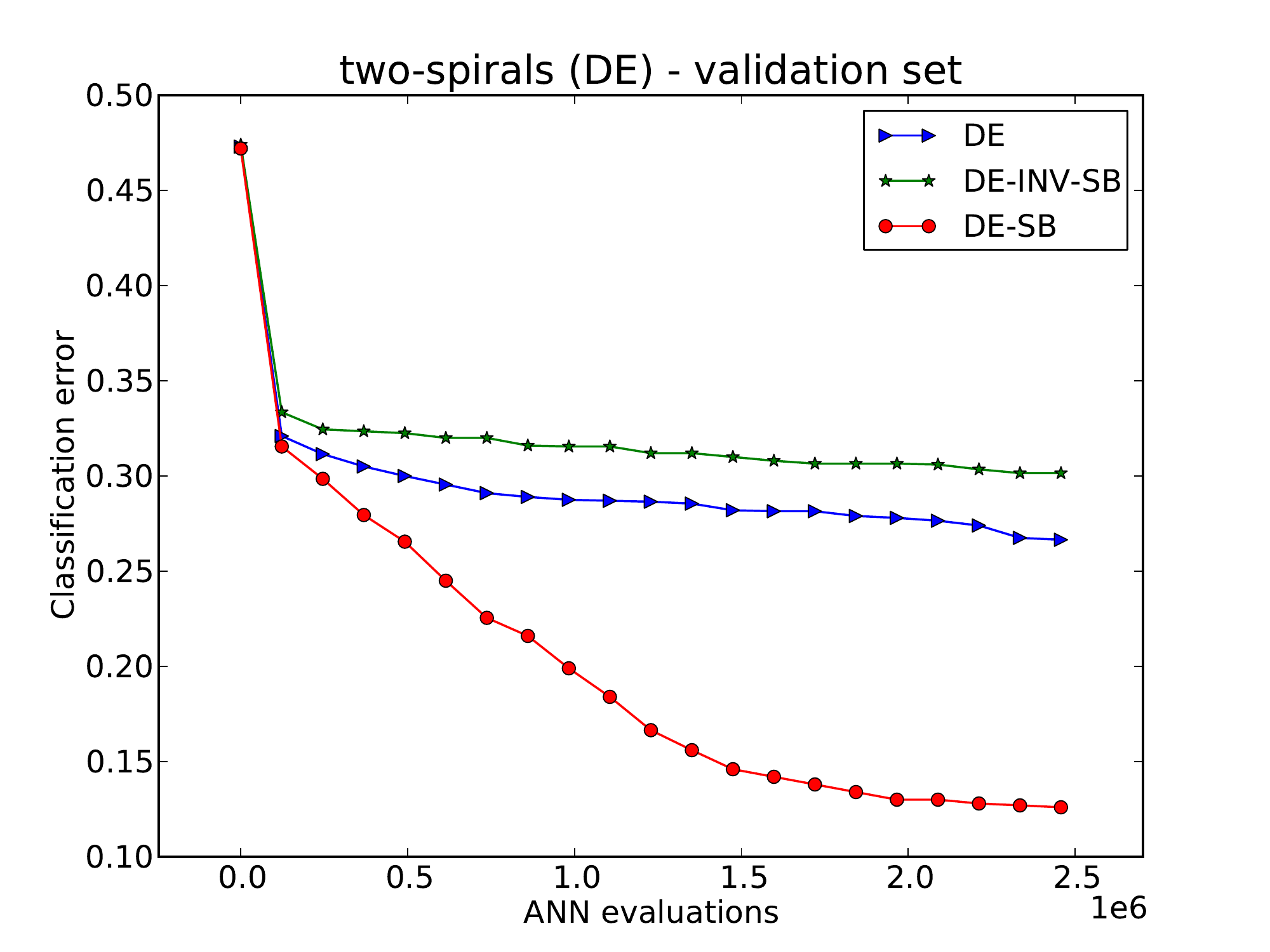}}\hspace{-0.6cm}
      \scalebox{0.26}{\includegraphics{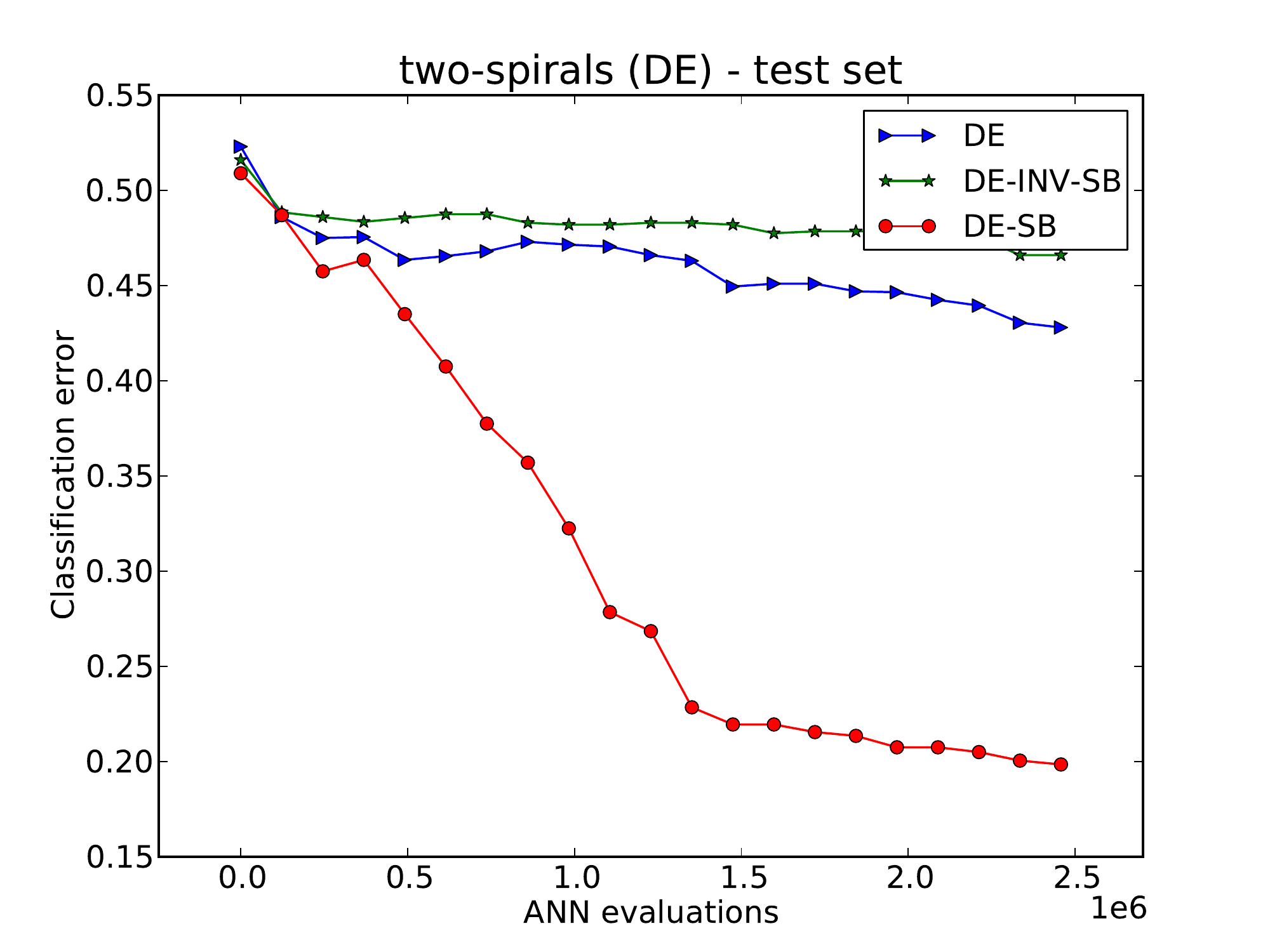}}\\
      \scalebox{0.26}{\includegraphics{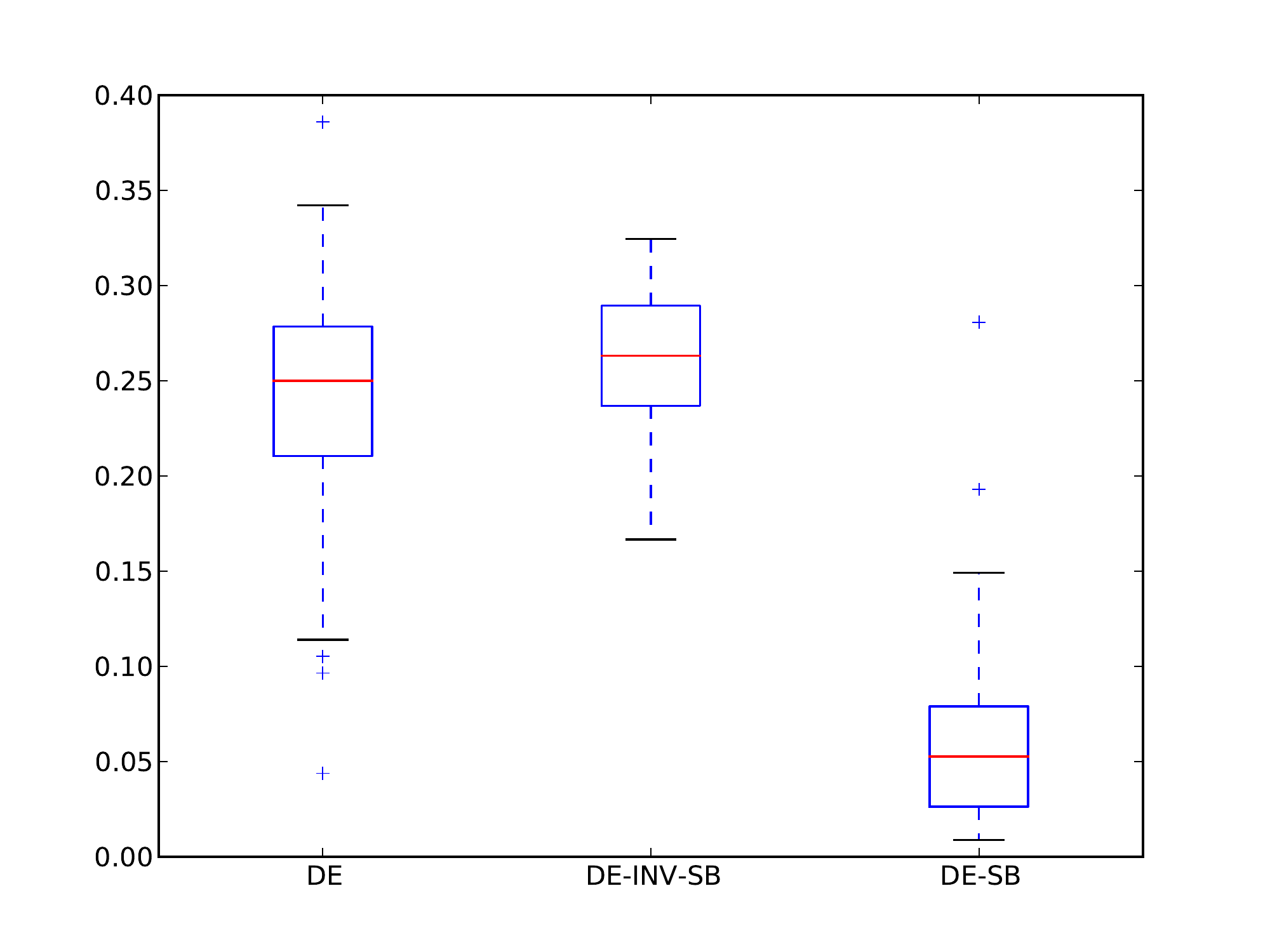}}\hspace{-0.6cm}
      \scalebox{0.26}{\includegraphics{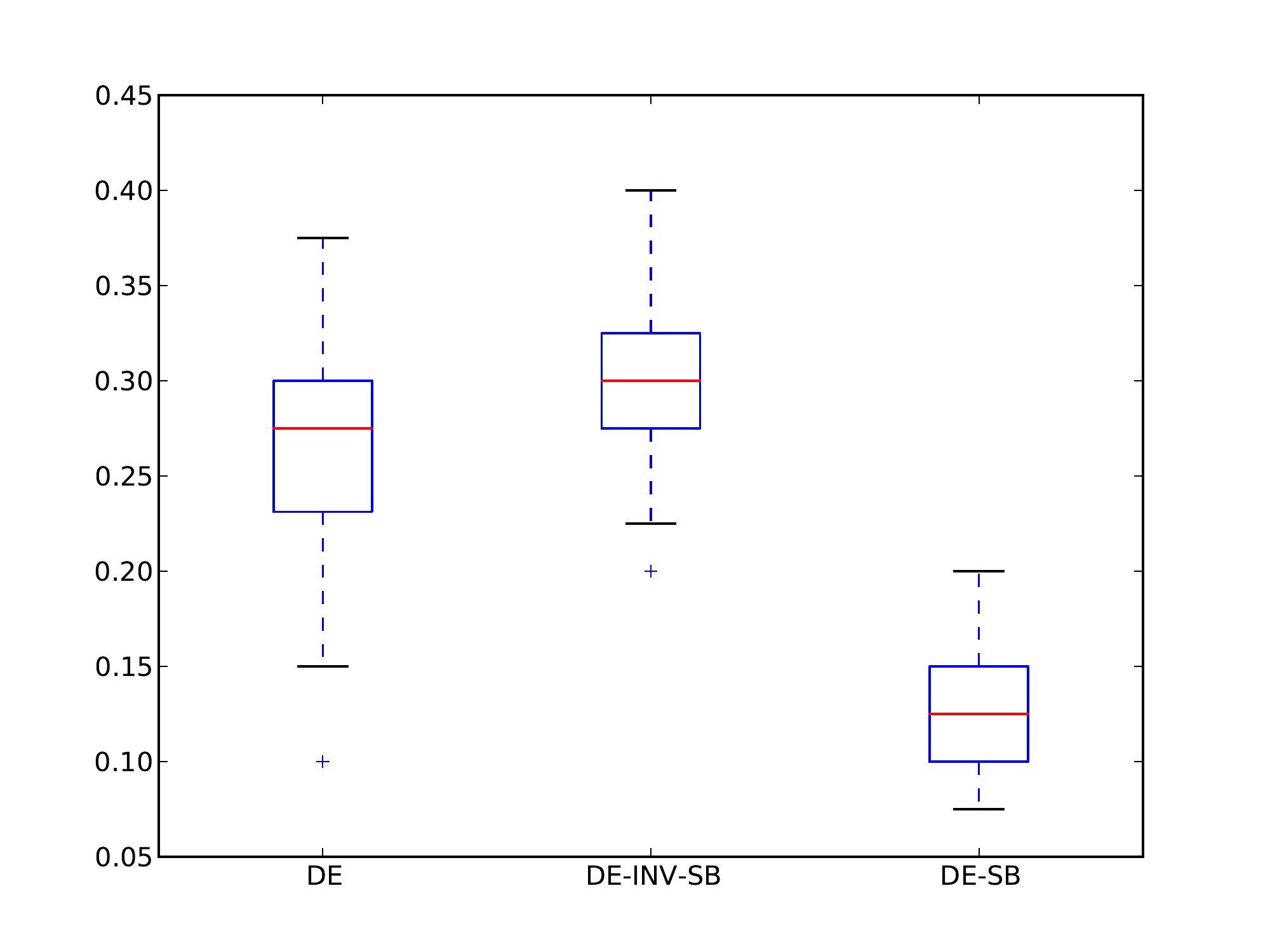}}\hspace{-0.6cm}
      \scalebox{0.26}{\includegraphics{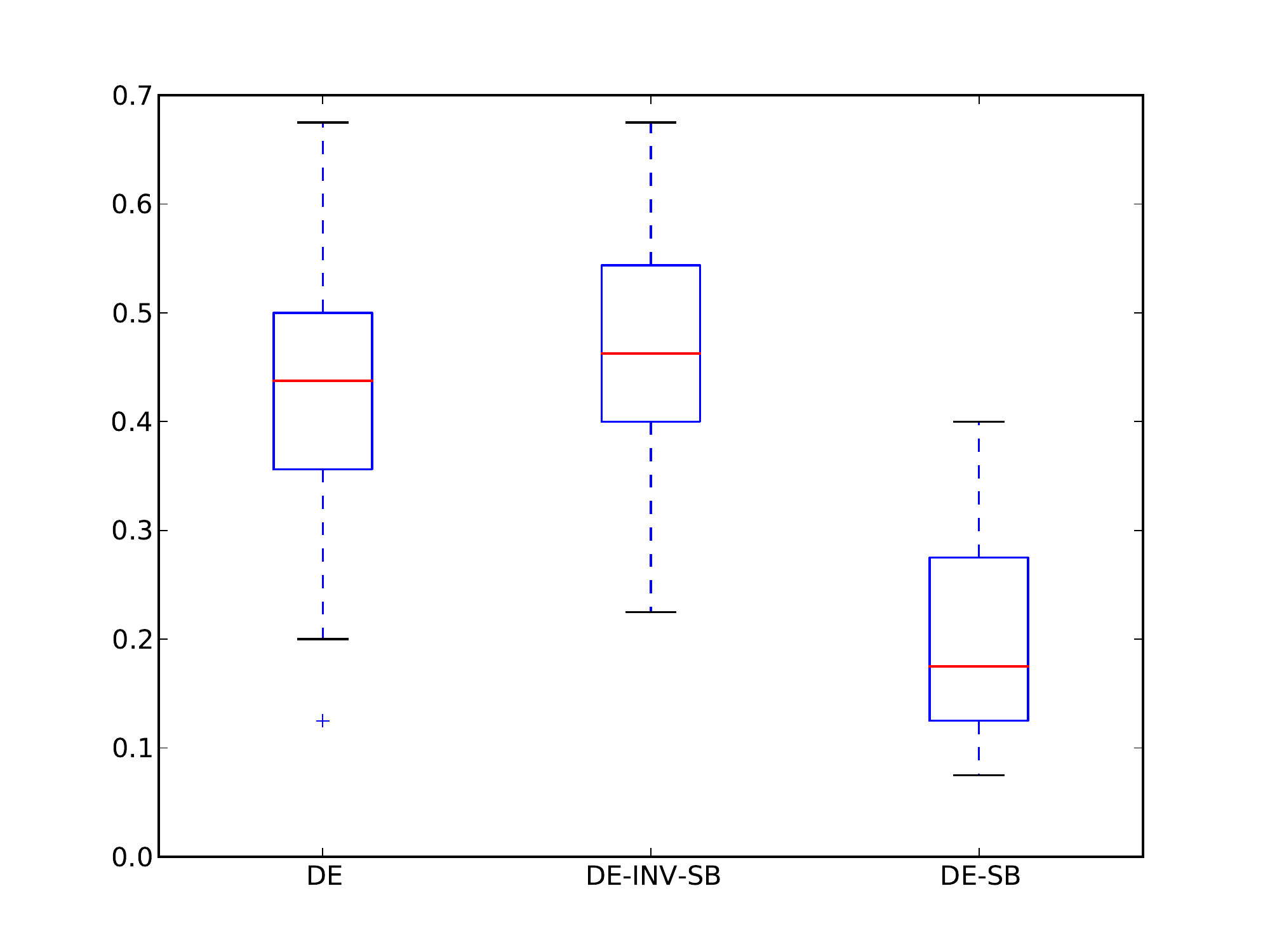}}
      \caption{\label{fig:two-spirals-res-DE} \it Classification error rates over ANN-evaluations on the {\bf Two-Spirals} dataset using the DE-variants.}
   \end{center}
\end{figure*}
%%%%%%%%%%%%%%%%%%%%%%%%%%%%%%%%%%%%%%%%%%%%%%%%%%%%%%%%%%%%%%%%%%%%%%%%%%%%%%%%%%%%%%%%%%%%%%%%%%%%%%%%%%%%%%%%
% plot
\begin{figure*}[h!]
   \begin{center}
      \scalebox{0.26}{\includegraphics{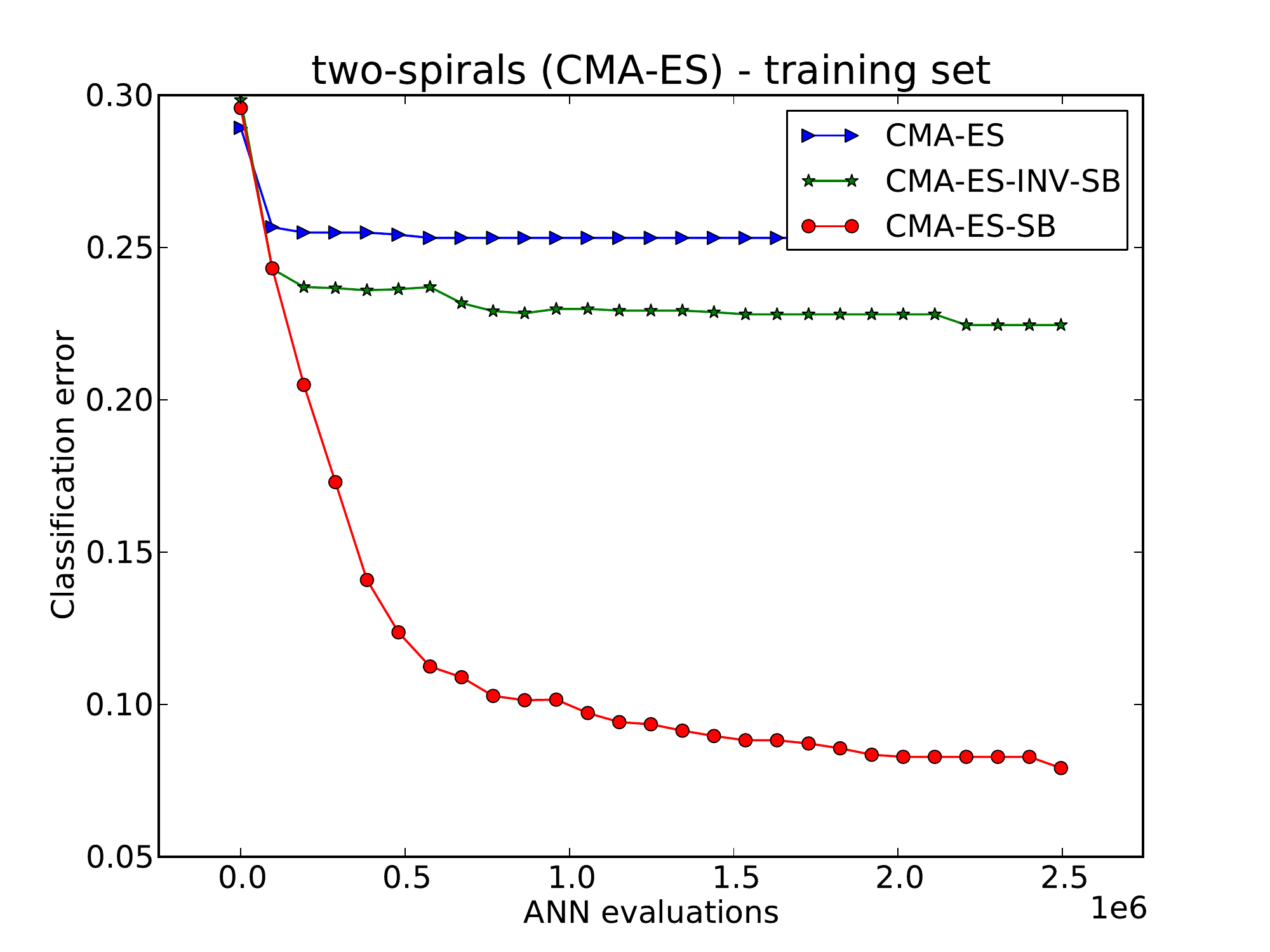}}\hspace{-0.6cm}
      \scalebox{0.26}{\includegraphics{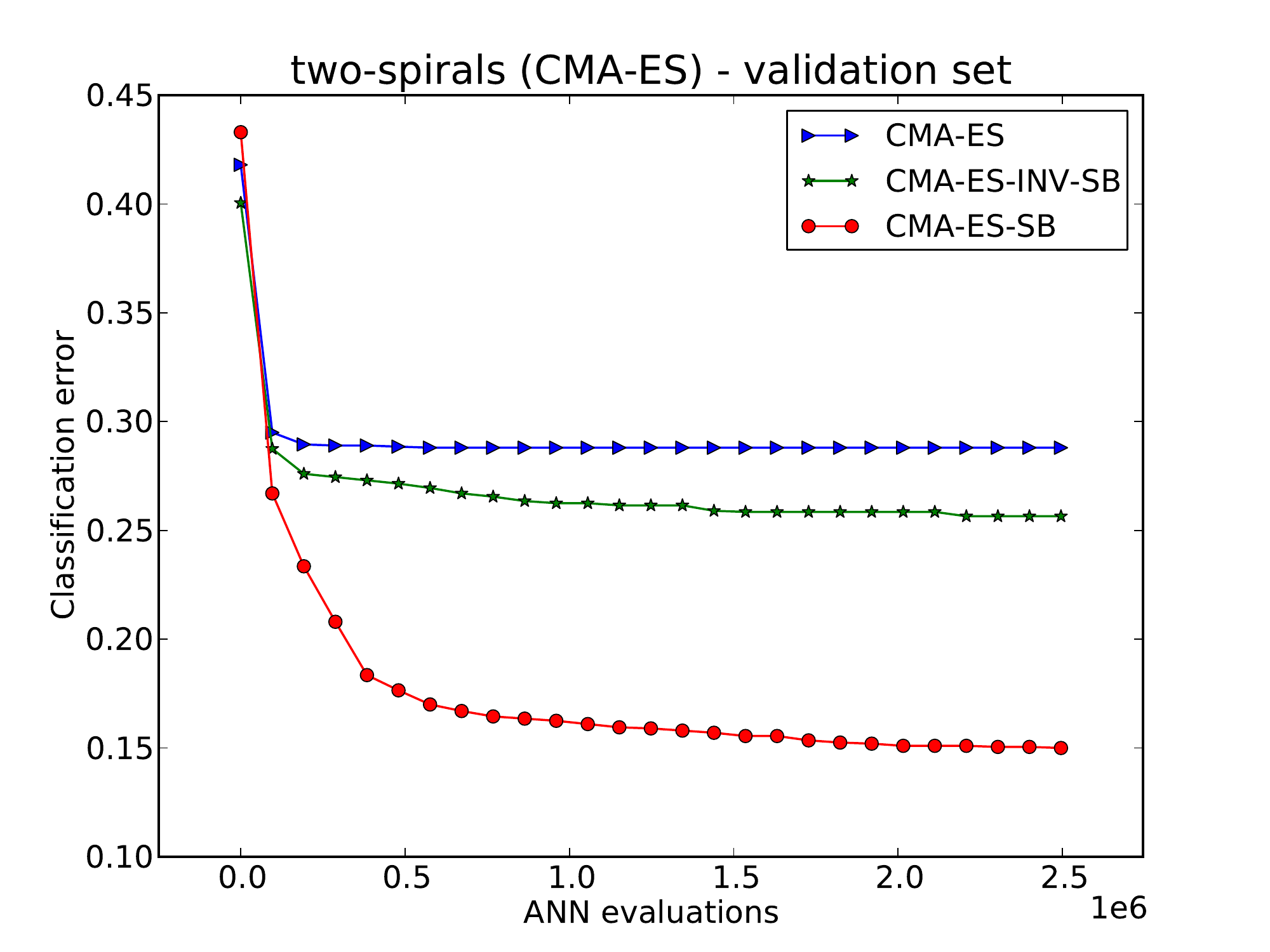}}\hspace{-0.6cm}
      \scalebox{0.26}{\includegraphics{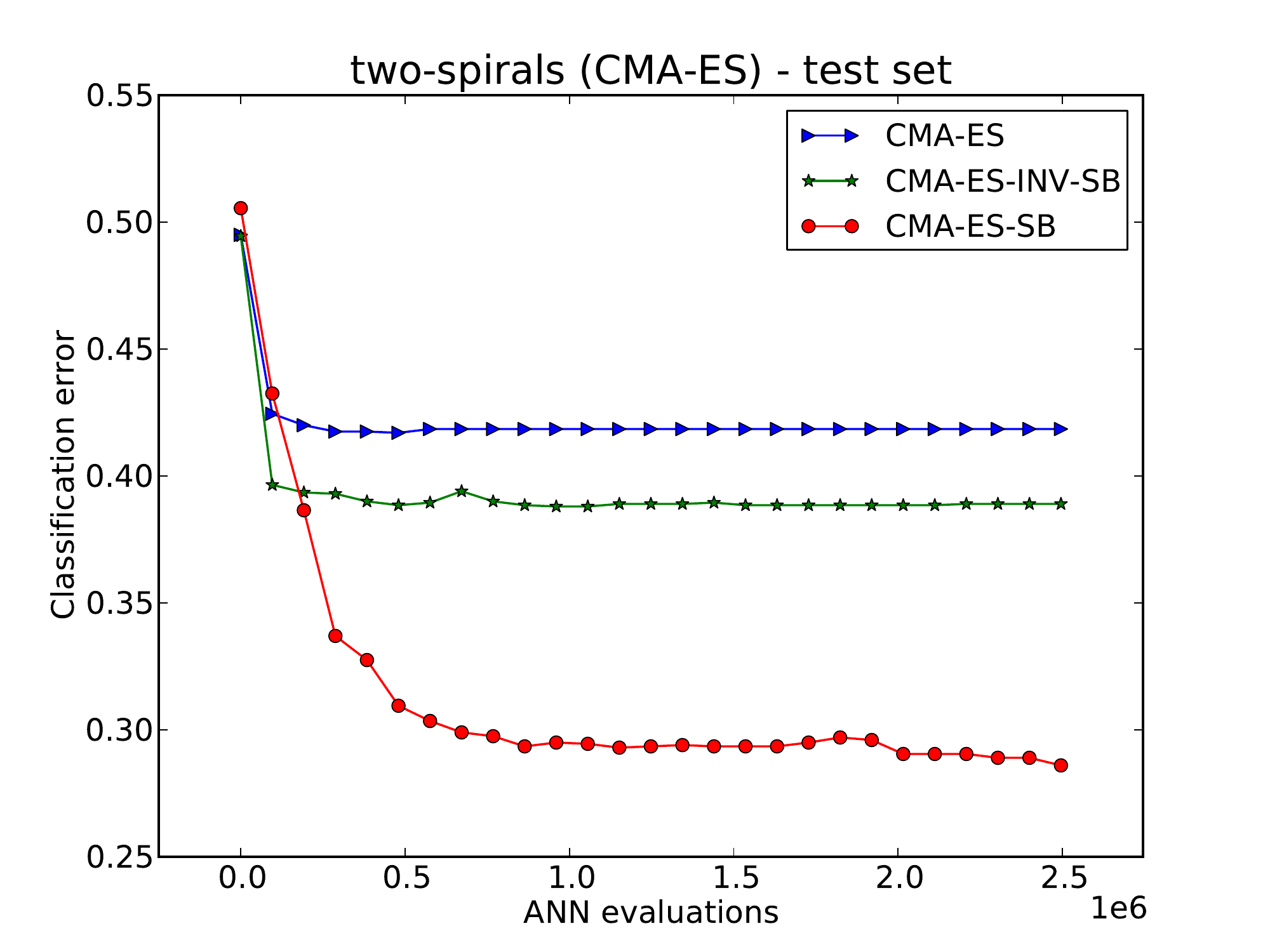}}\\
      \scalebox{0.26}{\includegraphics{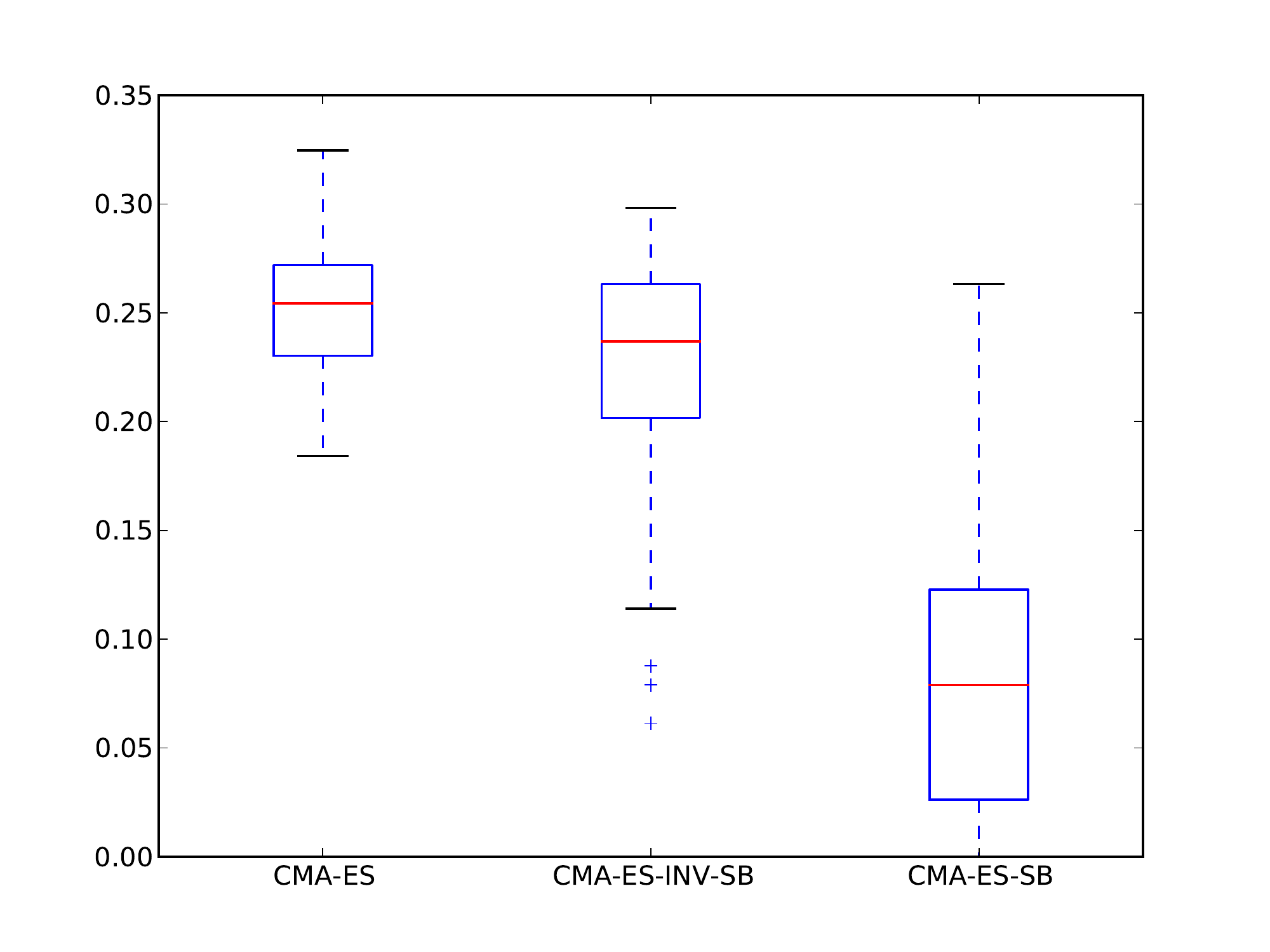}}\hspace{-0.6cm}
      \scalebox{0.26}{\includegraphics{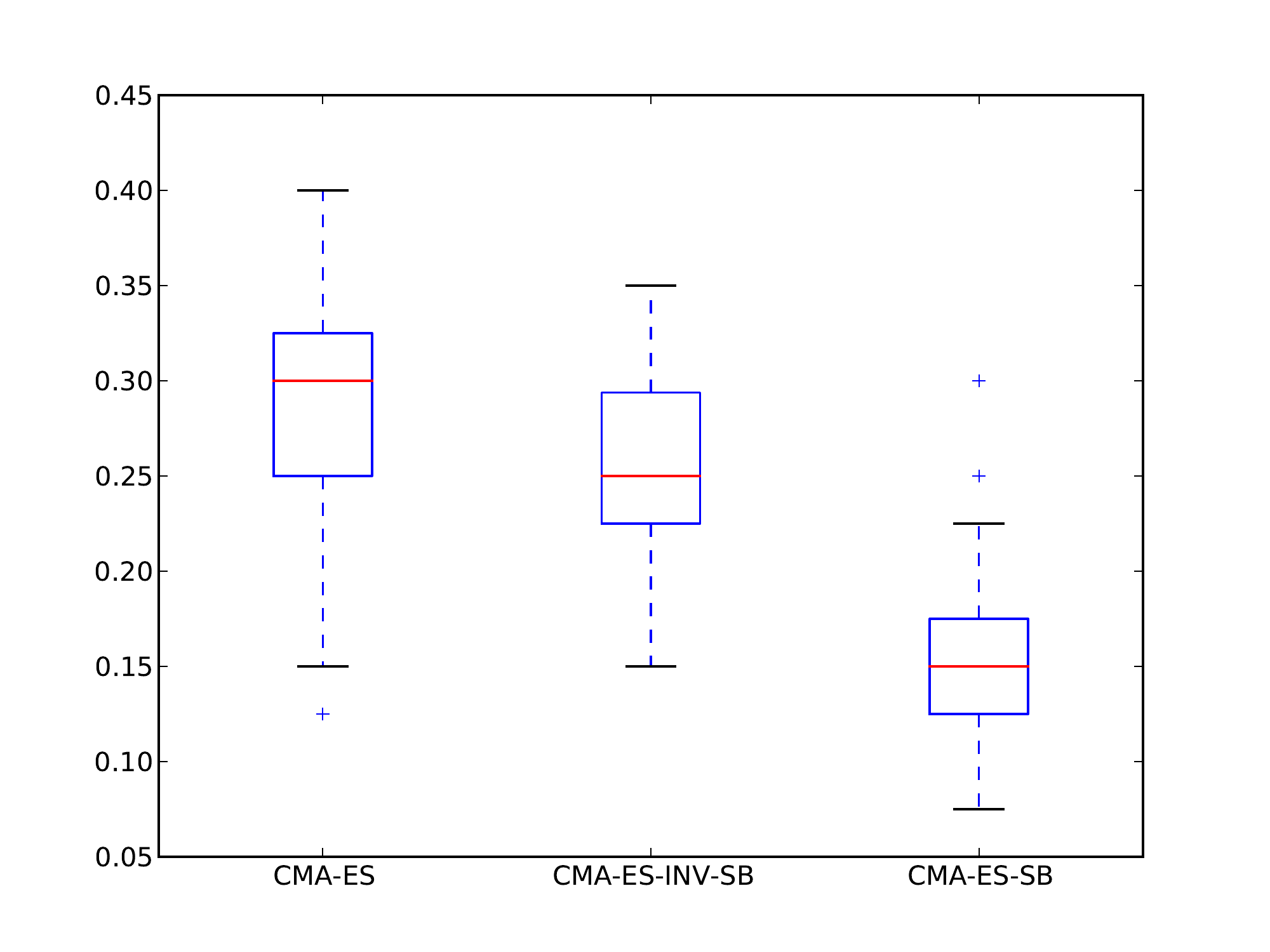}}\hspace{-0.6cm}
      \scalebox{0.26}{\includegraphics{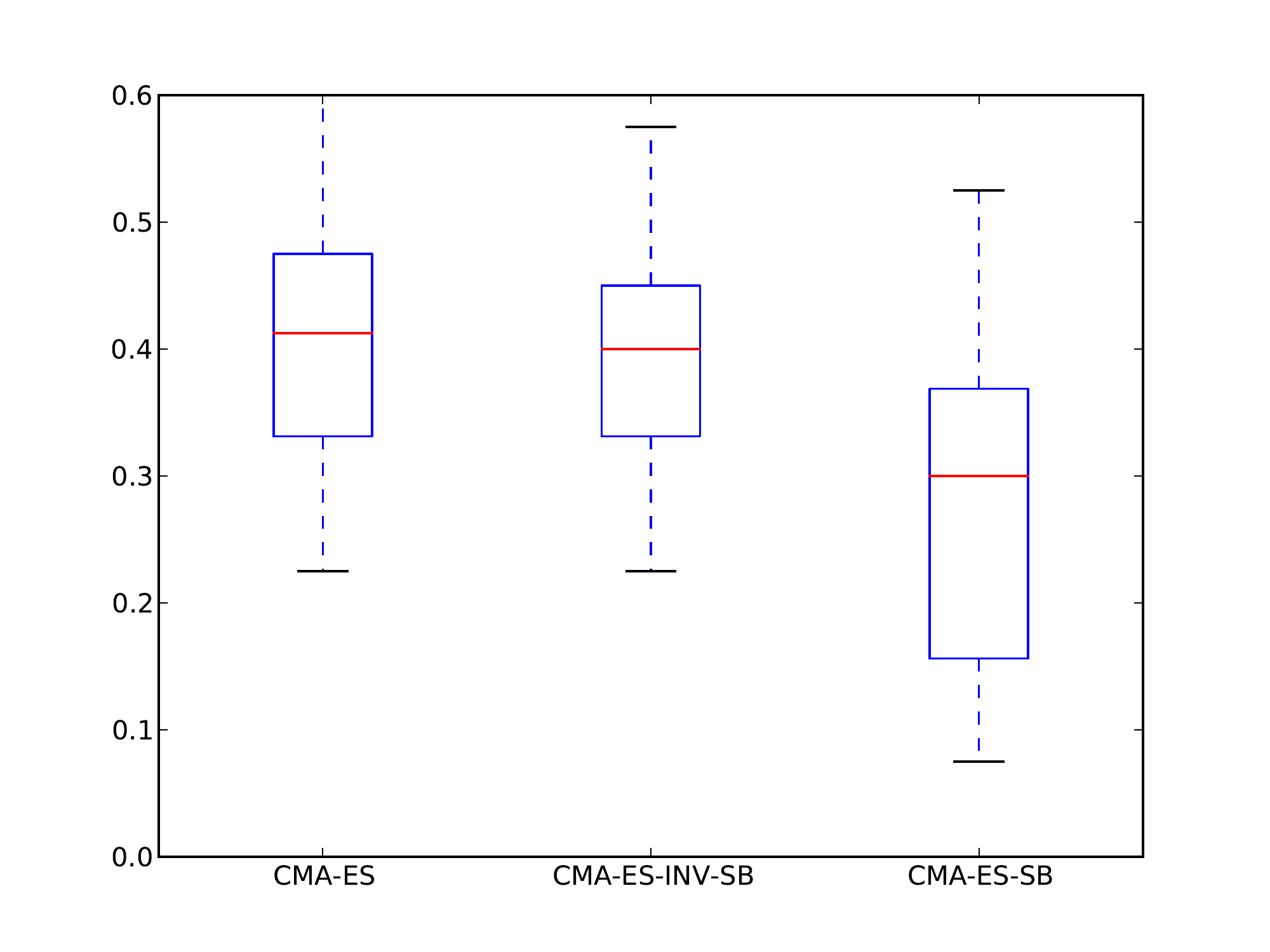}}
      \caption{\label{fig:two-spirals-res-CMAES} \it Classification error rates over ANN-evaluations on the {\bf Two-Spirals} dataset using the CMA-ES-variants.}
   \end{center}
\end{figure*}
On the training and test set, DE and DE-INV-SB mean results are not statistical significantly different. Furthermore, on the test set, CMA-ES and CMA-ES-INV-SB mean results are not significantly different. Otherwise, all other pairwise differences prove to be statistically significant. DE-SB and CMA-ES-SB continue to show the best results.

\clearpage
\subsubsection{Dataset: {\bf Digits}}
This problem~\cite{Alpaydin-digits-data-set} deals with the recognition of handwritten digits, which results in a classification problem with 10 classes. The data is generated by asking several writers to write 250 digits in random order inside boxes of 500 by 500 tablet pixel resolution. There are 16 features extracted from the digitized data. We use a 16-8-3-10-10 net and 1000 data samples each for the training set, the validation set as well as the test set.
%%%%%%%%%%%%%%%%%%%%%%%%%%%%%%%%%%%%%%%%%%%%%%%%%%%%%%%%%%%%%%%%%%%%%%%%%%%%%%%%%%%%%%%%%%%%%%%%%%%%%%%%%%%%%%%%
% plot
\begin{figure*}[h!]
   \begin{center}
      \scalebox{0.26}{\includegraphics{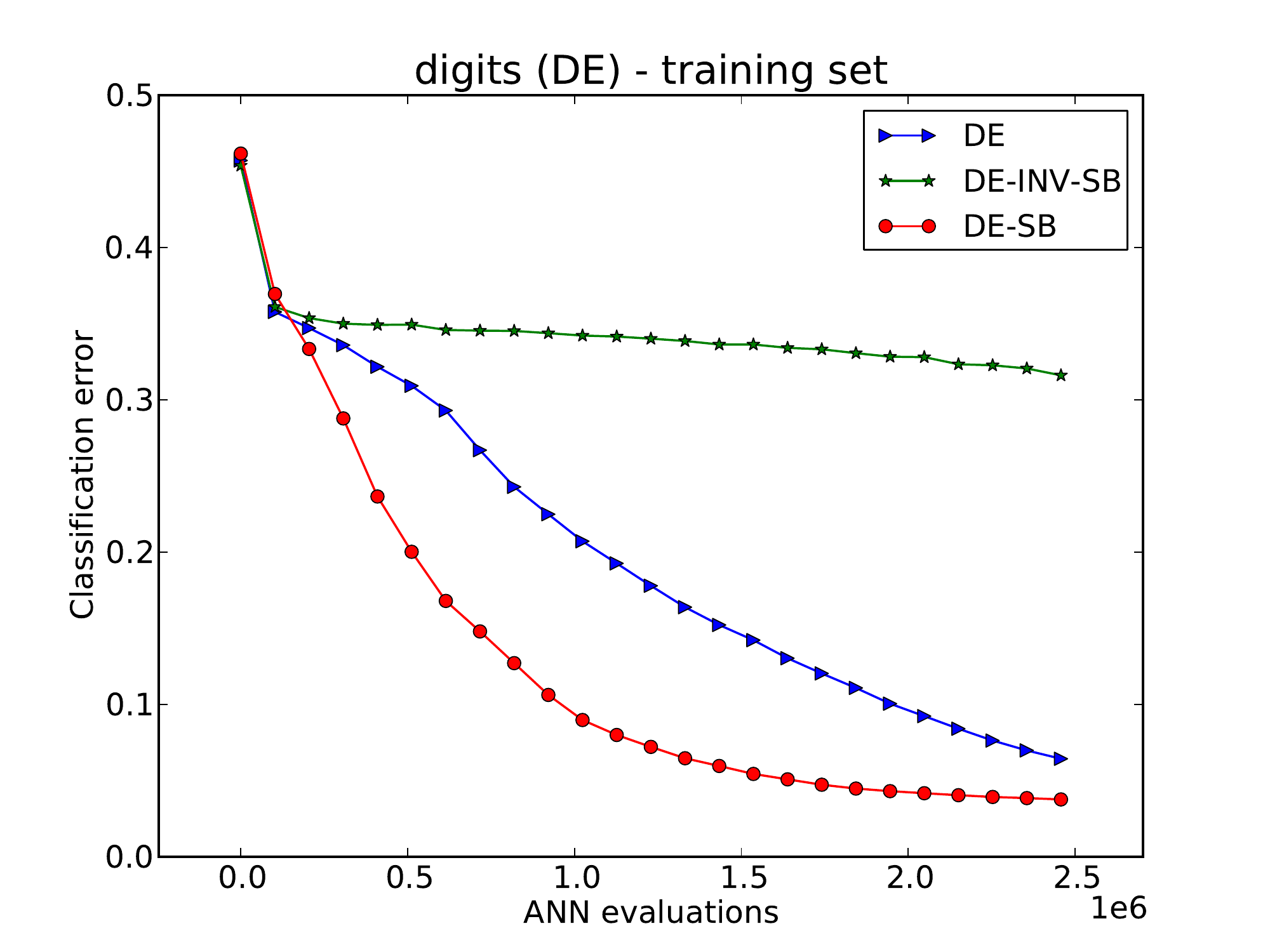}}\hspace{-0.6cm}
      \scalebox{0.26}{\includegraphics{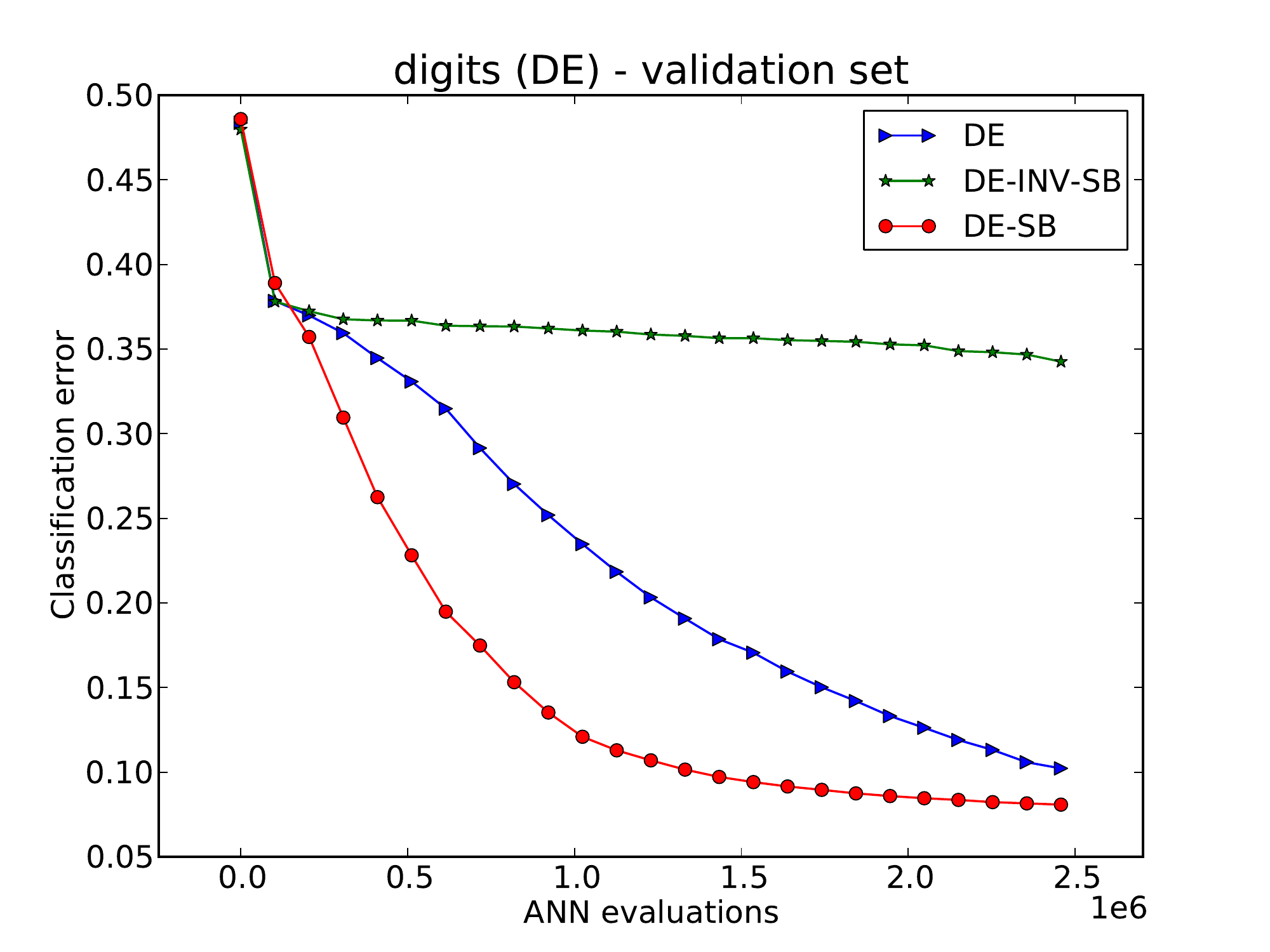}}\hspace{-0.6cm}
      \scalebox{0.26}{\includegraphics{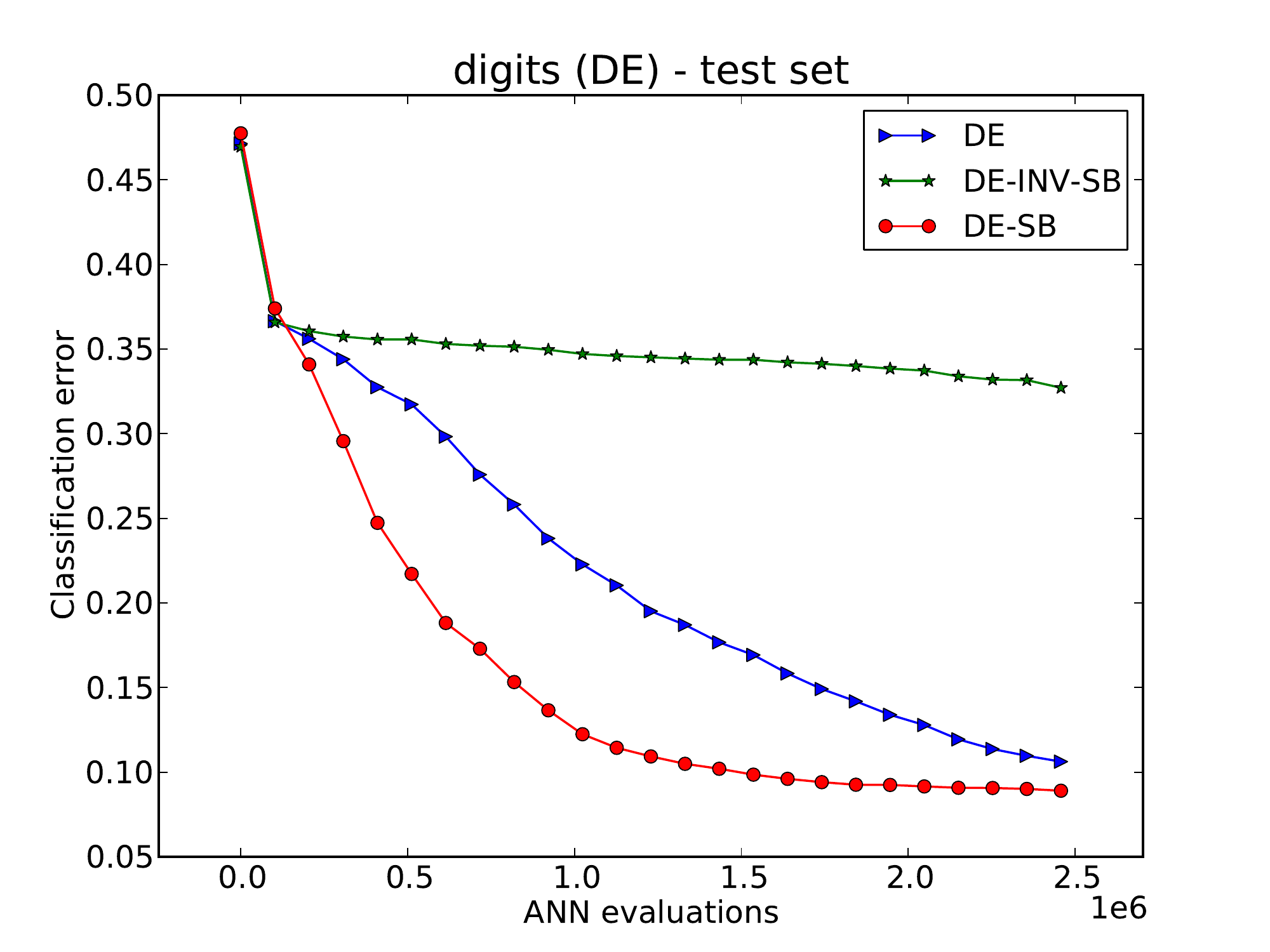}}\\
      \scalebox{0.26}{\includegraphics{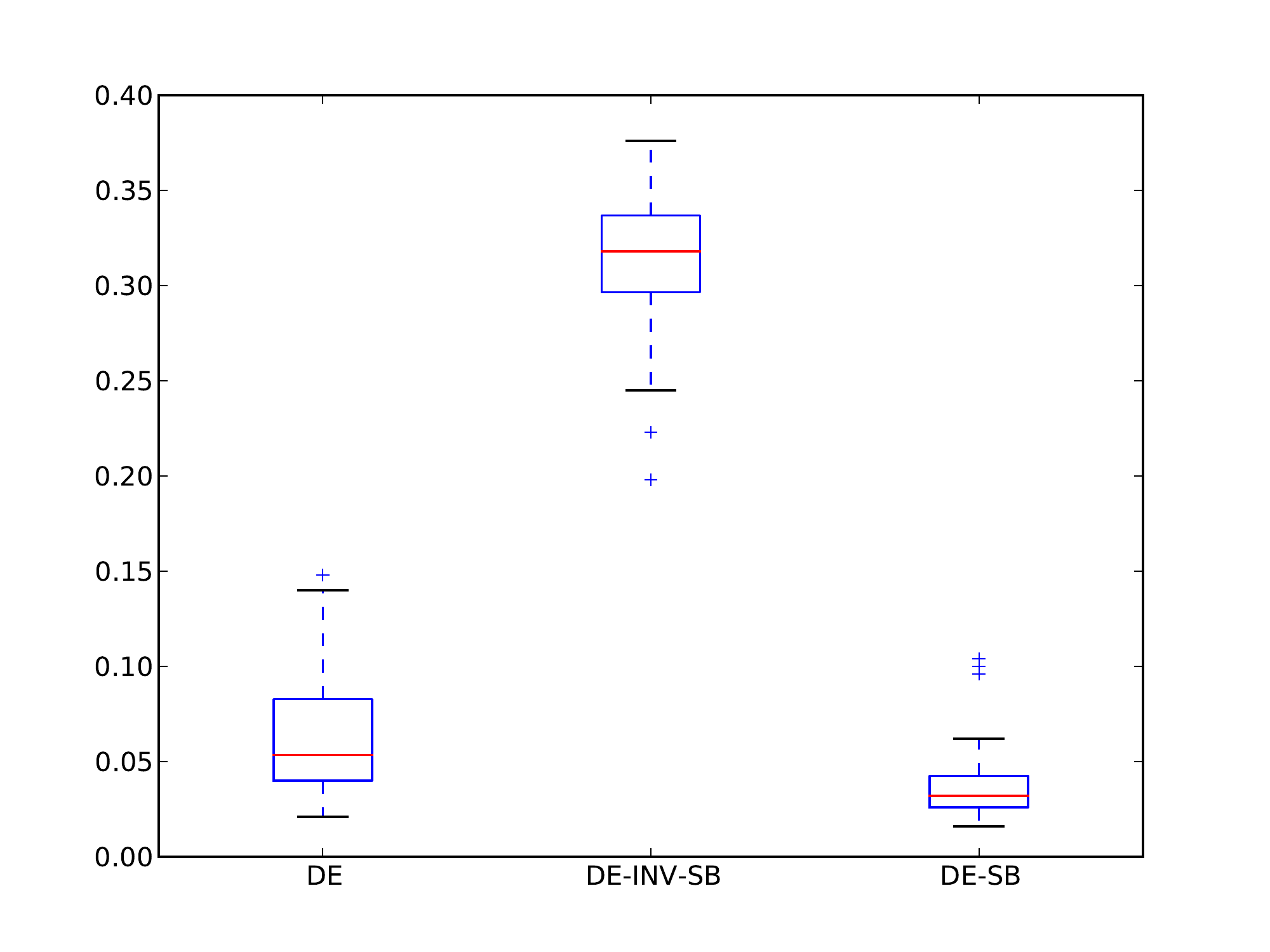}}\hspace{-0.6cm}
      \scalebox{0.26}{\includegraphics{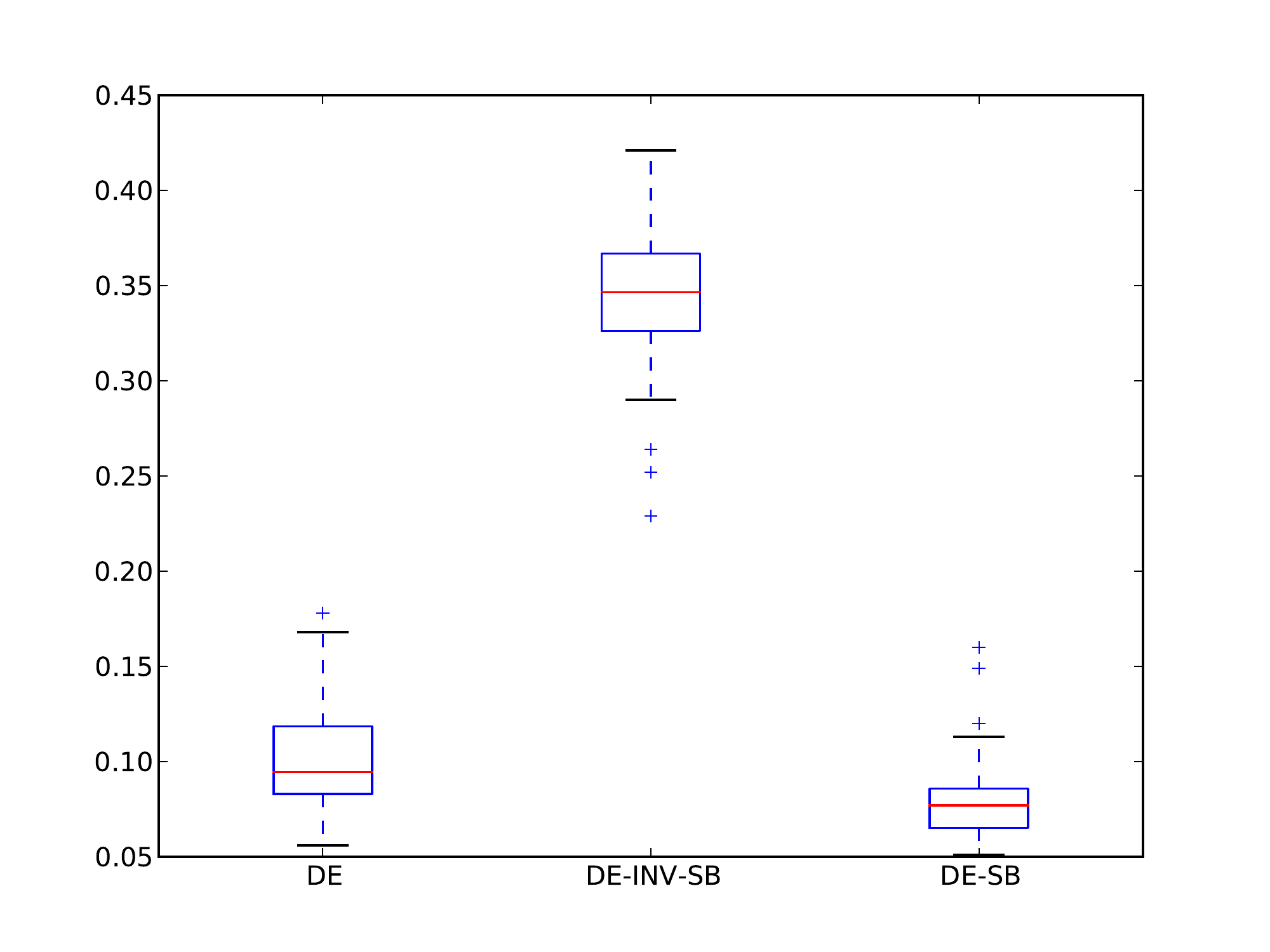}}\hspace{-0.6cm}
      \scalebox{0.26}{\includegraphics{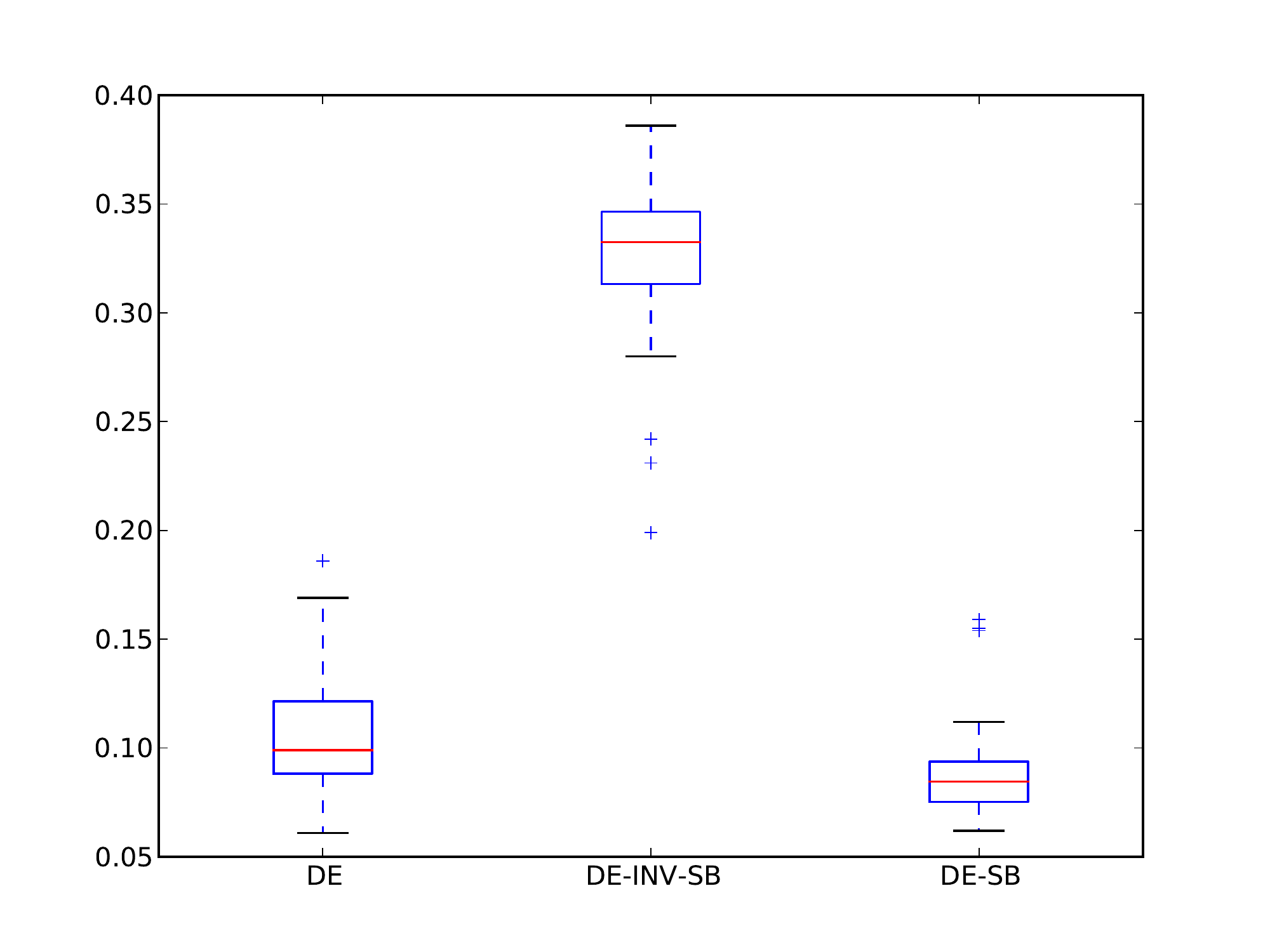}}
      \caption{\label{fig:digits-res-DE} \it Classification error rates over ANN-evaluations on the {\bf Digits} dataset using the DE-variants.}
   \end{center}
\end{figure*}
%%%%%%%%%%%%%%%%%%%%%%%%%%%%%%%%%%%%%%%%%%%%%%%%%%%%%%%%%%%%%%%%%%%%%%%%%%%%%%%%%%%%%%%%%%%%%%%%%%%%%%%%%%%%%%%%
% plot
\begin{figure*}[h!]
   \begin{center}
      \scalebox{0.26}{\includegraphics{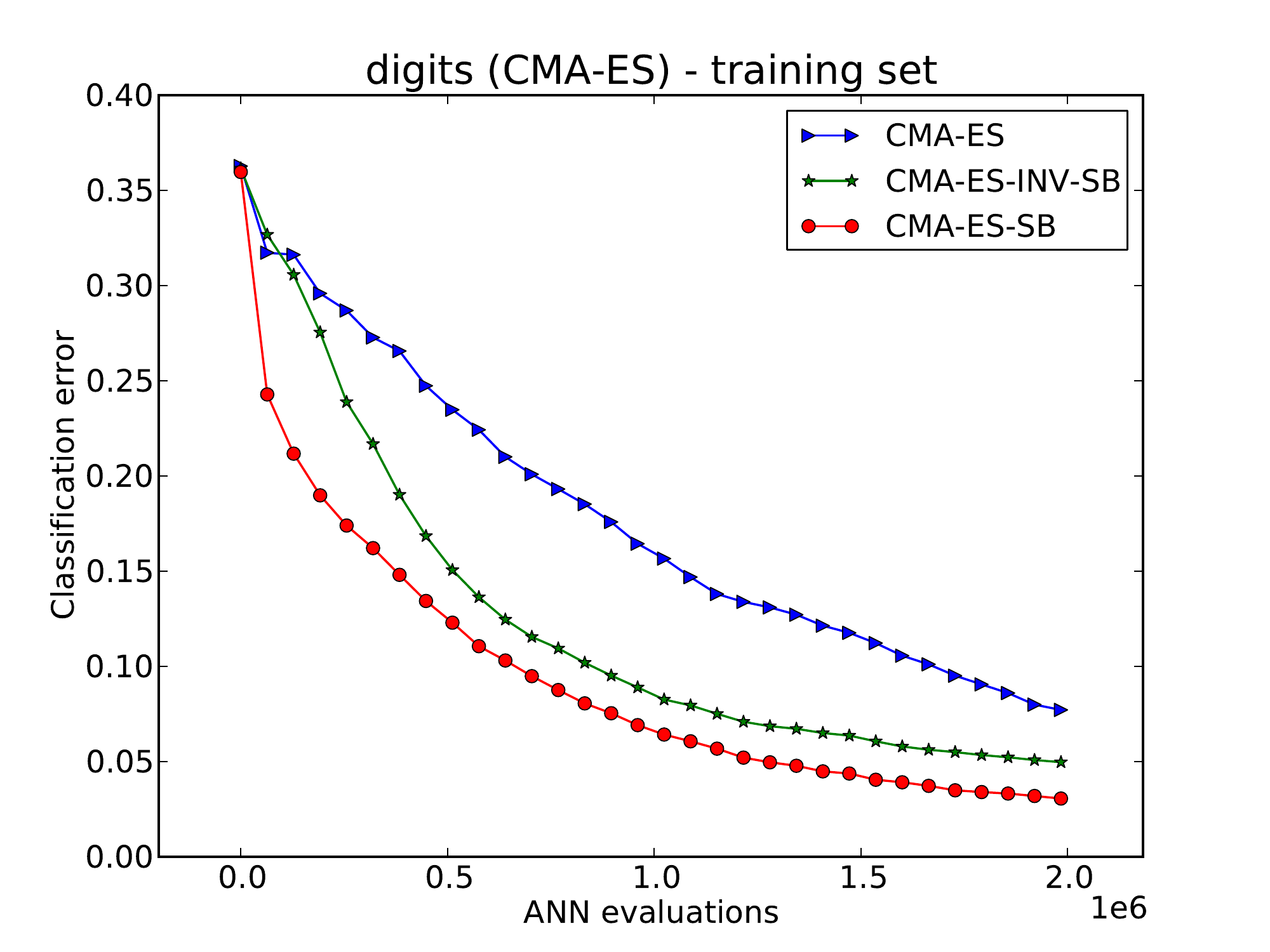}}\hspace{-0.6cm}
      \scalebox{0.26}{\includegraphics{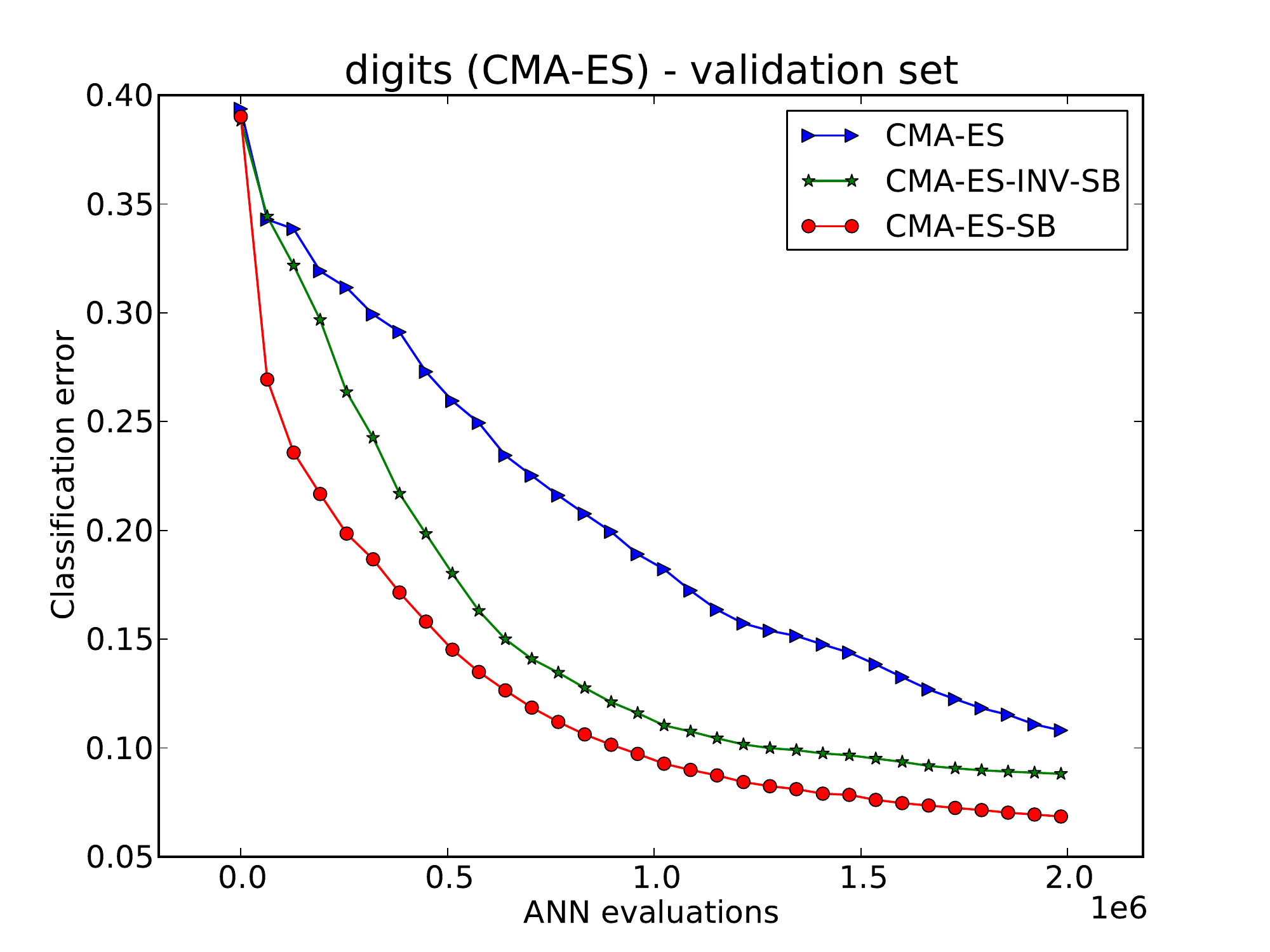}}\hspace{-0.6cm}
      \scalebox{0.26}{\includegraphics{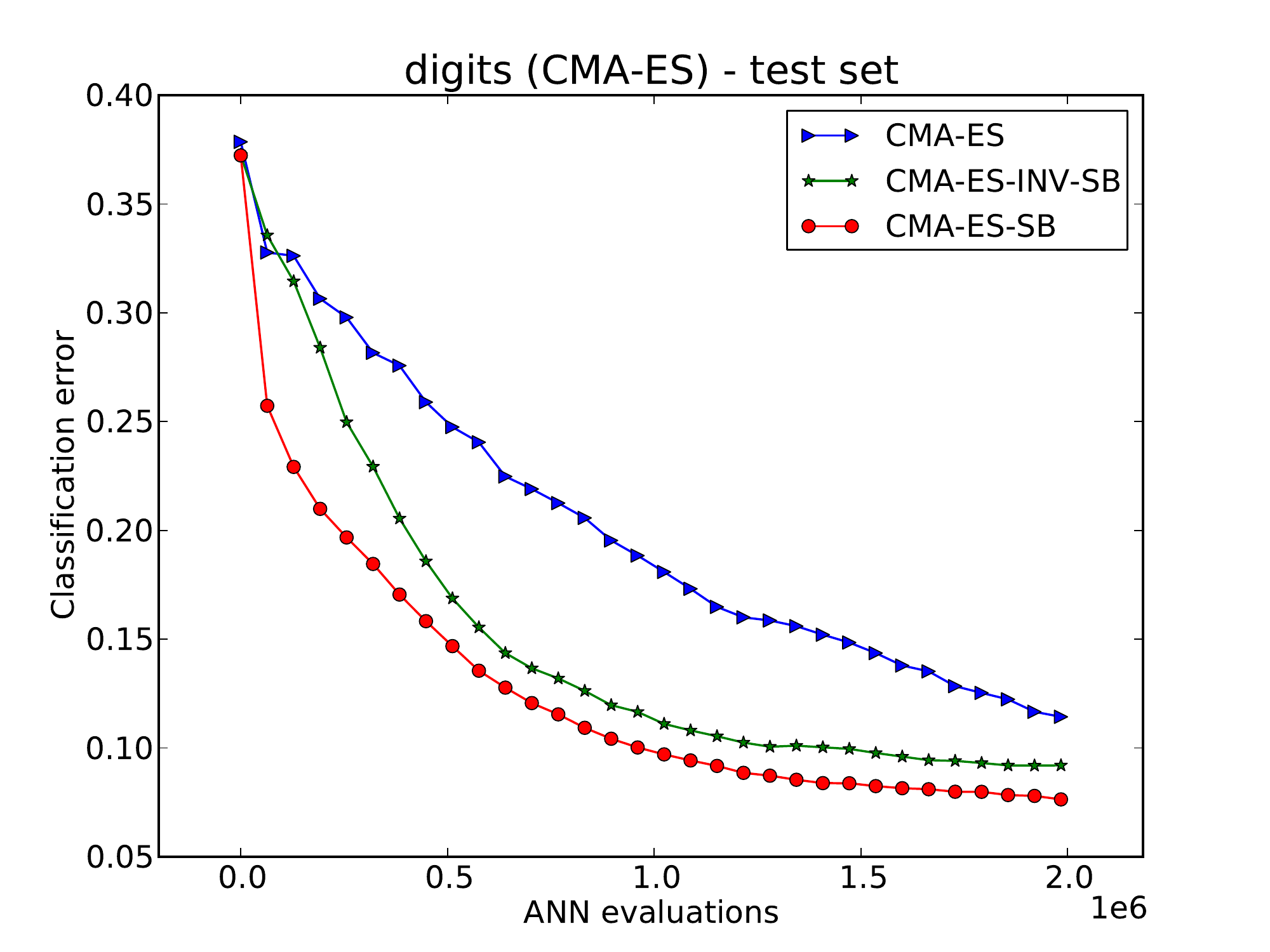}}\\
      \scalebox{0.26}{\includegraphics{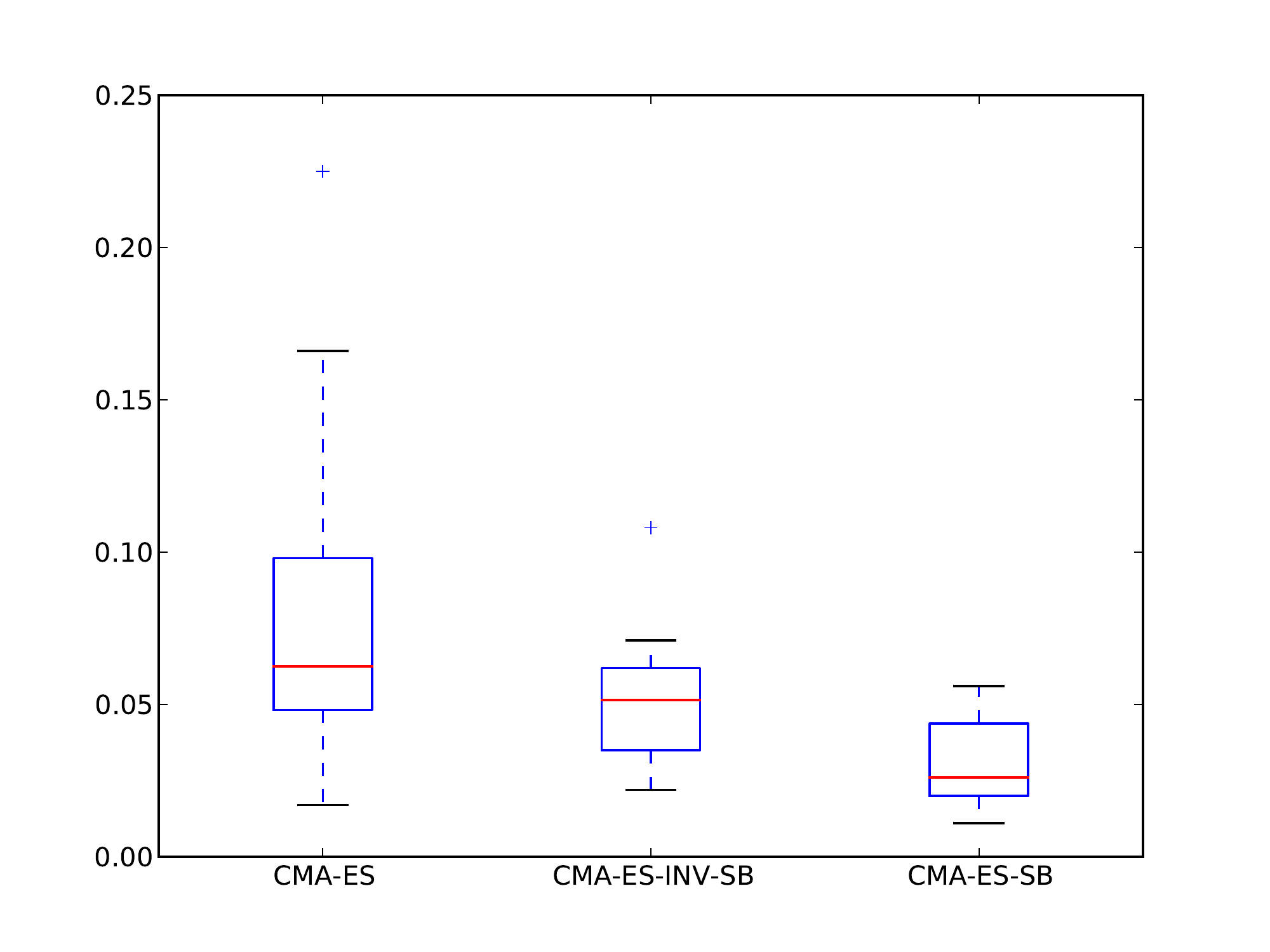}}\hspace{-0.6cm}
      \scalebox{0.26}{\includegraphics{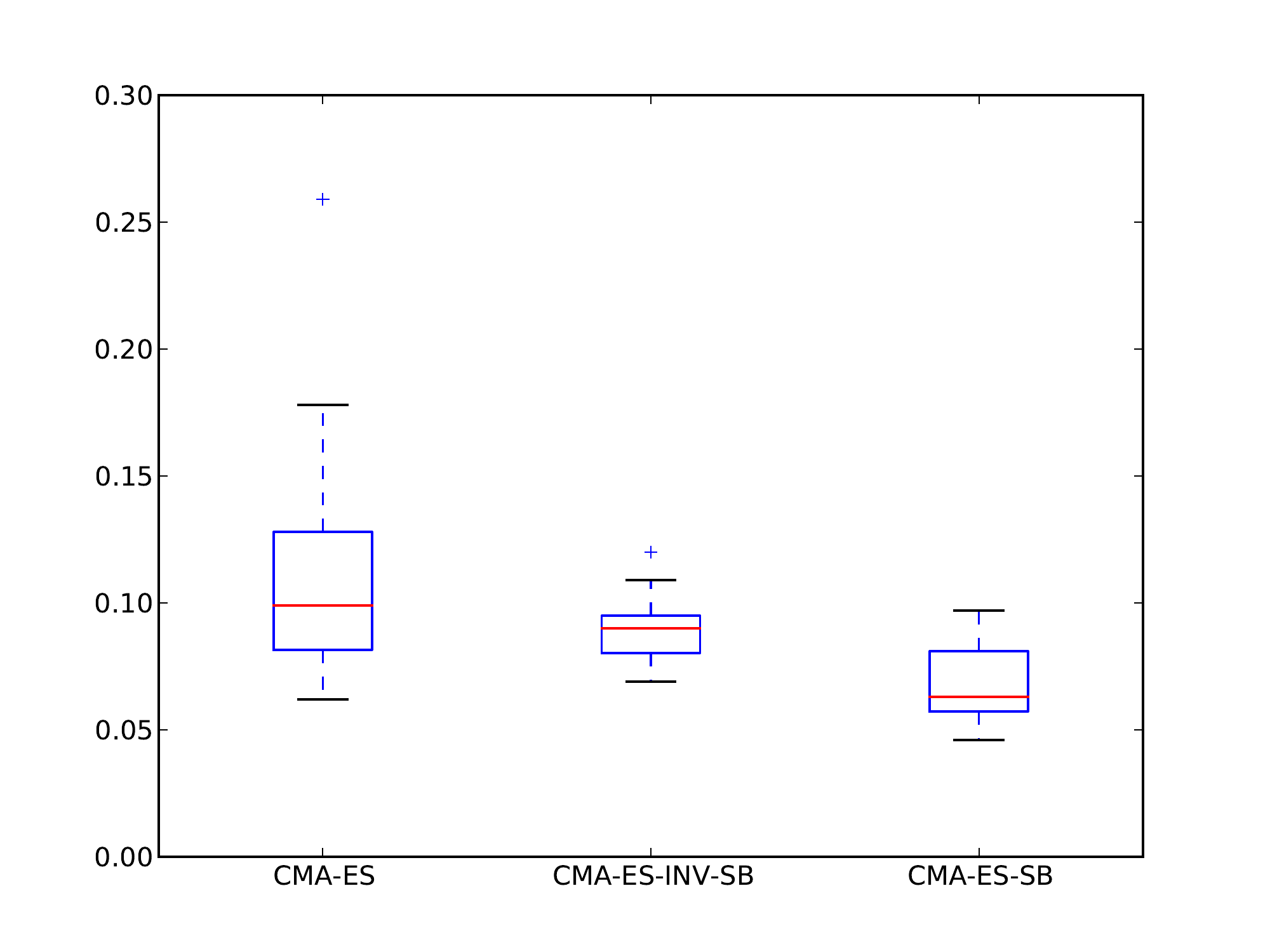}}\hspace{-0.6cm}
      \scalebox{0.26}{\includegraphics{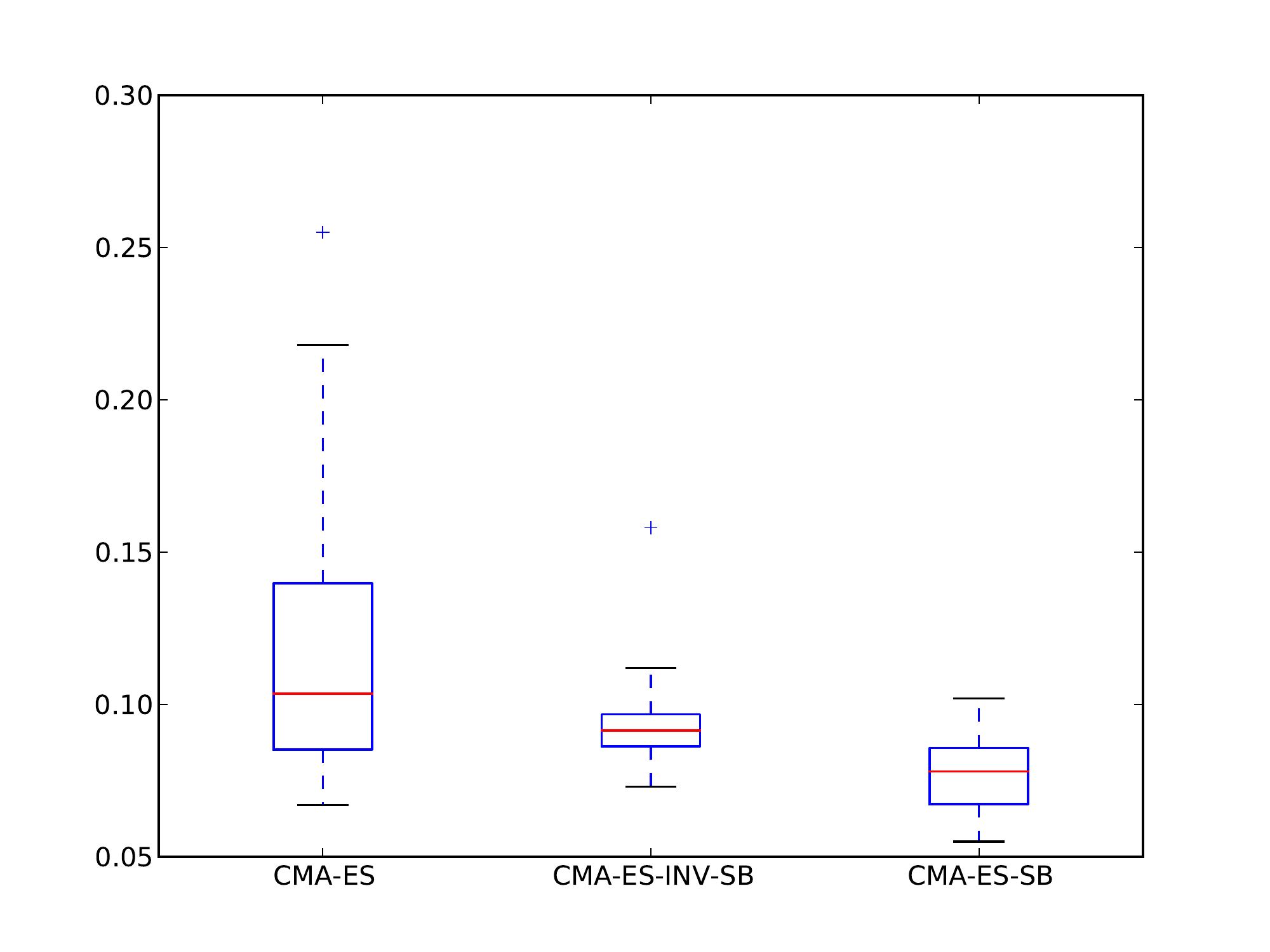}}
      \caption{\label{fig:digits-res-CMAES} \it Classification error rates over ANN-evaluations on the {\bf Digits} dataset using the CMA-ES-variants.}
   \end{center}
\end{figure*}
All pairwise differences prove to be statistically significant. Again, DE-SB and CMA-ES-SB are the fastest methods.
\clearpage

\subsection{Ideal separation}
In this Section, we compare the ideal separation to the proposed approximations. Since the complexity of the brute force method for the ideal separation is exponential, we restrict the experiments to small networks as used in the problems {\bf syn5}, {\bf sinc} and {\bf inc-sinc}. It can be seen that the results are almost identical. \\
%%%%%%%%%%%%%%%%%%%%%%%%%%%%%%%%%%%%%%%%%%%%%%%%%%%%%%%%%%%%%%%%%%%%%%%%%%%%%%%%%%%%%%%%%%%%%%%%%%%%%%%%%%%%%%%%
% plot
\begin{figure*}[h!]
   \begin{center}
      \scalebox{0.26}{\includegraphics{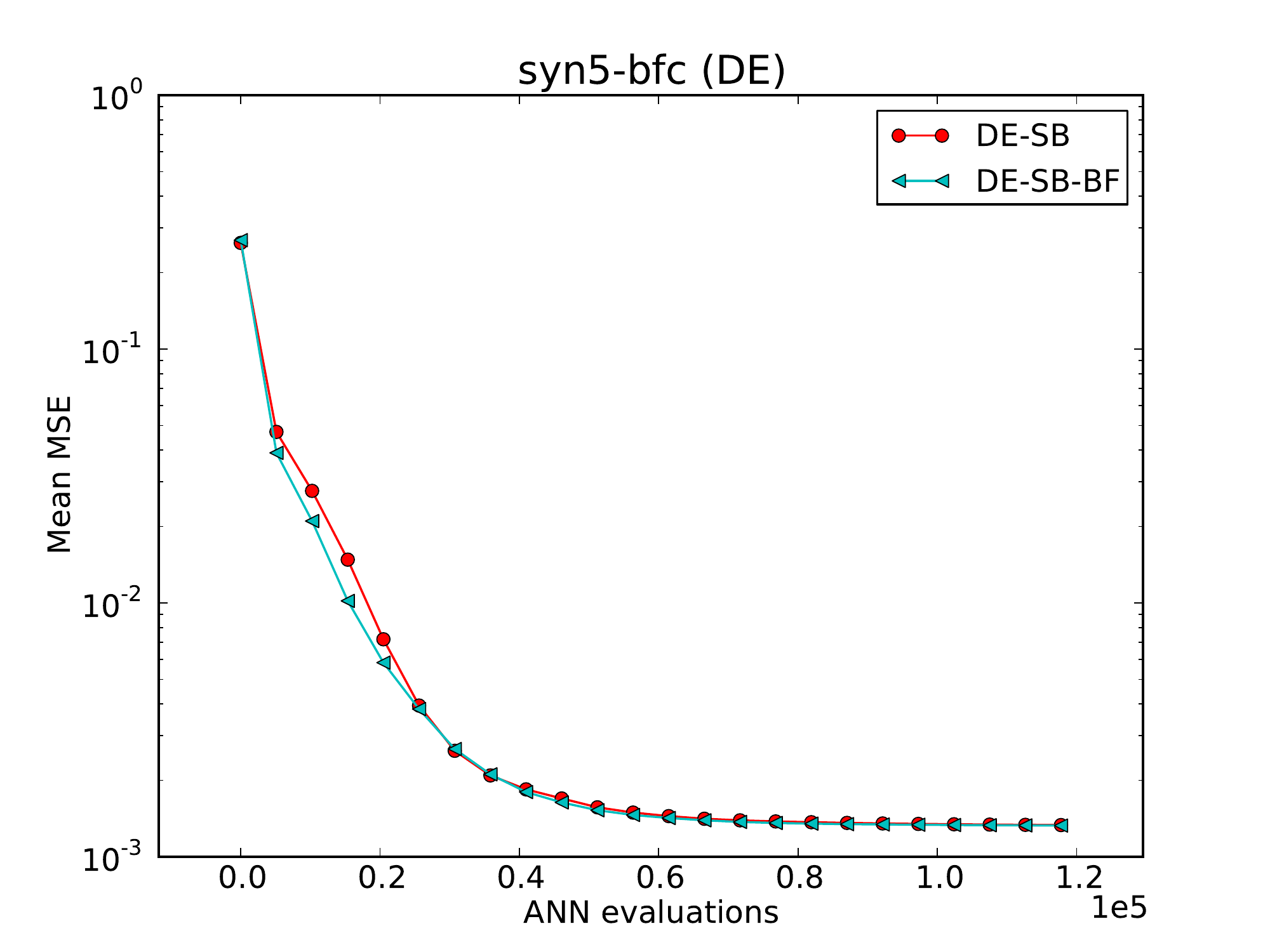}}\hspace{-0.6cm}
      \scalebox{0.26}{\includegraphics{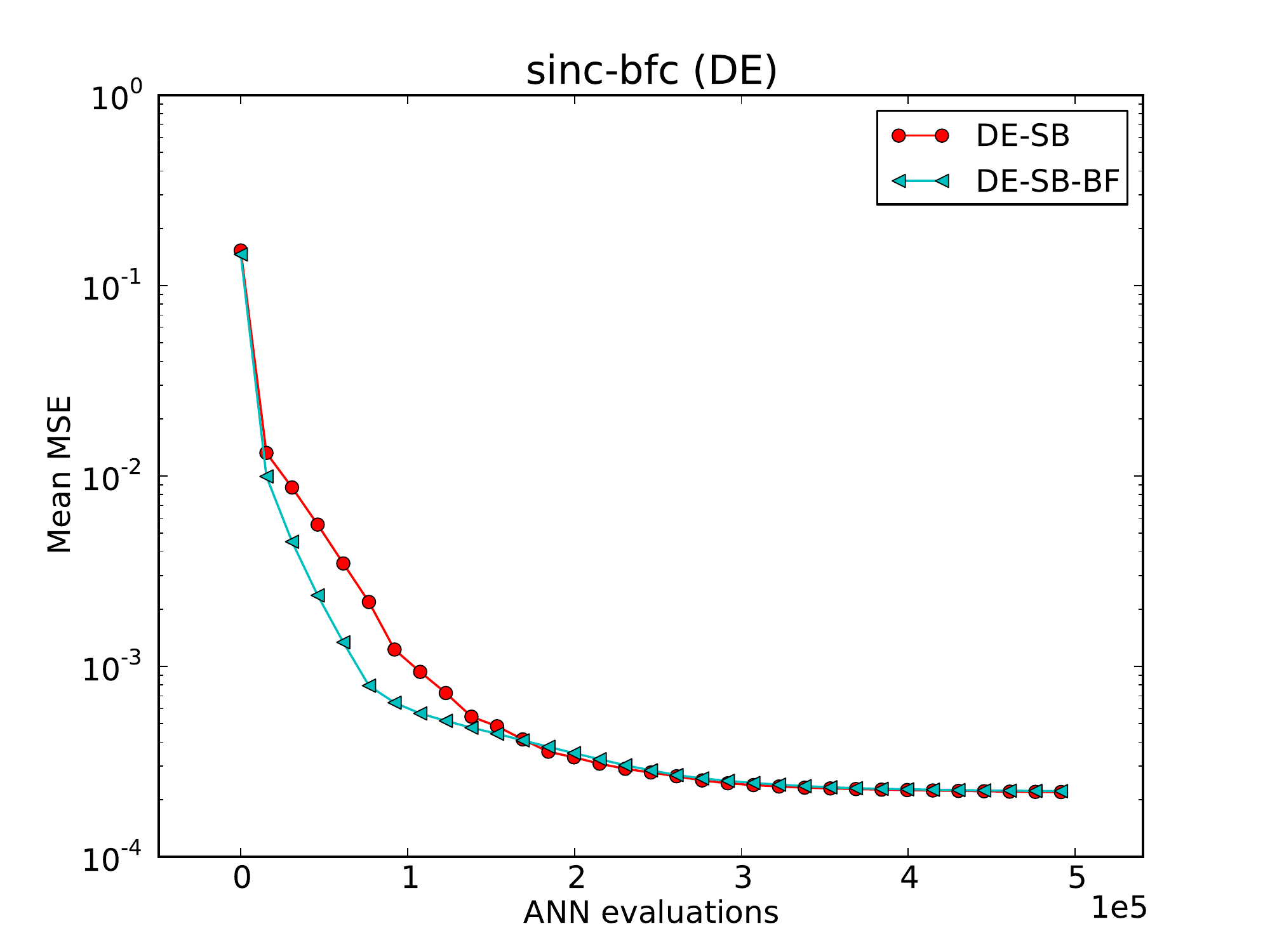}}\hspace{-0.6cm}
      \scalebox{0.26}{\includegraphics{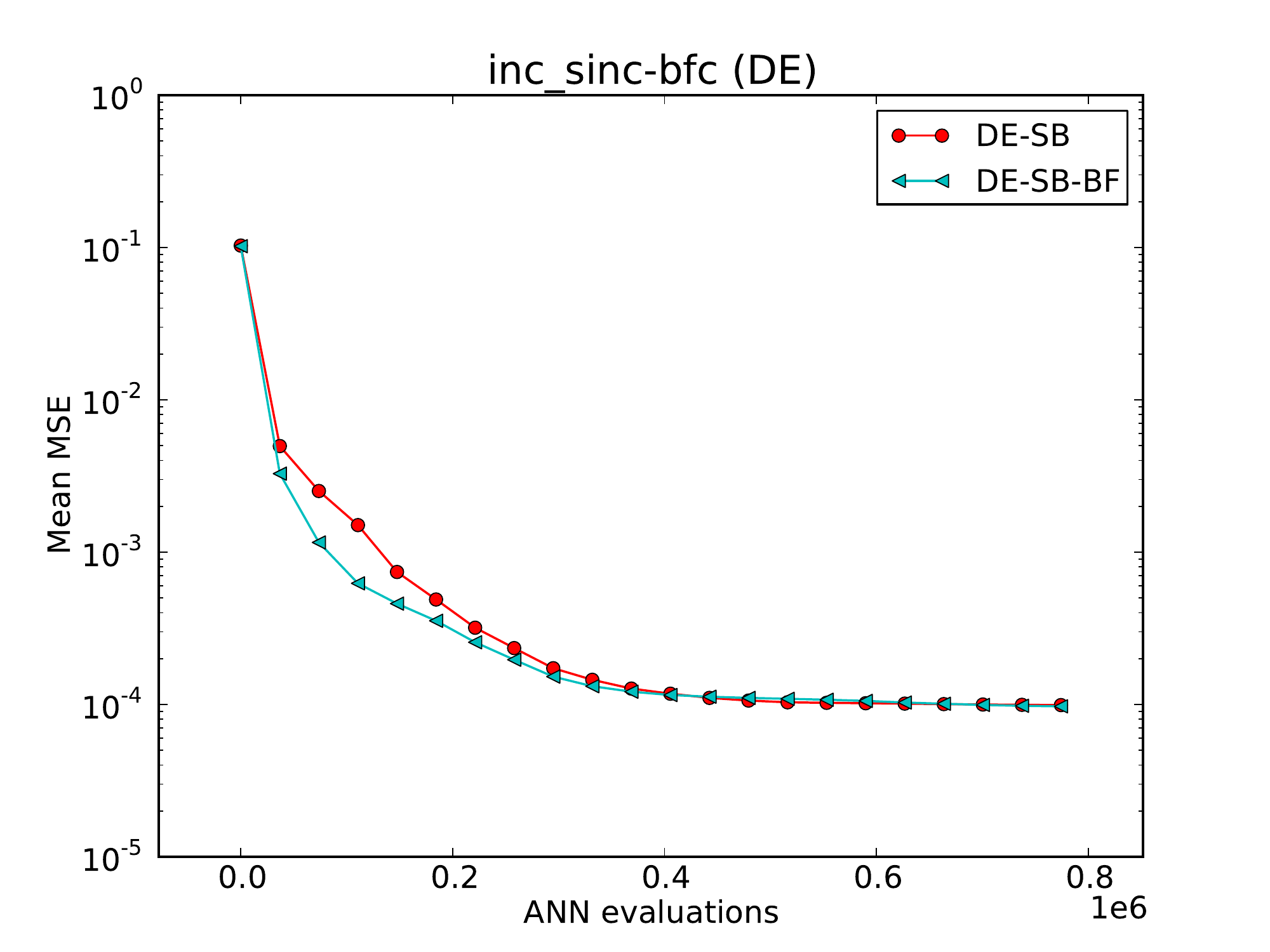}}\\
      \scalebox{0.26}{\includegraphics{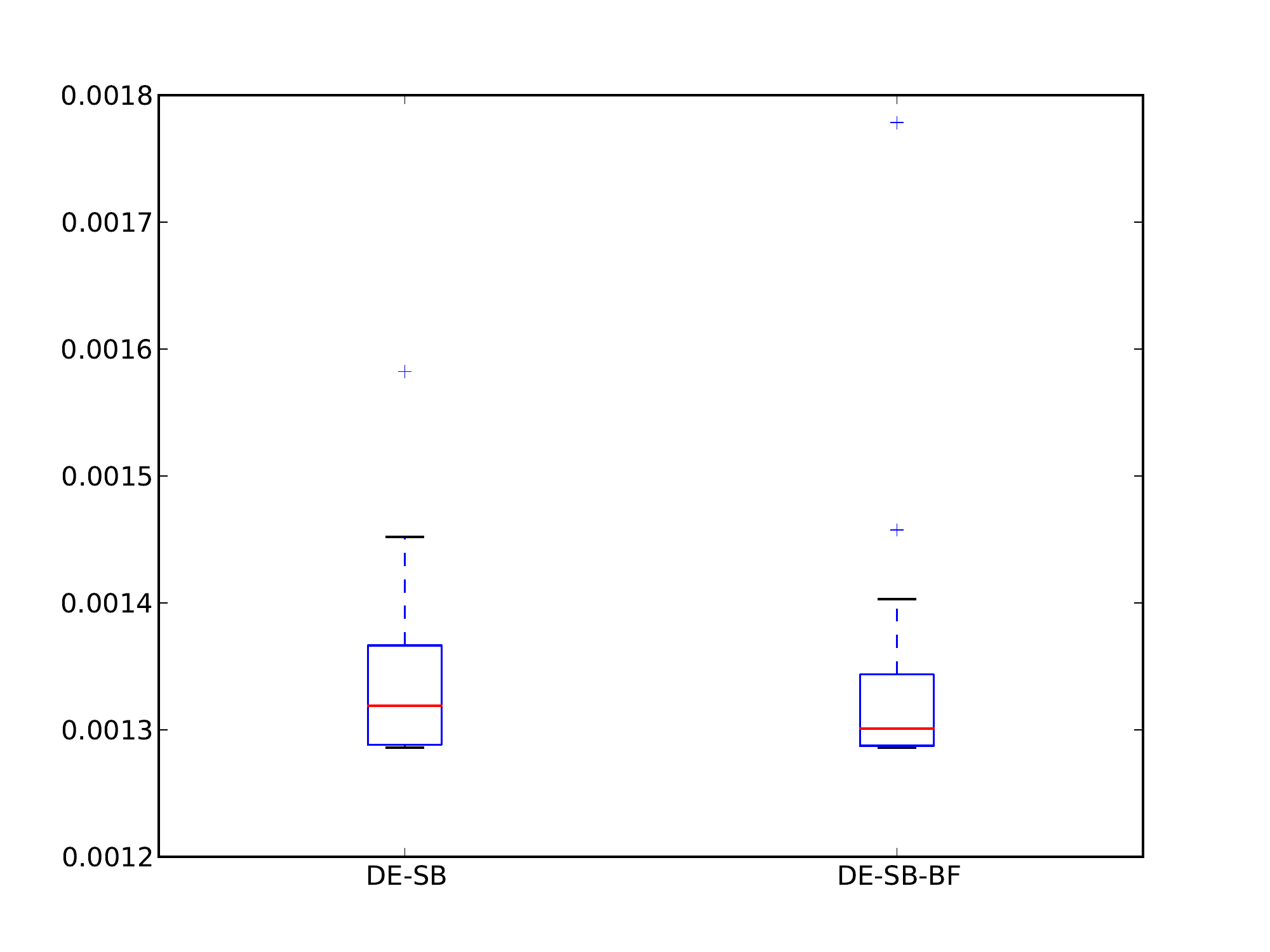}}\hspace{-0.6cm}
      \scalebox{0.26}{\includegraphics{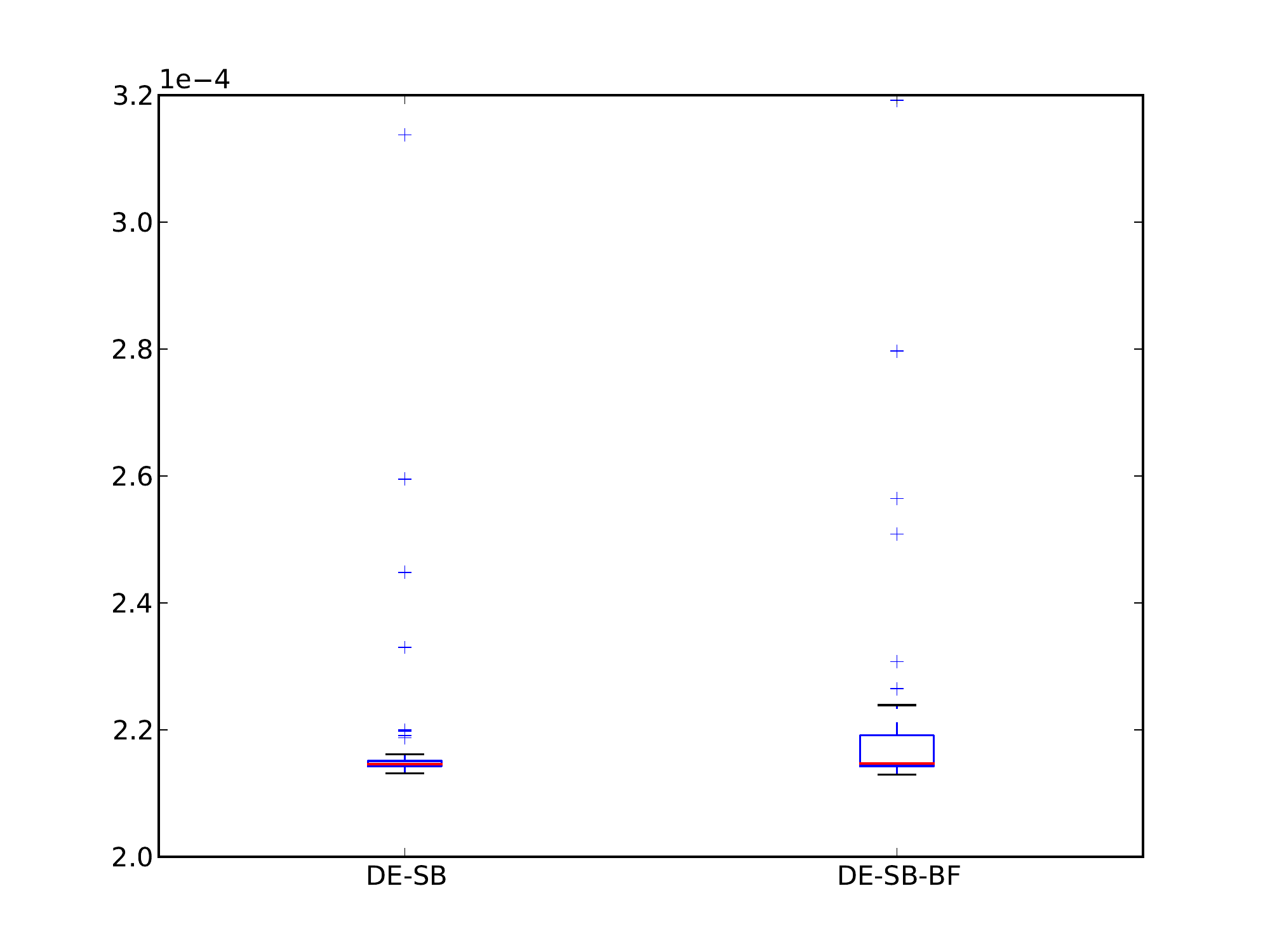}}\hspace{-0.6cm}
      \scalebox{0.26}{\includegraphics{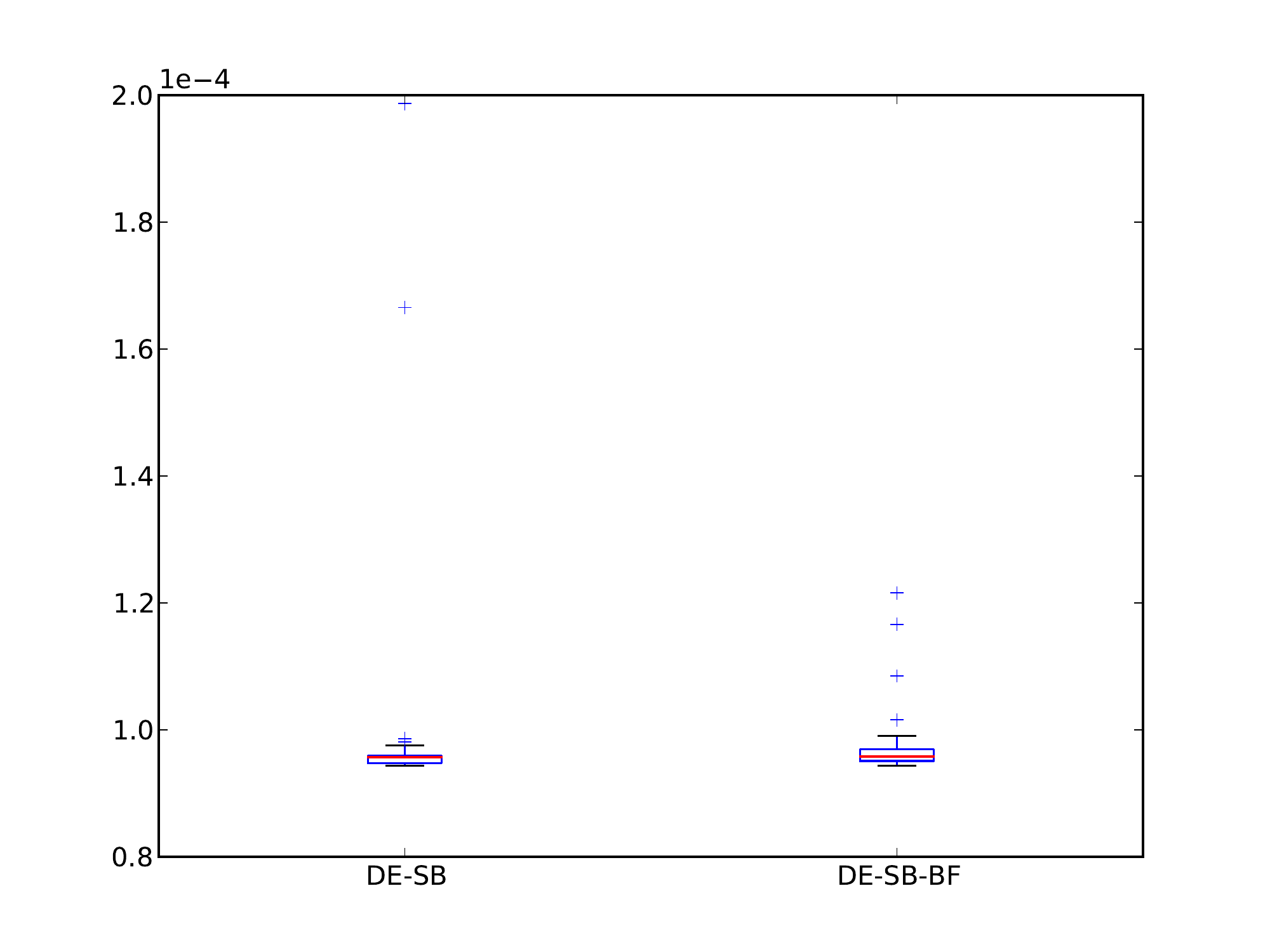}}
      \caption{\label{fig:ideal-de} \it Comparing DE-SB with DE using ideal separation by brute force symmetry breaking (DE-SB-BF) on the {\bf syn5}, {\bf sinc} and {\bf inc-sinc} datasets.}
   \end{center}
\end{figure*}
%%%%%%%%%%%%%%%%%%%%%%%%%%%%%%%%%%%%%%%%%%%%%%%%%%%%%%%%%%%%%%%%%%%%%%%%%%%%%%%%%%%%%%%%%%%%%%%%%%%%%%%%%%%%%%%%
% plot
\begin{figure*}[h!]
   \begin{center}
      \scalebox{0.26}{\includegraphics{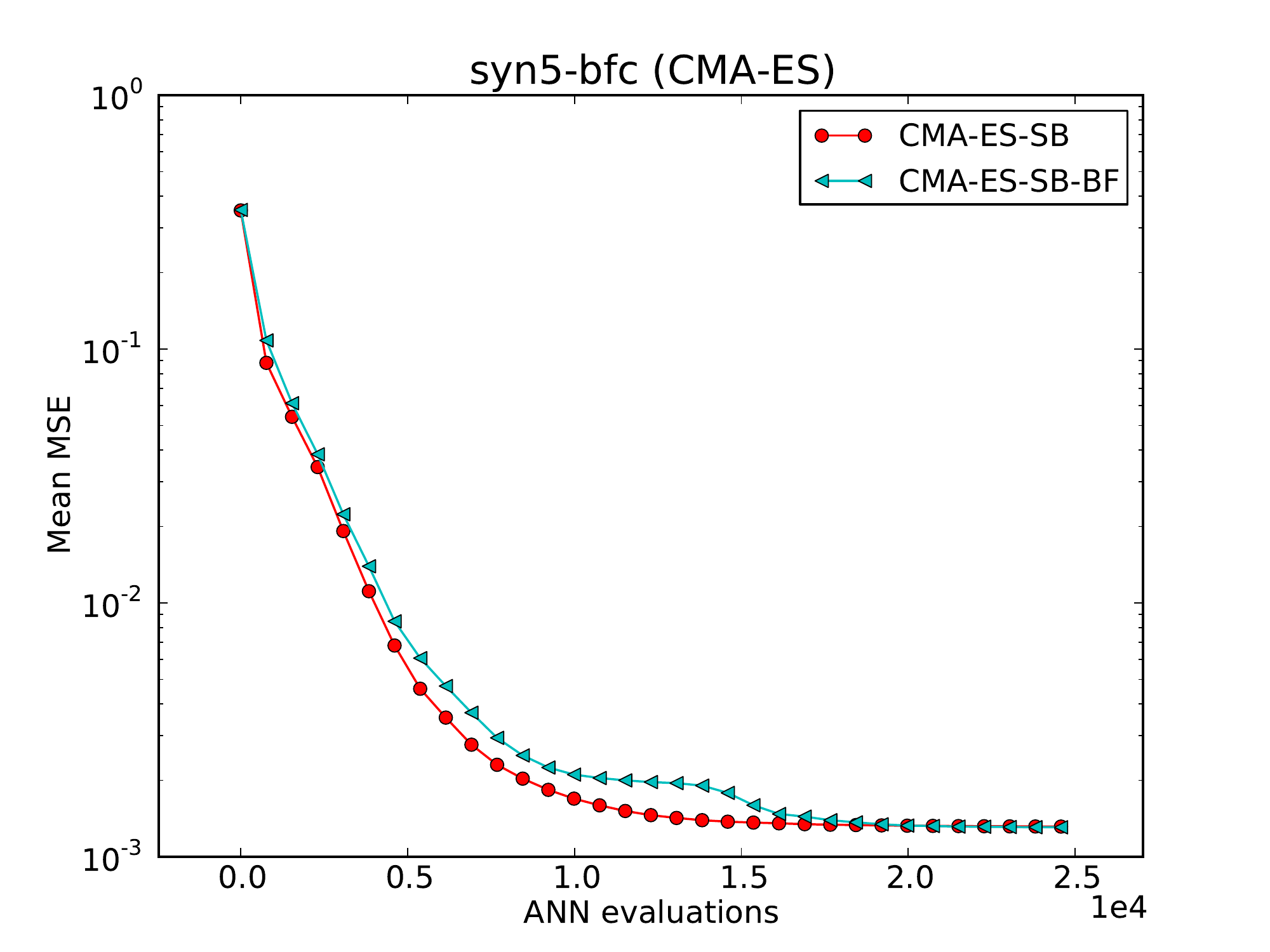}}\hspace{-0.6cm}
      \scalebox{0.26}{\includegraphics{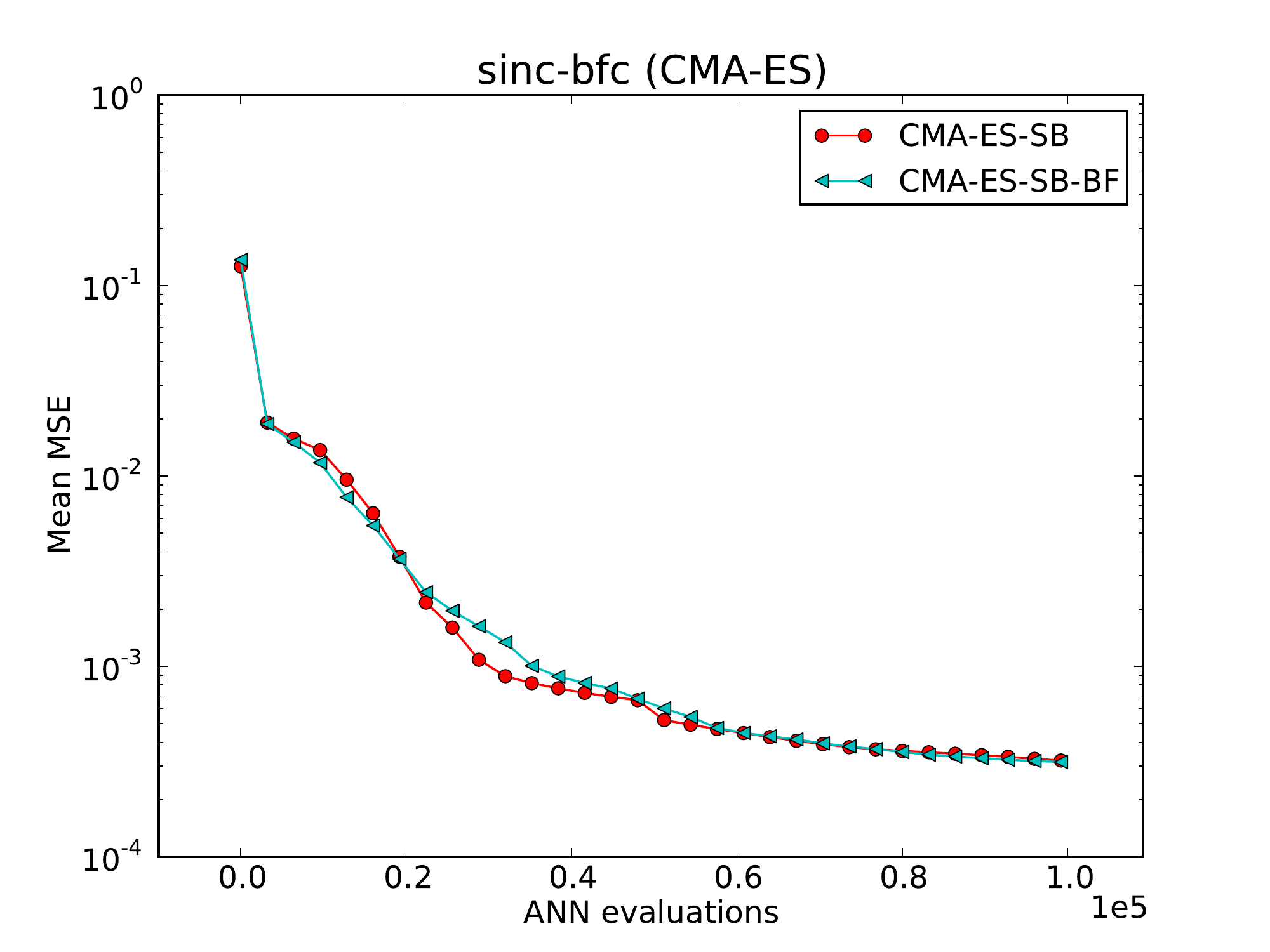}}\hspace{-0.6cm}
      \scalebox{0.26}{\includegraphics{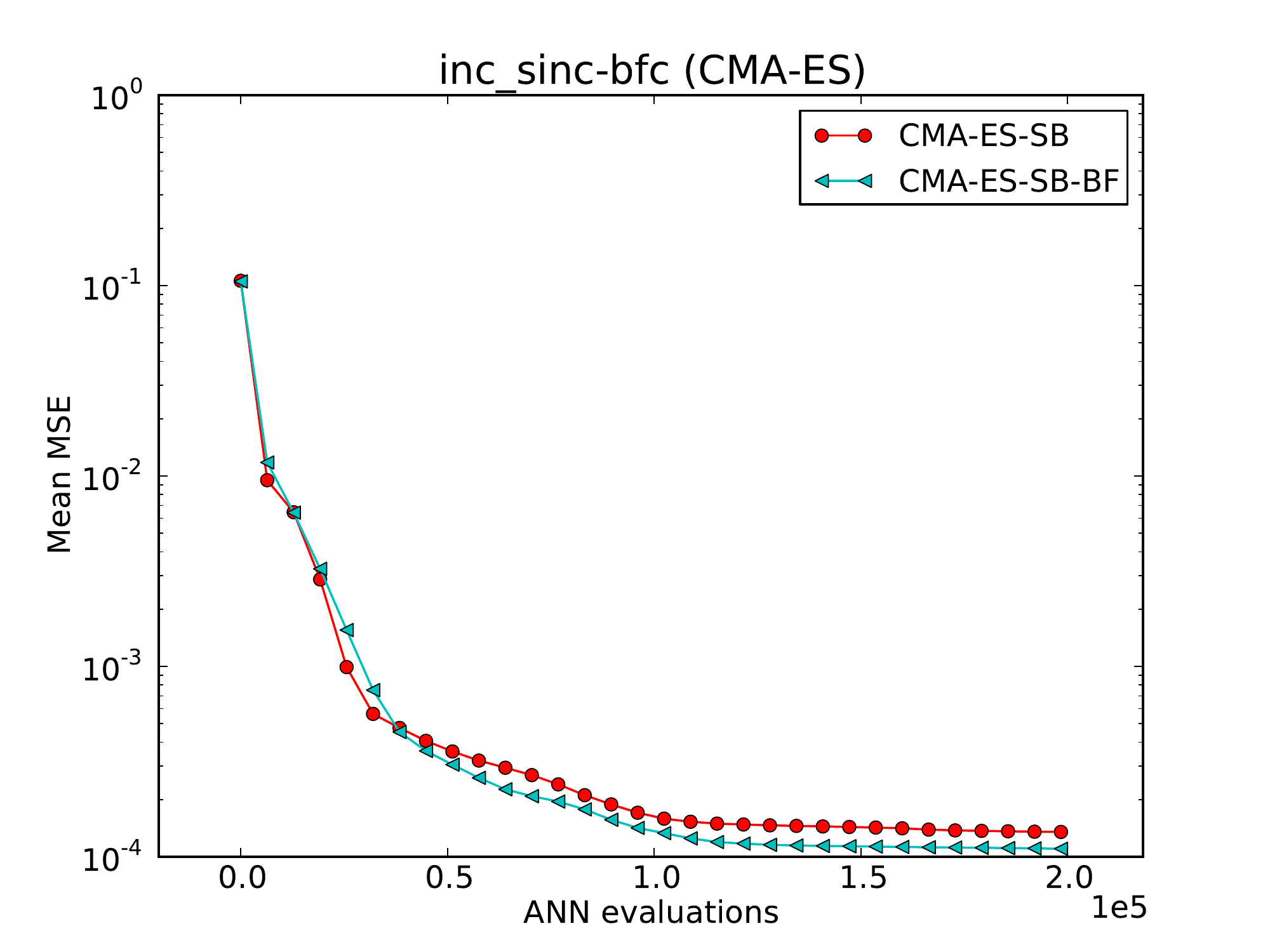}}\\
      \scalebox{0.26}{\includegraphics{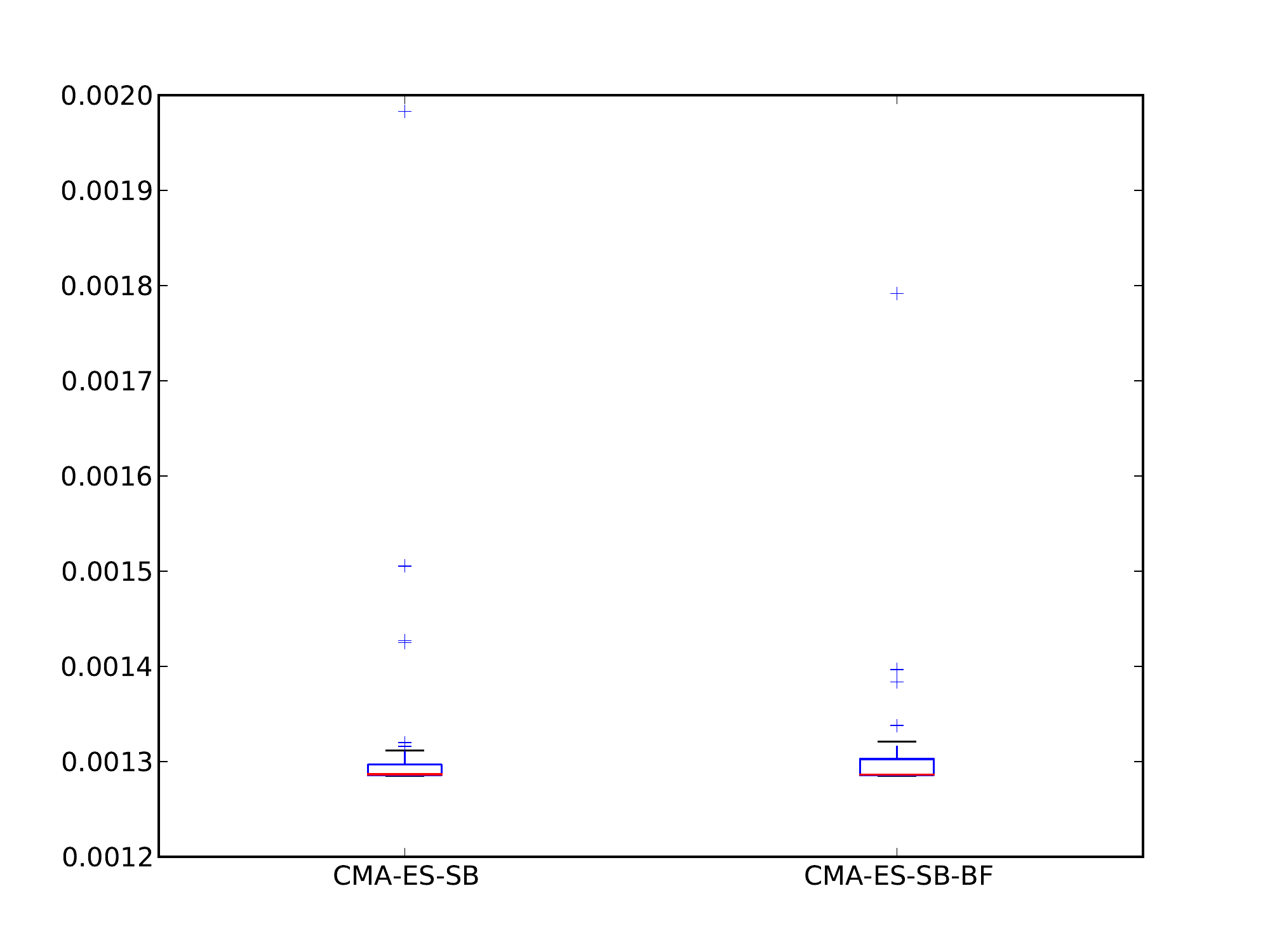}}\hspace{-0.6cm}
      \scalebox{0.26}{\includegraphics{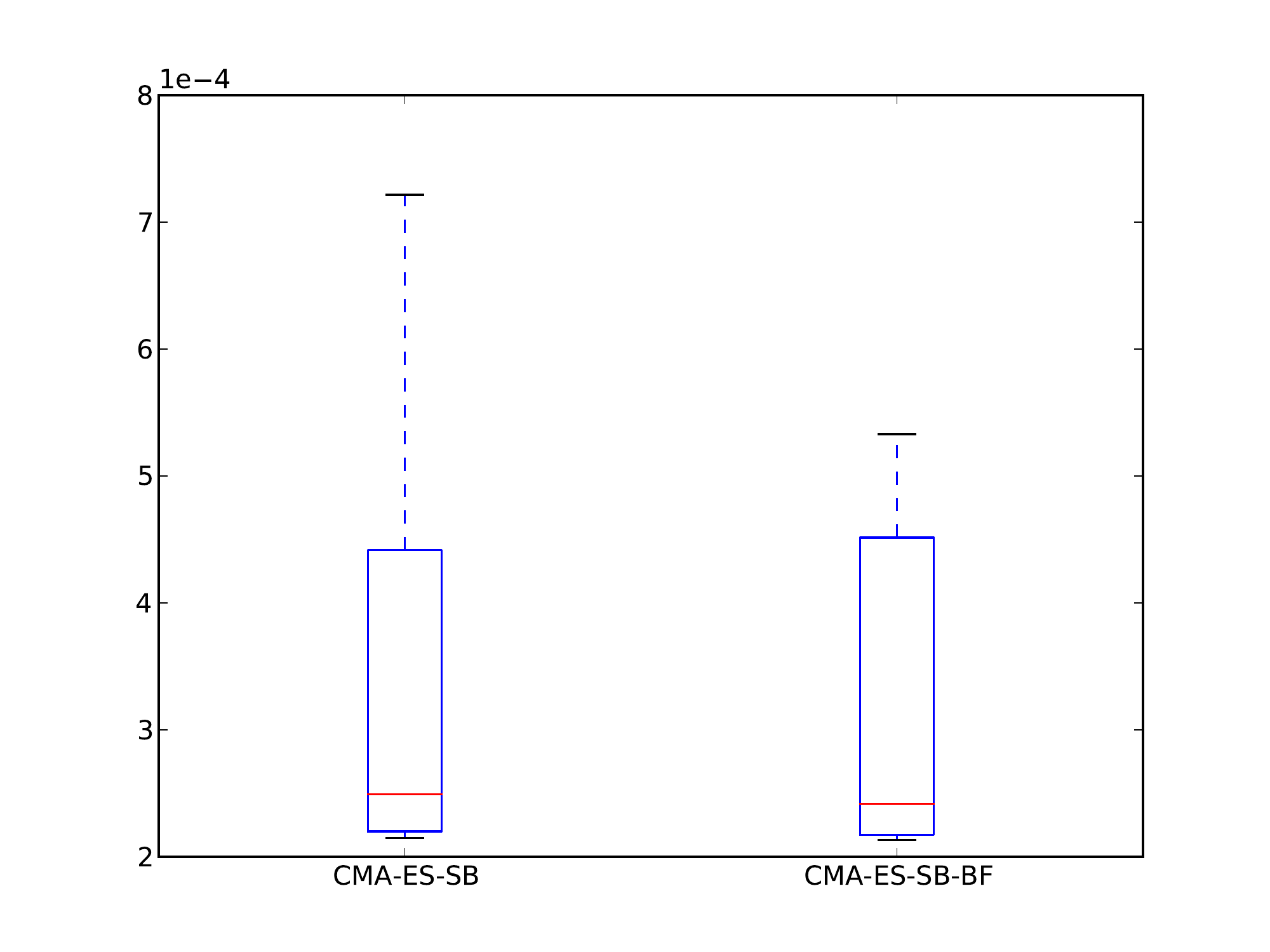}}\hspace{-0.6cm}
      \scalebox{0.26}{\includegraphics{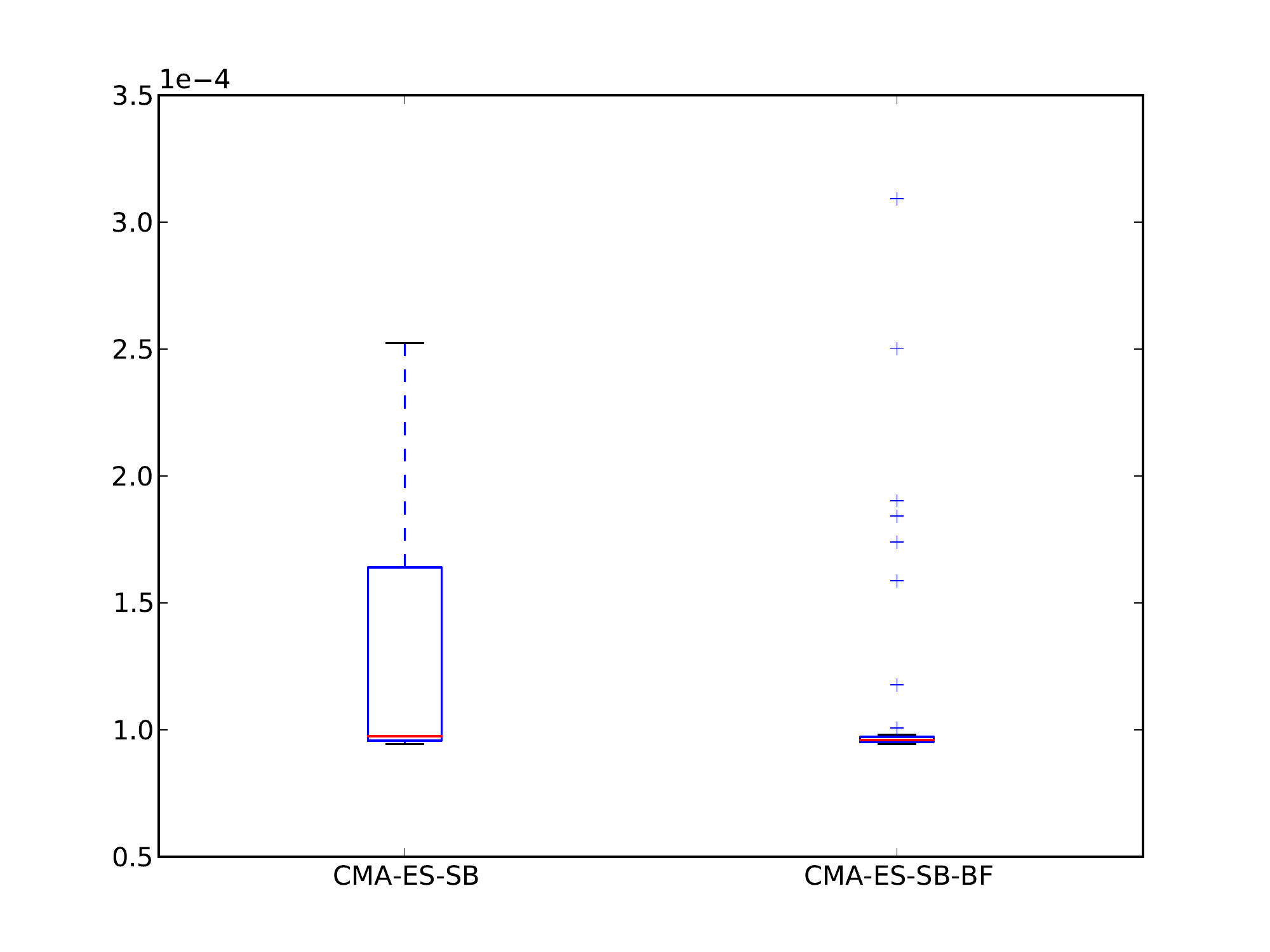}}
      \caption{\label{fig:ideal-cmaes} \it Comparing CMA-ES-SB with CMA-ES using ideal separation by brute force symmetry breaking (CMA-ES-SB-BF) on the {\bf syn5}, {\bf sinc} and {\bf inc-sinc} datasets.}
   \end{center}
\end{figure*}
\clearpage
\section{Conclusions}~\label{sec:discussion}
The problem of symmetries in ANN-parameter space is a well known problem resulting in important complication in the training of ANN's. 
However, a detailed investigation of this problem for Evolutionary Algorithms other than Genetic Algorithms is missing in the literature. Furthermore, there are contradictionary results about the efficacy of symmetry breaking methods in the performance of the global search. We show that a possible explanation for this situation is the use of symmetry breaking methods which are invariant to the global optimum and therefore can only be effective on a limited number of problems. Furthermore, we show theoretically and illustrate experimentally, that the application of global optimum invariant symmetry breaking may even lead to inferior performance. To circumvent these problems, we propose methods for global optimum variant symmetry breaking approaches for Differential Evolution (DE) and Covariance Matrix Adaptation Evolution Strategies (CMA-ES), which are two popular, robust and state-of-the-art global optimization methods. 

Experimental studies conducted on fixed topology feedforward neural networks indicate a significant improvement over standard DE and CMA-ES techniques in terms of global convergence speed. Further comparisons of the proposed approach with a common global optimum invariant symmetry breaking approach support our hypotheses. 

Based on the obtained results, we conclude that other global optimization based methods may also benefit from the use of the proposed global optimum variant symmetry breaking. Further research is required to adapt the proposed approach to other techniques to improve their performance.\\

The proposed method can be tested and verified using the open source C++ Monte Carlo Machine Learning Library (MCMLL), which is available under the GNU GPLv2 license. The website of the library can be found on: \href{http://mcmll.sourceforge.net}{mcmll.sourceforge.net}. The project website is available at \href{http://sourceforge.net/projects/mcmll/}{sourceforge.net/projects/mcmll}.
\small{
\bibliographystyle{plain}
\bibliography{myrefs}
}
\end{document}